%% file: main.tex
\documentclass{article}
    \PassOptionsToPackage{numbers, compress}{natbib}
\usepackage[preprint]{neurips_2024}

\usepackage[utf8]{inputenc} % allow utf-8 input
\usepackage[T1]{fontenc}    % use 8-bit T1 fonts
\usepackage[pagebackref=true]{hyperref}       % hyperlinks
\usepackage{url}            % simple URL typesetting
\usepackage{booktabs}       % professional-quality tables
\usepackage{amsfonts}       % blackboard math symbols       
\usepackage{nicefrac}       % compact symbols for 1/2, etc.
\usepackage{microtype}      % microtypography
\usepackage{xcolor}         % colors
\usepackage{multirow}
\usepackage{subcaption}
\usepackage{rotating}
\usepackage{bm}
\usepackage{floatrow}
\usepackage{wrapfig}
\usepackage{pgfplots}
\pgfplotsset{compat=1.16}
\usepgfplotslibrary{groupplots}
\newfloatcommand{capbtabbox}{table}[][\FBwidth]

\newcommand{\acr}[1]{{{\textsc{#1}}}}

\usepackage{style}
\usetikzlibrary{arrows.meta,arrows}
\usetikzlibrary{backgrounds,calc,bending,backgrounds}

\title{Understanding Transformer Reasoning Capabilities via Graph Algorithms}

\author{%
  Clayton Sanford$^{1,2}$\thanks{Correspondence: \texttt{clayton.h.sanford@gmail.com} or \texttt{baharef@google.com}.}, Bahare Fatemi$^{1}$, Ethan Hall$^{3}$, Anton Tsitsulin$^{1}$, \\ \textbf{Mehran Kazemi$^{4}$, Jonathan Halcrow$^{1}$, Bryan Perozzi$^{1}$, and Vahab Mirrokni$^{1}$}\\
  $^{1}$Google Research, $^{2}$Columbia University,$^{3}$Google, $^{4}$Google DeepMind
}

\begin{document}

\maketitle

\input{abstract}
\input{intro}

\input{related_work}
\input{hierarchy}
\input{empirical}

\input{conclusion}

\bibliography{bib.bib}
\bibliographystyle{plainnat}
%%%%%%%%%%%%%%%%%%%%%%%%%%%%%%%%%%%%%%%%%%%%%%%%%%%%%%%%%%%%

\appendix

\input{prelims}
\input{transformer_mpc}

\input{graphs_old}
\input{gnn_limitations}

\input{empirical_appendix}

\end{document}

%% file: abstract.tex
\begin{abstract}
% Transformers have become a standard architecture for neural networks in various domains.
% While tremendous advances have been achieved by transformer-based neural networks empirically, theoretical understanding of their algorithmic reasoning capabilities in realistic parameter regimes has been lacking.

% Algorithmic reasoning capabilities closely relate to generalization: in order to successfully solve an out-of-distribution problem instance, models need to learn the underlying algorithm of solving the task.
% This paper focuses on the abilities of transformer models to solve algorithmic tasks on graphs.
% {\color{red} We aim to address this.}
%We aim to address this by answering a simple question --
% How much capacity does a transformer model need in order to perfectly solve different classes of algorithmic problems?
Which transformer scaling regimes are able to perfectly solve different classes of algorithmic problems? 
While tremendous empirical advances have been attained by transformer-based neural networks, a theoretical understanding of their algorithmic reasoning capabilities 
% that explains their capabilities 
in realistic parameter regimes is lacking.
We %propose to 
investigate this question in terms of the network's depth, width, and number of extra tokens for algorithm execution.
%We study 9 algorithmic reasoning problems and our representational hierarchy cleanly separates them into classes solvable by transformers in different parameter scaling regimes.
Our novel representational hierarchy separates 9 algorithmic reasoning problems into classes solvable by transformers in different realistic parameter scaling regimes.
%We introduce a novel representational hierarchy that categorizes reasoning problems into three classes: retrieval, parallelizable, and search tasks, each solvable by transformers in different parameter scaling regimes.
We prove that logarithmic depth is necessary and sufficient for tasks like graph connectivity, while single-layer transformers with small embedding dimensions can solve contextual retrieval tasks.
% We prove that single-layer transformers with small embedding dimensions can solve contextual retrieval tasks %like node and edge counting,
% while logarithmic depth is necessary and sufficient for other more complex multi-step algorithmic tasks.% require logarithmic depth.
%parallelizable algorithmic tasks like connectivity
%require logarithmic depth.
%% For search tasks like shortest path, logarithmic-depth transformers are sufficient with larger embedding dimensions.
We also support our theoretical analysis with ample empirical evidence using the GraphQA benchmark.
These results show that transformers excel at many graph reasoning tasks, even outperforming specialized graph neural networks.
\end{abstract}

%% file: intro.tex
% \section{Overview of theoretical results}
\section{Introduction}

The transformer neural network architecture, which was initially introduced for neural machine translation~\cite{bahdanau2014neural,vsp17}, quickly became the standard neural network architecture across many fields, powering recent breakthroughs in language modeling~\cite{devlin2018bert,raffel2020exploring,fewshot,touvron2023llama,achiam2023gpt,team2023gemini,reid2024gemini}, computer vision~\cite{dosovitskiy2021image,liu2021swin}, and natural sciences~\cite{jumper2021highly,merchant2023scaling}.
Across fields, transformers have superseded other architectures, surpassing them in downstream performance while maintaining reasonable computational footprint.%, especially across parallel computing infrastructure.

%While the empirical success of transformers is undeniable, there is only limited theoretical understanding of their reasoning capabilities % exceed those of alternative neural architectures
%in practical parameter scaling regimes.
How can we analyze reasoning capabilities of neural networks?
One approach is to study algorithms executable with their internal representations.
Neural algorithmic reasoning~\cite{graves2014neural,zaremba2014learning,joulin2015inferring,kaiser2015neural,velivckovic2021neural} is a field of research dedicated to exploring such capabilities.
Algorithmic execution is desirable because models use it to generalize out-of-distribution~\cite{zhou2022teaching,lee2024teaching} and scale to larger problem sizes~\cite{zhou2023algorithms}.
%For transformer-based models, algorithms have been explored models elementary arithmetic~\cite{lee2024teaching}

In this work, we focus on classes of transformers solving algorithmic reasoning problems on graphs.
Why graph problems in particular?
% *. Graphs represent reasoning with dependencies: implicit graph
% *. Graph reasoning is a well-researched field; this is not some crazy useless new problems that we study
% 3. Graph problems vary in complexity: this allows us to actually even have a hierarchy
Recent research by \citet{Besta_2024} suggests that graphs are an ideal abstraction of complex reasoning with dependencies, including {chain-of-thought}~\cite{wei2022chain} and its generalizations~\cite{yao2024tree,creswell2023selectioninference,kazemi-etal-2023-lambada,kassner-etal-2023-language}.
% These representations give us an implicit structure
% These informative connections permit the application of analytical tools that taxonomize the hardness of reasoning tasks across different complexity classes to an understanding of the reasoning capabilities of transformer-based models.
% developed for tasks spanning different complexity classes to understand what reasoning can be achieved by transformer-based models.
% However this raises a very natural question -- what graph reasoning can Transformers actually do?
% We believe there is an intriguing connection
% However, given that we can represent abstract reasoning with graphs, how can 
% However, then 
% % Their study allow us 
% We are able to reuse analytical tools for tasks spanning different complexity classes to a {\color{red} framework of problems} solvable in different parameter regimes.
% Second,  
Furthermore, we investigate graph reasoning in order to evaluate transformer capabilities compared to specialized models that explicitly capture the structure of data.
Graph neural networks (GNNs)~\cite{scarselli2008graph,kipf2016semi,gilmer2017neural,battaglia2018relational,chami2022machine} offer strong baselines for algorithmic reasoning~\cite{velivckovic2022clrs,cappart2023combinatorial} and comprehensive theoretical frameworks for their capabilities and limitations~\cite{loukas2020graph,xhlj18}.
% Graph reasoning problems have an established literature with strong empirical baselines for graph neural networks (GNNs)~\cite{scarselli2008graph,kipf2016semi,gilmer2017neural,battaglia2018relational,chami2022machine} that are commonly employed for algorithmic reasoning~\cite{velivckovic2022clrs,cappart2023combinatorial}, alongside comprehensive theoretical frameworks for their reasoning capabilities and limitations~\cite{loukas2020graph,xhlj18}.
% Finally, g
The complexity of graph reasoning tasks varies greatly; some are easily solved by non-graph architectures and others require sophistication beyond standard GNNs~\cite{zhang2023rethinking}.
This makes graph tasks a compelling testbed for both theoretical and empirical evaluations.

% This allow us to study complexity hierarchy of algorithmic reasoning problems on graphs in realistic parameter scaling regimes for transformers. % \emph{in realistic parameter regimes}.
While transformer-based models are commonly optimized for downstream performance~\cite{kaplan2020scaling,hoffmann2022training,psdsgl2024}, theoretical investigations of transformer capabilities in realistic parameter regimes have been limited.
%Despite the existence of numerous hybrid architectures between GNNs and transformers \cite[see e.g.][]{mcb+22}, 
We analyze the capabilities of standard transformer models to solve graph reasoning tasks by employing a graph tokenization scheme similar to~\cite{knmcllh22}; see \Cref{fig:transformer_architecture} for a graph encoding example.
We generalize the insights of \cite{sht24}---which established that graph connectivity can be computed in a more depth-efficient manner with transformers than GNNs---to broader families of graph reasoning tasks.
We also study the representational impacts of \emph{blank tokens}, which can be thought of as a theoretically-tractable version of either chain-of-thought prompting~\cite{wei2022chain}, scratch space~\cite{nye2021show}, or pause/filler tokens~\cite{goyal2023think,pfau2024let}.

We conduct an {extensive study} of graph reasoning problems with transformers from both theoretical and empirical perspectives. We summarize our principal contributions as follows:
% Our principal contributions are as follows:
\begin{enumerate}[noitemsep,left=0pt]
\item We introduce a novel \textit{representational hierarchy} of graph reasoning tasks that formalizes reasoning capabilities of transformers in several realistic parameter scaling regimes.
This includes two graph reasoning task classes---which we refer to as the \textit{parallelizable} and \textit{search tasks} and include well-known algorithmic tasks like \textit{connectivity} and \textit{shortest path} respectively. 
\item We prove that logarithmic-depth transformers are necessary and sufficient to solve parallelizable tasks in a highly parameter-efficient manner; similar constructions for search tasks employ much larger networks.
We distinguish both of these classes from an easier family of problems---\textit{retrieval tasks}, such as \textit{node count} and \textit{edge existence}---by showing that retrieval tasks can be efficiently computed by single-layer transformers, while parallelizable and search tasks cannot. 
\item We empirically validate our representational contrast by showing that transformers outperform GNNs on tasks that require the analysis of long-range dependencies, including parallelizable tasks like connectivity. 
Our experiments suggest that GNNs perform better on simpler tasks that require only local analysis of neighboring nodes in low sample-complexity regimes, benefiting from their well-aligned inductive bias.
\end{enumerate}

%% file: related_work.tex
\section{Related work}

\paragraph*{Fundamental capabilities of transformers.}

% Proposal: separate by analysis class
% - circuit complexity (dismiss most (all?) for unrealistic parameter regimes
% - algorithms (what is the novelty wrt. sht23,sht24? very important)

% \begin{figure}[t!]
%     \centering
%     \includegraphics[width=\textwidth]{figures/figure_1_draft.jpeg}
%     \caption{An example task of solving connectivity with a transformer.}
%     \label{fig:transformer_architecture}
% \end{figure}

\begin{figure}[t!]
    \centering
    \resizebox{\textwidth}{!}{\input{figures/figure_1_rework}}
    \caption{The graph encoding scheme employed in our theoretical and empirical analysis that presents a graph reasoning task (e.g. connectivity) as a tokenized input to a standard transformer model.}
    \label{fig:transformer_architecture}
\end{figure}

%Our theoretical investigation exists in a dynamic research ecosystem of works that aim to characterize fundamental capabilities and limitions of transformers under a variety of model assumptions and scaling regimes. 
% Closest to our line of investigation are works that aim to theoretically understand capabilities and limitations of transformers.
Early representational research~\cite{ybrrk19,pbm21,wcm21} established the universality of transformers by simulating Turing machines step-by-step, albeit in an unrealistic polynomial-depth scaling regime.
These universality results were extended to bounded-depth transformers with chain-of-thought tokens \cite{ms23-cot, malach23}, but with a persistent gap between positive and negative results.
% I want these to be in the same paragraph - Anton
A more fine-grained theoretical approach established the limitations of bounded-size transformers by relating them to threshold circuits, implying that $\L$-complete tasks like graph connectivity are unsolvable by constant-depth transformers~\cite{ms23, ms22, mss21}.
However, these circuit complexity reductions are not bidirectional and their positive results pertaining to classes of regular languages~\cite{haf22} reveal little about when tasks like connectivity \textit{can} be expressed.
% Another line of work obtains lower bounds for \todo{which?} algorithms by reducing transformers to circuits.
% This line of work not always upper bounds \todo{CHECK}; moreover, certain graph reasoning tasks (\eg connectivity) cannot be done with constant depth models \cite{ms23, ms22, mss21} and positive results for formal language tasks \cite{haf22}.

The closest analytical framework to ours characterizes transformers as distributed computing protocols by quantifying the informational bandwidth of their self-attention units.
This line of work sharply separates transformers from other architectures and provides width and depth separations~\cite{sht23,sht24}.
The ability of transformers to simulate finite-state automata in logarithmic depth provides a further roadmap for how transformers efficiently leverage parallel computation~\cite{lagkz22}.%The parallel capabilities of transformers are further elucidated by 
% The ability of transformers to simulate automata using a parallel computation abilities of transformers is also captured

% Transformers as automata: roadmap for how transformers can use parallel processing powers to unroll a $T$-step automata simulation in $O(\log T)$ layers \cite{lagkz22}.

\paragraph*{Empirical analysis of transformer reasoning capabilities.}

Multiple empirical investigations explore the algorithmic capabilities of transformer-based models~\cite{zbbegw22,lm22,lee2024teaching}.
Graph reasoning problems have been used to evaluate capabilities of transformer-based LLMs~\cite{fhp23,wang2023can,yzwxz23,smk23,jin2023large} including works with graph-specific GNN adapters~\cite{czwhhhy23,pfztkah24}.
In the empirical part of our study, we use the GraphQA dataset~\cite{fhp23}---which was evaluated initially on LLM prompting appraoches and later on GNN-augmented prompts~\cite{pbm21}---to validate our complexity hierarchy and compare transformers with GNNs and prompt-based LLMs on a selection of graph reasoning tasks.

\paragraph*{Transformers and graph neural networks.}

% When studying expressivity of machine learning algorithms that operate on graphs, GNNs provide a useful reference point.
GNNs provide a useful benchmark for model expressivity on graph inputs.
For instance, the inability of GNNs to efficiently solve ``global structure'' tasks like subgraph connectivity made evident by a connection to the \congest{} distributed computing model~\cite{loukas2020graph}.
A further connection~\cite{morris2023weisfeiler} to the Weisfeiler-Lehman (WL) isomorphism heuristic~\cite{weisfeiler1968reduction} captures the inability of feature-less GNNs to distinguish certain graph instances~\cite{xhlj18}. 
See \Cref{asec:gnn_limitations} for a further discussion of the representational limitations of GNNs.
% Another important connection~\cite{morris2023weisfeiler} is through the Weisfeiler-Lehman (WL) isomorphism heuristic~\cite{weisfeiler1968reduction}; feature-less GNNs cannot distinguish graphs that are not separated by the WL test~\cite{xhlj18}.

Graph transformers~\cite{dwivedi2020generalization,rong2020self,mcb+22} integrate the transformer architecture with graph-structured data.
While GNNs can simulate the WL test~\cite{acinrssw22}, transformers can simulate GNNs and exceed their WL limitations~\cite{zhang2023rethinking}; they are strictly more expressive than 2-WL~\cite{knmcllh22}.
Despite our focus on standard transformers, our empirical results in \Cref{sec:transformers-gnns} include comparisons to a wide range of GNNs.

%% file: figures/figure_1_rework.tex
\begin{tikzpicture}%
    % Graph
	\node[circle, draw=cycle1, thick, fill=cycle6!10!white, inner sep=0.1em] (v1) {$v_1$};
	\node[circle, draw=cycle1, thick, fill=cycle6!10!white, inner sep=0.1em] (v2) at ([shift=({100:3 em})]v1) {$\bm{v_2}$};
	\node[circle, draw=cycle1, thick, fill=cycle6!10!white, inner sep=0.1em] (v3) at ([shift=({-30:3 em})]v2) {$v_3$};
	\node[circle, draw=cycle1, thick,
	 fill=cycle6!10!white, inner sep=0.1em] (v4) at ([shift=({80:3 em})]v3) {$\bm{v_4}$};
	
    % Isolated node
	\node[circle, draw=cycle5, thick,
	 fill=cycle5!10!white, inner sep=0.1em] (v5) at ([shift=({-60:4 em})]v4) {$v_5$};
	 
	% Green block
	\node[circle, draw=cycle9, thick,
	 fill=cycle9!10!white, inner sep=0.1em] (v6) at ([shift=({-30:2 em})]v1) {$v_6$};
	\node[circle, draw=cycle9, thick,
	 fill=cycle9!10!white, inner sep=0.1em] (v7) at ([shift=({-60:3 em})]v6) {$v_7$};
	\node[circle, draw=cycle9, thick,
	 fill=cycle9!10!white, inner sep=0.1em] (v8) at ([shift=({60:3 em})]v7) {$v_8$};
	 
	% Red edges
	 \draw[-, black] (v1) -- (v2);
	 \draw[-, black] (v2) -- (v3);
	 \draw[-, black] (v3) -- (v4);
    
    % Green edges
	 \draw[-, black] (v6) -- (v8);
	 \draw[-, black] (v6) -- (v7);
	 \draw[-, black] (v7) -- (v8);
	 
    \node (graph) at ([shift=({90:2 em})]v4) {\textbf{Graph} $\mathrm{G}$};
    \node (task) at ([shift=({90:-10 em})]v4) {\textbf{Task}: Are $\bm{v_2}$ and $\bm{v_4}$ connected?};
    
    % \begin{pgfonlayer}{background}
    %  \draw[fill=cycle1!5!white] (graph.north -| task.west) rectangle (task.south -| task.east);
    % \end{pgfonlayer}

    \foreach \i in {1,...,8}
    {
    \node[draw, rectangle, minimum width=2cm, minimum height=0.5cm,fill=cycle2!10!white] (rect\i) at (6, 4-0.5*\i) {$v_{\i}$}; 
    }
    \foreach \i/\text in {9/{$v_1,v_2$},10/{$v_2,v_3$},11/{$v_3,v_4$},12/{$v_6,v_7$},13/{$v_6,v_8$},14/{$v_7,v_8$}}
    {
    \node[draw, rectangle, minimum width=2cm, minimum height=0.5cm,fill=cycle3!10!white] (rect\i) at (6, 4-0.5*\i) {\text}; 
    }
    \node[draw, rectangle, minimum width=2cm, minimum height=0.5cm,fill=cycle4!15!white] (rect15) at (6, -3.5) {$\bm{v_2,v_4}$}; 
    
    \draw[-Stealth, thick, black] (2.65, 0) -- (4.85,0);
    \node[] at ($(rect1.north) + (0, 0.2)$) {\textbf{Input} $\mathrm{X}$};
    \node[rotate=90] at ($(rect1.west)!0.5!(rect8.west) + (-0.2, 0)$) {\small\textbf{Vertex tokens}};
    \node[rotate=90] at ($(rect9.west)!0.5!(rect14.west) + (-0.2, 0)$) {\small\textbf{Edge tokens}};
    \node (tt) at ($(rect15.south west) + (-1, 0.5)$) {\small\textbf{Task tokens}};
    \draw[-Stealth, black] (tt.south) to[out=-90,in=180] (rect15.west);
    
    \foreach \i/\text in {1/{Embedding},2/{Self-Attention},3/{MLP},4/{Self-Attention},5/{MLP},6/{Self-Attention}}
    {
    \node[draw, rectangle, minimum width=3.5cm, minimum height=0.5cm,fill=cycle3!15!white] (tr\i) at (10, 4-\i) {\text};
    }
    \foreach \i in {1,...,5}
    {
    \draw[-Stealth, thick, black] (10, 3.75-\i) -- (10,3.25-\i);
    }
    
    \draw[-Stealth, thick, black] (7.15, 0) -- (7.65,0);
    \node at ($(tr6.south) + (0, -0.35)$) {\dots};
    \node at ($(tr1.north) + (0, 0.55)$) {\textbf{Transformer} $f$};
    \begin{pgfonlayer}{background}
    \draw[fill=cycle3!5!white] ($(tr1.north west) + (-0.5, 0.35)$) rectangle ($(tr6.south east) + (0.5, -0.75)$);
    \end{pgfonlayer}
    
    \foreach \i in {1,...,14}
    {
    \node[draw, rectangle, minimum width=0.8cm, minimum height=0.5cm,fill=cyclegray!10!white] (task\i) at (13.5, 4-0.5*\i) {$\varnothing$}; 
    }
    \draw[-Stealth, thick, black] (12.35, 0) -- (13,0);
    \node[draw, rectangle, minimum width=0.8cm, minimum height=0.5cm,fill=cycle4!20!white] (rect15) at (13.5, -3.5) {Yes}; 
    \node[] at ($(task1.north) + (0, 0.2)$) {\textbf{Target} $f(\mathrm{X})$};
    
\end{tikzpicture}

%% file: hierarchy.tex
\section{Hardness taxonomy of transformer graph reasoning tasks}\label{sec:hierarchy}

We provide our core result: a rigorous quantification of the hardness of graph reasoning tasks for transformer-based models.
While graph reasoning tasks, like other algorithmic problems, can be categorized into well-known computational and circuit complexity classes (e.g. $\textsf{TC}^0$, $\L$, $\NL$, $\textsf{NP}$, $\textsf{NP}$), the relationship between membership in these classes and the hardness of solving a task with a parameter-efficient neural network is not immediately obvious.
Our hierarchy bridges the gap between these worst-case computational classes and the representational capabilities of bounded-size transformers of different parameter scaling regimes. 
These regimes include transformers whose depth $L$ scales with the input sequence length $N$; this contrasts with most theoretical results, which study on the constant-depth regime.
% We prove theoretical results that quantify the hardness of learning different graph reasoning tasks in various transformer scaling regimes.

Our positive and negative theoretical results employ the transformer model of \cite{sht24}, which is presented in \Cref{assec:prelim-trans} and assumes that the embedding dimension $m$ grows less rapidly than than the sequence length $N$ and that the multi-layer perceptrons (MLPs) are arbitrary functions.
The result is a model of a transformer as a \emph{bounded-capacity communication protocol}.
In this model, arbitrary functions of each individual embedding vector can be computed, but the interactions between these vectors are restricted by the low-rank of the attention matrix.
The relevance of this model is motivated by the rapid scaling in the context lengths of modern transformers in recent years and the high ratio of MLP parameter count to the embedding dimension each operates on. %\todo{citations needed}
% We introduce this formal model of a transformer in \Cref{sec:prelims}.

This model also permits the inclusion of blank \emph{pause token} inputs of~\cite{goyal2023think}, which provide additional computational power to the transformer model by extending the computational ``tape'' without introducing new information about the input.

These results divide the tasks that are we empirically investigate in \Cref{sec:transformers-gnns} into three difficulty families based on the hardness of solving the task with parallel computation.
\begin{enumerate}[noitemsep,left=0pt]
    \item \textbf{Retrieval tasks}---including node count, edge count, edge existence, and node degree---are tasks that can intuitively be solved by a single lookup step or global aggregation. 
    These are the easiest tasks in our framework, and we show in \Cref{ssec:easy-tasks} that these retrieval tasks can be computed by a single-layer transformer with small embedding dimension. (In contrast, all other examined tasks cannot be solved in that regime.) 
    \item \textbf{Parallelizable tasks}---including connectivity, connected nodes, and cycle check from the experimental results and a host of other graph reasoning tasks, such as bipartiteness, planarity, and minimum spanning forest---are non-trivial tasks that can be solved efficiently in a parallel computation setting. \Cref{ssec:med-tasks} establishes that these tasks can be solved by bounded-size transformers with logarithmic depth.
    \item \textbf{Search tasks}---including shortest path and other tasks like diameter and directed reachability---comprise a harder family of tasks that are less easily solved by parallel algorithms.
    In \Cref{ssec:hard-tasks}, we prove that these tasks belong in an equivalence class and exhibit large-model scaling regimes where they can be computed.
\end{enumerate}

The representational hardness of these classes is quantified by several results that determine whether transformers that obey different parameter scaling rules can compute them.
We define the following scaling regimes for the depth $L$, embedding dimension $m$, and number of ``pause'' tokens $N'$ of a family of transformers as a function of the size of the input graph, $N = O(|V| + |E|)$.
% \begin{enumerate}[label=(\Alph*)]
\begin{itemize}[left=0pt]
    \item %\textbf{Single-layer, small-width:} 
    \depthone{} (\sfmath{D1}): Single-layer multi-headed transformers with small embedding dimension $m = O(\log N)$ without any pause tokens.
    \item %\textbf{Logarithmic-depth, small-width:} 
    \logdepth{} (\sfmath{LD}): Transformers with depth $L = O(\log N)$, embedding dimension $m = O(N^{\epsilon})$ for any fixed $\epsilon > 0$, and no pause tokens.
    \item %\textbf{Logarithmic-depth, small-width, pause tokens:} 
    \logdepthpause{} (\sfmath{LDP}): Transformers with the same depth and width constraints as \logdepth{}, with at most $N' = \poly(N)$ blank ``pause'' tokens appended to the input sequence.
    \item %\textbf{Logarithmic-depth, large-width:} 
    \logdepthwide{} (\sfmath{LDW}): Transformers with depth $L = O(\log N)$, embedding dimension $m = O(N^{1/2 + \epsilon})$, and no pause tokens.
    % \item \textbf{Sublinear-depth, small-width, pause tokens:} Transformers with depth $L = o(N)$, embedding dimension $m = O(N^{\epsilon})$, and $N' = \poly(N)$ pause tokens.
\end{itemize}
Our positive and negative results that relate scaling regimes and graph reasoning tasks are displayed in \Cref{tab:theory}.
We present high-level result summaries in the following sections with proofs in Appendix.

\begin{figure}[t!]
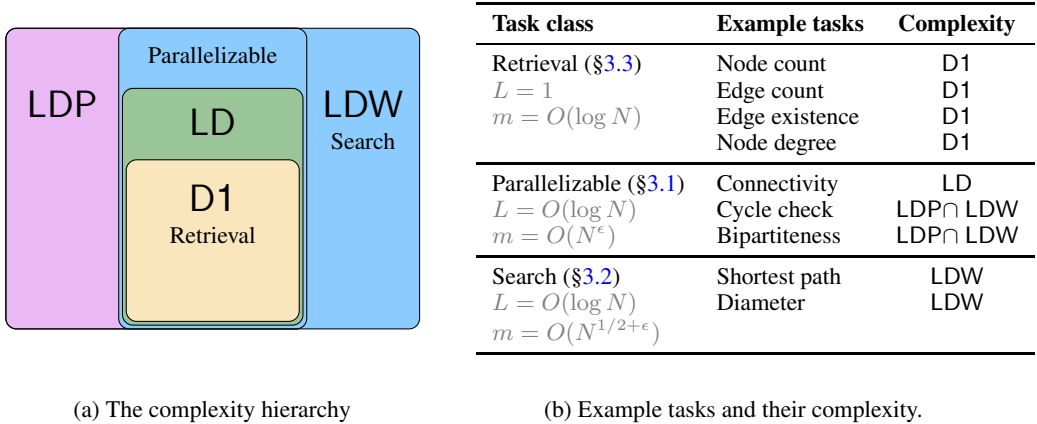

    \footnotesize
    \centering
    %\hspace{-16mm}
    \begin{subfigure}[b]{0.4\textwidth}
        \centering
       \include{figures/complexity_hierarchy_alt}
        %\vspace{-2.2em}
        \caption{The complexity hierarchy}
        \label{sfig:hierarchy}
    \end{subfigure}
    % \hfill
    \qquad
    \begin{subfigure}[b]{0.49\textwidth}
        \centering
        \include{figures/complexity_table}
    % \caption{A summary of the theoretical hierarchy of \Cref{sec:hierarchy} that visualizes which type of graph reasoning tasks can be solved in which transformer scaling regime.}
        \vspace{0.5em}
        \caption{Example tasks and their complexity.}
        \label{stab:theory}    
    \end{subfigure}
    \caption{A summary of the theoretical hierarchy of \Cref{sec:hierarchy} that visualizes which type of graph reasoning tasks can be solved in which transformer scaling regime (\depthone{} (\sfmath{D1}), \logdepth{} (\sfmath{LD}), \logdepthwide{} (\sfmath{LDW}) and \logdepthpause{} (\sfmath{LDP})).}
    \label{tab:theory}
\end{figure}

The most technically significant results concern \logdepth{} models and are provided in \Cref{asec:transformer-mpc}.
These bounds are consequences of \Cref{thm:tight-mpc-informal}, an improved analysis of the relationship between transformers and the massively parallel computation (MPC) distributed computing model of~\cite{ksv10}\footnote{
MPC is a theoretical model of the MapReduce computational paradigm that distributed computation among a large number of machines with restricted local memory size and communication bandwidth. 
We formally introduce the MPC model in \Cref{assec:mpc-prelim}.
}.
This connection between MPC and transformers is a sharp improvement of a similar result by \cite{sht24}.% that provides positive results for a wider range of graph reasoning tasks. 
\begin{theorem}[Simplified version of \Cref{thm:tight-mpc}]\label{thm:tight-mpc-informal}
% For constant $\delta, \epsilon >0$, any $R$-round MPC protocol with $N$ machines with $O(N^\delta)$ bits of local memory each can be simulated by a transformer $f \in \trans{m, 1, L}{N}$ of depth $L = O(R)$ and embedding dimension $m = O(N^{\delta + \epsilon})$
For constant $\delta, \epsilon >0$, any $R$-round MPC protocol with $N$ machines with $O(N^\delta)$ bits of local memory each can be simulated by a transformer of depth $L = O(R)$ and embedding dimension $m = O(N^{\delta + \epsilon})$.
\end{theorem}

Results pertaining to \depthone{} are stated in detail and proved in \Cref{asec:1L}.
In addition, we discuss the triangle counting task (and the more general clique counting task) in \Cref{ssec:triangle}, where we show a distinct result for much smaller depth ($L = O(\log\log N)$) that includes pause tokens.

\subsection{Parallelizable tasks and \logdepth{} transformers}\label{ssec:med-tasks}

We define the family of \emph{parallelizable tasks} to consist of graph reasoning tasks that are $\L$-complete and are equivalent to graph connectivity in $O(1)$-rounds, as proved by \cite{ns22}.
This family includes (but is not limited to): graph connectivity, minimum spanning forest, cycle detection, $st$-connectivity, number of connected components, graph bipartiteness, planarity testing, and one-cycle vs two-cycles testing.
While graph connectivity and minimum spanning forest were shown to be computable by logarithmic depth transformers with small polynomial width (\logdepth{}) by previous work \cite{sht24}, this poses broader questions:
Are all parallelizable graph reasoning tasks computable by a transformer of logarithmic depth and small embedding dimension? And do sub-logarithmic-depth transformers exist that solve any parallelizable tasks?

% the graph reasoning abilities of transformers on other parallelizable tasks were unknown.
We first show that all parallelizable tasks can be solved in two logarithmic-depth scaling settings.% (\logdepthpause{} and \logdepthwide{}).% (and therefore, trivially in regime (E)).

\begin{theorem}
% [Regimes (C) and (D) solve parallelizable tasks]
\label{thm:parallelizable-pos}
% For any $\epsilon \in (0, 1)$ and any parallelizable task, there exists a transformer $f \in \trans{m, H, L}{N, N'}$ that computes the solution to the task and has depth $L = O(\log N)$, heads $H = O(1)$ and either embedding dimension $m = O(N^\epsilon)$ and pause tokens $N' = O(N^4)$ [regime (C)], or $m = O(N^{1/2 + \epsilon})$ and $N' = 0$ [regime (D)].
% with embedding dimension $m = O(n^\epsilon)$, heads $H = O(1)$, depth $L = O(\log n)$, and chain-of-thought tokens $n' = O(n^4)$ that computes the solution to the problem.
For any parallelizable task, there exists transformers in \logdepthpause{} and \logdepthpause{} that solve the task.% and has depth $L = O(\log N)$ and embedding dimension $m$ and pause tokens $N'$ satisfying either: 
% embedding dimension $m = O(N^\epsilon)$ and pause tokens $N' = O(N^4)$ [regime (C)]; or $m = O(N^{1/2 + \epsilon})$ and $N' = 0$ [regime (D)].
% \begin{itemize}[nosep]
%     \item \logdepthpause{}:  $m = O(N^\epsilon)$ and $N' = O(N^4)$; or 
%     \item \logdepthwide{}: $m = O(N^{1/2 + \epsilon})$ and $N' = 0$.
% \end{itemize}
\end{theorem}

\Cref{thm:parallelizable-pos} is stated formally as \Cref{thm:parallelizable-pos-formal} and proved in \Cref{assec:med-tasks}.
Both components of the theorem are a consequence of a novel relationship between the MPC model and transformers (\Cref{thm:tight-mpc-informal}) and the analysis of MPC protocols for graph reasoning tasks by \cite{ns22}.
The \logdepthpause{} result is the direct implication of an $O(1)$-round MPC equivalence (\Cref{thm:ns22-med}) between all parallelizable tasks and an MPC protocol that solves the connectivity task (\Cref{thm:cc22-connectivity}). % and the fact that graph connectivity can be solved with depth $O(\log N)$ and embedding dimension $O(N^{\epsilon})$.
The \logdepthwide{} bound is a consequence of \Cref{thm:nc-mpc-ns22}, which shows that all languages in $\NC^2$ (including all languages in $\L$ and $\NL$) can be evaluated by an MPC protocol with $O(\log N)$ rounds and with local memory $O(N^{1/2 + \epsilon})$~\cite{ns22}.

We further prove the conditional optimality of logarithmic depth.

\begin{restatable}
% [Optimality of log-depth for parallelizable tasks]
{theorem}{corgraphloglb-informal}\label{thm:graph-log-lb-informal}
Conditional on \Cref{conj:1vs2}, any transformer that solves some parallelizable task with width $mH = O(N^\epsilon)$ and pause tokens $N' = \poly(N)$ must have depth $L = \Omega(\log N)$.
% $f \in \trans{m, H, L}{n, n'}$ with width $mH = O(N^\epsilon)$ for some $\epsilon \in (0, 1)$ and chain-of-thought tokens $n' = n^{O(1)}$ that computes any of the above problems must have depth $L = \Omega(\log n)$.
\end{restatable}

% \begin{restatable}{corollary}{corgraphloglb}[Optimality of log-depth for parallelizable tasks]\label{cor:graph-log-lb}
% If the one-cycle vs two-cycle conjecture holds, any transformer $f \in \trans{m, H, L}{n, n'}$ with width $mH = O(N^\epsilon)$ for some $\epsilon \in (0, 1)$ and chain-of-thought tokens $n' = n^{O(1)}$ that computes any of the above problems must have depth $L = \Omega(\log n)$.
% \end{restatable}

This result (stated formally as \Cref{thm:graph-log-lb} is a combination of the conditional depth lower bound on graph connectivity by \cite{sht24} and the $O(1)$-round equivalence between all parallel tasks.

\subsection{Search tasks}\label{ssec:hard-tasks}

\textit{Search tasks} are similarly defined to be those that are $\NL$-complete and equivalent to shortest path in $O(1)$-rounds of MPC and include shortest path, strong connectivity, $st$-reachability, radius, diameter, and median.
Like before, the $O(1)$-round MPC equivalence translates to an $O(1)$-depth equivalence in transformers. 
We give a similar positive result for \logdepthwide{} transformers; whether these tasks can be solved by \logdepthpause{} transformers is unknown.
% We give a similar positive result for regimes (D) and (E); whether these tasks can be solved in regime (C) is unknown.
% \todo{consider eliminating regime E}

\begin{theorem}
% [Regimes (D) and (E) solve search tasks]
\label{thm:search-pos}
% For any $\epsilon \in (0, 1)$ and any search task, there exists a transformer $f \in \trans{m, H, L}{N, N'}$ that computes the solution to the task and has heads $H = O(1)$ and either depth $L = O(\log N)$, embedding dimension $m = O(N^{1/2 + \epsilon})$ and no pause tokens $N' = 0$ [regime (D)]; or $L = O(N^{2/3})$, $m = O(N^\epsilon)$, and $N' = O(N^3)$ [regime (E)].
For any search task, there exists a transformer in $\logdepthwide{}$ that sovles the task.
%that solves the task and has (Regime D) depth $L = O(\log N)$, embedding dimension $m = O(N^{1/2 + \epsilon})$, and no pause tokens $N' = 0$.
%depth $L = O(\log N)$, embedding dimension $m = O(N^{1/2 + \epsilon})$, and no pause tokens $N' = 0$ [regime (D)]; or $L = O(N^{2/3})$, $m = O(N^\epsilon)$, and $N' = O(N^3)$ [regime (E)].
% \begin{itemize}[nosep]
%     \item  or 
%     \item (Regime E) $L = O(N^{2/3})$, $m = O(N^\epsilon)$, and $N' = O(N^3)$.
% \end{itemize}
% with embedding dimension $m = O(n^\epsilon)$, heads $H = O(1)$, depth $L = O(\log n)$, and chain-of-thought tokens $n' = O(n^4)$ that computes the solution to the problem.
\end{theorem}

This theorem, which is restated in \Cref{assec:hard-tasks} as \Cref{thm:search-pos-formal}, is also an immediate consequence of \Cref{thm:nc-mpc-ns22}.

While the minimum depth of a transformer with small embedding dimension that solves a search task is not identified, we prove that the minimum depth needed to solve some all search task is approximately equivalent in \Cref{thm:search-equiv}.

% The regime (D) component is also an immediate consequence of \Cref{thm:tran-nc}.
% The regime (E) statement follows from the MPC round equivalence and the existence of an $O(N^{2/3})$-round protocol for shortest path in \Cref{cor:search-deep}.

\subsection{Retrieval tasks and \depthone{} transformers}\label{ssec:easy-tasks}

Graph tasks whose algorithms consist of a single look-up or aggregation step can be efficiently solved by single-layer transformers.
This result assumes that the graph $G = (V, E)$ is encoded as some input sequence $X$ of length $N = O(|V| + |E|)$ that expresses each edge and vertex a single token. 
This encoding scheme is detailed in \Cref{asec:1L}.

\begin{restatable}
% [Regime (A) solves retrieval tasks]
{theorem}{thmposoneLinformal}\label{thm:pos-1L-informal}
For any retrieval task (including node count, edge count, edge existence, and node degree) there exists a transformer in \depthone{} that solves the task.
%single-layer single-headed transformer with embedding dimension $m = O(\log N)$ that solves the task.% on all graphs of size $|V| + |E| \leq N-1$.
\end{restatable}

We formalize this statement in \Cref{thm:pos-1L} and prove it their in \Cref{asec:pos-1L}.
These rely on proving the existence of a useful input MLP $\phi$ that precomputes embeddings with useful structure for all retrieval tasks.
% We encode an input graph $G = (V, E)$ as a \emph{node/edge embedding sequence} input to the transformer, where each token embedding is either (1) a fixed embedding of a single vertex $v \in V$ with the \texttt{isVertex} annotation and a positional encoding, (2) a concatenation of two vertex embeddings representing an edge $(u,v) \in E$ with the \texttt{isEdge} annotation and a positional encoding, or (3) some encoding of the input query with the \texttt{isQuery} annotation.
% This encoding scheme closely resembles the one used in experimental results, save for the fact that here, edges are represented as single tokens rather than pairs of tokens.
% We define and discuss the node/edge embedding sequence in greater detail in \Cref{asec:1L}.

% Individual theorem statements for each task and their proofs appear in \Cref{asec:1L}.

\iffalse
Note that these results only pertain to approximation. 
While small transformers exist that solve tasks like node degree, they imply little about whether those tasks are learnable.
In the previous section, we conjectured that the failure of transformer models to solve node degree as effectively as GNN-based models is due the difficulty of associating degree outputs with corresponding edges.
In contrast, the communication restrictions of GNNs between neighboring nodes is a favorable inductive bias for such locality-based graph reasoning tasks.
\fi

In contrast, we show that a collection of parallelizable and search tasks cannot be efficiently solved by transformers in \depthone{}.

\begin{restatable}
% [Regime (A) does not solve harder tasks]
{theorem}{thmnegoneLinformal}\label{thm:neg-1L-informal}
Any single-layer transformer that solves the graph connectivity, shortest path, or cycle detection task 
% on all graphs of size $|V| + |E| = O(N)$ 
has width satisfying $mH = \tilde\Omega(N)$.
\end{restatable}

The proof of the formal counterpart of this statement (\Cref{thm:neg-1L}) appears in \Cref{asec:neg-1L} and is a consequence of a standard communication complexity argument.
A more generalized result than \Cref{thm:neg-1L-informal} was proved by \cite{mss21}, which establishes that all problems outside of $\NC^1$---which include all $\L$-complete and $\NL$-complete languages, and hence, all parallelizable and search tasks---cannot be solved by \textit{constant-depth} transformers with polynomial width because they cannot be computed by $\textsf{TC}^0$ (constant-depth threshold circuits). 
We nonetheless include this theorem to demonstrate a clean lower bound that applies to very simple input graph instances.

\subsection{Triangle counting}\label{ssec:triangle}

% \subsection{Solving local tasks with single-layer small-width transformers}\label{ssec:single-layer}

% \subsubsection{Triangle and clique counting}\label{sssec:triangle-counting}

We finally construct depth-efficient transformers for triangle counting due to the MPC algorithms of \cite{belmr21}.
Unlike previous positive results, which applied uniformly across all graphs instances of bounded size, the complexity of the corresponding transformers for triangle counting is a function of the arboricity\footnote{
The \textit{arboricity} is the minimum number of spanning forests needed to cover the graph, which grows at most linearly with the degree of the graph.
} of the input graph.
When the arboricity grows sub-polynomially with $N$---as is the case for bounded-degree graphs---no pause tokens are necessary.
Unlike the parallelizable and search classes of problems, strictly sub-logarithmic depth is attainable with pause tokens, even for worst-case graphs.

% By combining their Theorem~1.5 with \Cref{thm:tight-mpc}, we obtain the following.

% \begin{corollary}\label{cor:triangle}
% For any $\epsilon \in (0, 1)$ and $k \geq 3$, there exists a transformer that computes the number of triangles in an input graph of arboricity $\alpha$ and has embedding dimension $m = O(n^\epsilon)$, depth $L = O(\log\log n)$, and $N' = O(\alpha N^{1 - \epsilon})$ pause tokens of $\alpha = \Omega(N^\epsilon)$ and $N' = 0$ otherwise.% $n' = \max(0, O(\alpha \cdot n^{1-\epsilon}) - n)$.
% \end{corollary}

\begin{theorem}\label{cor:triangle-informal}
There exists a transformer that computes the number of triangles in any input graph of arboricity $\alpha$ and has embedding dimension $m = O(N^\epsilon)$, depth $L = O(\log\log N)$, and pause tokens \[N' = \begin{cases}O(\alpha N^{1-\epsilon}) & \text{if $\alpha = \Omega(N^\epsilon)$} \\ 0 & \text{otherwise.}\end{cases}\]
% $N' = O(\alpha N^{1 - \epsilon})$ if $\alpha = \Omega(N^\epsilon)$ and $N' = 0$ otherwise.% $n' = \max(0, O(\alpha \cdot n^{1-\epsilon}) - n)$.
\end{theorem}

% In the event that $\alpha = o(n^\epsilon)$ (which is the case in constant-degree graphs), the transformer does not require any extra blank tokens.
This result is a special case of \Cref{cor:clique}, a more general result about clique counting that appears in \Cref{assec:triangle}.

% \todo{To appendix}

% By importing Table 1 of \cite{belmr21} with graphs with $n$ nodes, $N$ edges, and arboricity $\alpha$, there exist transformers of the following sizes that perform exact triangle counting:
% \begin{itemize}
%     \item Depth $L = O(1)$, embedding dimension $m = O(N^{1/2 + \epsilon})$.
%     \item Depth $L = O(\log\log n)$, embedding dimension $m = O(N^{\epsilon})$.
%     \item Depth $L = O(\log n)$, embedding dimension $m = O(\alpha^2 \cdot N^{\epsilon})$. (There may be a better dependence on $N$?)
% \end{itemize}

% \todo{Relevant experiments on size and sample complexity? The hope is that increasing the depth helps a more for harder tasks wrt this hierarchy}

\paragraph*{Theoretical conclusions.}

These results provide a tight characterization of the the reasoning capabilities of transformers whose depth, width, and input padding conform to different scaling regimes.
They strengthen the established connection between transformers and massively parallel computation (MPC) \cite{sht24} and generalize the resulting representational bounds to broader categories of graph tasks.
We conclude that the logarithmic-depth regime is apt for for considering tasks in $\L$ and $\NL$, which had previous illuminated the limitations of transformers with constant depth and a limited number of chain-of-thought tokens \cite{ms23-cot}. 
While expressivity does not imply learnability, these theoretical benchmarks sharply characterize the fundamental limitations of transformers and coincide with experimental results conveyed in the subsequent section.

%% file: figures/complexity_hierarchy_alt.tex
\begin{tikzpicture}
\draw[rounded corners, fill=cycle6!45, fill opacity=0.5] (-1.5,0) rectangle (2.5, 4);
\node at (-0.75,3) {\Large\sfmath{LDP}};

\draw[rounded corners, fill=cycle3!45, fill opacity=0.5] (0,0) rectangle (4, 4);
\node at (3.25,3) {\Large\sfmath{LDW}};
\node at (3.25,2.5) {Search};

\draw[rounded corners] (-1.5,0) rectangle (2.5, 4);

\node at (1.25,3.65) {Parallelizable};

\draw[rounded corners, fill=cycle4!40] (0.05,0.05) rectangle (2.45,3.2);
\node at (1.25,2.75) {\Large\sfmath{LD}};

\draw[rounded corners, fill=cycle1!30] (0.1,0.1) rectangle (2.4,2.25);
\node at (1.25,1.75) {\Large\sfmath{D1}};
\node at (1.25,1.25) {Retrieval};

\end{tikzpicture}

%% file: figures/complexity_table.tex
\begin{tabular}{llc}
\toprule
    %  Task class & Example tasks & Single-layer & Log depth & Log depth & Log depth & Log depth  \\
    %  & & Small embed. & Small embed. & Small embed. & Large embed. & Small embed. \\
    %  & & No blanks & No blanks & Poly blanks & No blanks & Poly blanks \\
     \textbf{Task class} & \textbf{Example tasks}& \textbf{Complexity}\\ 
     \midrule
     Retrieval (\S\ref{ssec:easy-tasks})  & Node count & \DOne\\
     \gray{$L=1$}& Edge count & \DOne\\
     \gray{$m=O(\log N)$} & Edge existence & \DOne\\
     & Node degree & \DOne\\
     \midrule 
     Parallelizable (\S\ref{ssec:med-tasks}) & Connectivity & \LD\\
     \gray{$L=O(\log N)$} & Cycle check & \LDP $\cap$ \LDW \\
     \gray{$m=O(N^\epsilon)$} & Bipartiteness & \LDP $\cap$ \LDW \\
     \midrule 
     Search (\S\ref{ssec:hard-tasks}) &  Shortest path & \LDW \\
     \gray{$L=O(\log N)$} & Diameter & \LDW\\
     \gray{$m=O(N^{1/2 +\epsilon})$}&& \\
     \bottomrule
\end{tabular}

%% file: empirical.tex
% \section{Contrasting transformers and GNN capabilities with graph reasoning tasks}\label{sec:transformers-gnns}

\section{Empirical graph reasoning capabilities}\label{sec:transformers-gnns}

We further illuminate the reasoning capabilities of transformers by conducting an empirical investigation of the abilities of a variety of neural architecture and training settings to learn graph algorithmic tasks.
We use the GraphQA benchmark tasks~\cite{fhp23} for our experiments. We evaluate standard autoregressive transformers---both small models (at most 60M parameters) trained from scratch and fine-tuned (FT) T5-11B model~(with 11B parameters)~\cite{raffel2020exploring}. For the fine-tuned models, we explore task-specific fine-tuning---and contrast those results with graph neural networks (GNNs) and prompting-based methods on pre-trained LLMs.

These experimental results validate key tenets of our theoretical results and demonstrate the utility of transformers' algorithmic reasoning capabilities.
Our principal empirical conclusions are as follows:
\begin{enumerate}[left=0pt]
    \item \textbf{Transformers excel at global reasoning tasks.} 
    Transformers outperform GNNs on tasks that require efficiently aggregating information about distant nodes in a graph, such as connectivity and shortest path.
    \item \textbf{GNNs uncover local graph structure with few samples.} 
    While transformers are capable of efficiently expressing all graph learning tasks under investigating, the structural limitations of GNNs provide them with favorable inductive biases for intrinsically local tasks, such as cycle check and node degree, and permit them to outperform transformers in a low-sample regime.
    \item \textbf{Trained transformers outperform LLM prompting.}  
    Transformers trained to explicitly solve graph reasoning tasks consistently attain greater accuracy across tasks than a variety of prompting strategies applied to more recent larger LMs. 
\end{enumerate}
A comprehensive evaluation of each GraphQA task on every training setting appears in \Cref{asec:empirical}, in addition details about transformer training, the GraphQA benchmark, and alternative GNN and prompting approaches.

\subsection{Transformers excel at global reasoning tasks}\label{ssec:transformer-win}

As indicated in \Cref{sec:hierarchy}, graph reasoning algorithms can be categorized based on the extent to which they entail aggregating ``local'' information about nodes and their immediate neighbors or modeling ``global'' connections between nodes separated by a long distances.
This section investigates the following question about transformers and long-distance reasoning:
\begin{quote}
    \textit{When do transformers outperform GNNs on tasks that require global reasoning?}
\end{quote}
We consider two tasks that require reasoning across long distances in a graph instance: evaluating \textbf{connectivity} and computing the \textbf{shortest path} between a pair of nodes.
Neither of these tasks can be solved by only investigating the neighbors of the source and sink node, which therefore implies that some analysis of global graph structure is necessary.

\input{connectivity_figure}

% \begin{wrapfigure}[14]{r}{0.35\textwidth}
% \input{figures/chart_classification_accuracy}
% \vspace*{-2mm}
% \caption{Connectivity accuracy on trained transformers as a function of sample complexity.}\label{fig:empirical:connectivity}
% \end{wrapfigure}

\input{shortest_path_table}

\Cref{tab:connectivty}  displays the accuracy of a variety of trained transformers and GNNs on the connectivity task contrastsing the performance of all such models when trained on 1,000 and 100,000 graph connectivity instances.
In the most restricted sample complexity regime, trained GNNs are consistently more accurate than the small transformer; however, increasing the number of training samples yields a far more substantial improvement in the performance of the small transformer, which outperforms all GNNs trained on 100,000 samples.
Notably, the pre-trained transformer, fine-tuned on just 1000 training instances, nearly solves the connectivity task. This suggests significant enhancements due to the larger model size and the data-rich fine-tuning phase. % that the pre-training phase has significantly enhanced the model's performance, providing further evidence that transformers are highly effective in high-data regimes.}
\Cref{fig:connectivity} plots the training and test error of ten small transformers trained on connectivity datasets of increasing size and reveals a sharp and continual improvement in accuracy.
The fine-tuned T5 transformer has similar accuracy to the most sample-rich small transformer and exceeds that of all GNNs.

% \begin{table}[]
%     \centering
%     \begin{tabular}{c|c}
%          &  \\
%          & 
%     \end{tabular}
%     \todo{A table the follows the same form as the reachability table. Title: ``Shortest Path Accuracy.'' Columns: \# Samples. Rows: Transformer (small), T5, GraphToken, MPNN, GIN}
%     \caption{Transformers vs GNNs on shortest path: Fine-tuned large transformers outperform other transformers and GNNs, even the alternatives are trained on much larger training sets.}
%     \label{tab:shortest_path}
% \end{table}

On the other hand, \Cref{tab:shortest_path} demonstrates that the MPNN GNN models outperform small transformers when trained to compute the shortest path, even on larger training datasets.
However, the fine-tuned T5 model has far higher accuracy than all alternatives, even when trained only 1000 samples.

\paragraph*{Theoretical interpretation:}
Because connectivity is the prototypical example of a task in the parallelizable class and can thus be efficiently implemented by \logdepth{} transformers with very small width (\Cref{thm:parallelizable-pos}), the fact that small transformers succeed in solving nearly all connectivity instances is illuminating but not surprising. 
In contrast, message-passing GNNs are unable to solve connectivity in a similarly depth-efficient manner due to fundamental capacity limitations.\footnote{
See \Cref{ssec:appendix-graphreasoning} for a discussion of prior work on theoretical relationships between GNNs and both the Weisfeiler-Leman graph isomorphism test \cite{xhlj18} and the \congest{} distributed computing model and their implication that GNNs cannot solve connectivity in a parameter-efficient manner.
}

Shortest path belongs to the search class of graph reasoning tasks and is $\NL$-complete.
\Cref{thm:search-pos} implies that shortest path is computable by \logdepthwide{} transformers, which are likely to require very large embedding dimension to be learnable by finite samples. 
This task can only be computed in a depth- and parameter-efficient manner if a variety of search tasks, including all-pairs shortest-path and diameter, are as well (\Cref{thm:search-equiv}).
The fact that only the pre-trained model has nearly optimal performance on shortest path reinforces the theoretical intuition that solving shortest path requires a very large number of model parameters.

% - Point 1: Transformers get excellent performance on connectivity (a ``global'' task in \logdepth{}) and dominate GNNs in a high sample regime

% - Table 1. Columns: 1k vs 100k samples. Rows: Transformer, Graphtoken, MPNN, GIN. Cells: Reachability accuracy. 

% - Takeaway 1: GNNs win at 1k, transformers win at 100k.

% - Figure 1. Sample complexity plot of transformers on reachability.

% - Theory interp 1: Transformers have a parameter-efficient implementation of connectivity (Thm 4). GNNs do not (Appendix with GNN limitations). Representation is not learnability, but we can attain nearly opt solution with enough samples.

% - Point 2: Transformers outperform GNNs on shortest path, but only in very high-parameter regime.

% - Table 2. Rows: Transformer 1k, PALM 1k, Graphtoken 1k, MPNN 1k, GIN 1k, GCN 1k. Cells: Shortest path accuracy.

% - Takeaway 2: Transformers win when we're fine-tuning large models at same sample complexity.

% - Theory interp 2: Shortest path is in \logdepthwide{}, so it makes sense that we'd need a large number of parameters

\subsection{GNNs uncover local graph structure with few samples}\label{ssec:gnn-win}

While small transformers outperform GNNs on graph reasoning algorithms that entail analysis of long-range graph structure, their empirical successes are not uniform.
Here, we investigate the following question:

\begin{quote}
    \textit{When do GNNs outperform transformers on graph reasoning tasks?}
\end{quote}

% \begin{table}[]
%     \centering
%     \begin{tabular}{c|c}
%          &  \\
%          & 
%     \end{tabular}
%     \todo{Title: ``Task accuracies.'' Columns: Node Degree (1k, 100k), Cycle Check (1k, 100k). Rows: Transformer (small), GraphToken, MPNN, GIN}
%     \caption{Transformers vs GNNs on cycle check and node degree: GNNs have favorable inductive biases for extracting local structure.}
%     \label{tab:local}
% \end{table}

\begin{wraptable}[12]{r}{0.6\textwidth}
\centering
%\resizebox{\textwidth}{!}{%
% \setlength{\tabcolsep}{4pt}
\vspace*{-2.5ex}
\begin{tabular}{lcccc}
\toprule
& \multicolumn{2}{c}{\textbf{Node Degree}} & \multicolumn{2}{c}{\textbf{Cycle Check}} \\
\cmidrule(lr){2-3}  \cmidrule(lr){4-5}
\footnotesize
     \textbf{Model} & \textbf{1K}& \textbf{100K} & \textbf{1K}& \textbf{100K}\\ 
     \midrule
    GCN~\cite{kipf2016semi} & 9.8 & 9.4 & 83.2 & 83.2\\
    MPNN~\cite{gilmer2017neural} & 99.4 & \textbf{99.8} & 99.0 & \textbf{100.0} \\
    GIN~\cite{xhlj18} & 36.2 & 37.8 & 98.8 &  83.2\\
    % GraphToken~\cite{pfztkah24}  & 70.2 & 72.0 \\
    \midrule
    % Transformer 
    60M transformer &  31.6 & 91.7 & 97.1 & 98.0 \\
    % Transformer (pre-trained) 
    11B transformer (FT) & 68.8 & --- & 98.0 & --- \\
    \bottomrule
\vspace*{-2.5ex}
\caption{Transformers vs GNNs on cycle check and node degree: GNNs are favorably biased for  local structure.}
\label{tab:local}
\end{tabular}
%}
\end{wraptable}

\Cref{fig:connectivity} and \Cref{tab:shortest_path} demonstrate that GNNs outperform small transformers in the low-sample regime, despite the sufficient expressivity of transformers.
This gap in performance, which is reinforced for the \textbf{node degree} and \textbf{cycle check} tasks in \Cref{tab:local}, suggests that GNNs have a beneficial \textit{inductive bias} for learning graph reasoning tasks that can be solved by attending exclusively to local heuristics.

Just as the bounded kernel size of convolutional neural networks (CNNs) enables the sample-efficient learning of relevant textures and gradients for image analysis, message-passing GNNs are unable to send messages instantaneously across multiple edges, which simplifies a search of the space of ``one-hop'' graph algorithms and represents a positive inductive bias.
In contrast, the ability of transformers to send information between any pair of input tokens---and the alternative inductive biases suggested by the input positional encoding---likely induces a steeper sample complexity to learn node degree.
% This is particularly notable for 

% - Counterpoint to H1: On other tasks, GNNs dominate transformers, even in high-sample regime

% - Table: Columns: Node degree, Cycle check; 1k vs 100k. Rows: GNNs and transformers. Cells: Accuracies.

% - Takeaway: Despite the fact that transformers can solve these tasks, GNNs do better. 

% - Theory interp: Representation isn't the full picture. Message-passing GNNs that are restricted by neighbors are more likely to have parameter-efficient constructions of these.

% - Caveat: In the hardest case, cycle check should be hard for GNNs. We suspect we're not in this case and that the cycles are realistically of constant lenght. (CONSIDER LOOKING INTO DATASET)

\paragraph*{Theoretical interpretation:}
While model expressivity is necessary for learnability, it is not sufficient.
% Unlike graph connectivity and shortest path, message-passing GNNs are capable of solving node degree in an efficient manner without requiring the propagation of information across the input graph. 
The locality constraints of message-passing GNNs likely provides a favorable inductive bias for learning tasks like node degree with an exclusively on local structure that makes learning these tasks possible in a sample-efficient manner.
While cycle check is more representationally difficult for GNNs than transformers in the worst case (see \Cref{ssec:appendix-graphreasoning}), the random graphs sampled for the GraphQA benchmark have very small cycles (\Cref{fig:appendix:cycle_length}) and do not resemble the large-diameter worst-case instance.% large-diameter graphs that cannot be easily solved by GNNs. \todo{Kinda sketchy\dots}

\subsection{Trained transformers outperform LLM prompting}\label{ssec:prompting-lose}

Large language models (LLMs) are regularly evaluated by their reasoning abilities, and it remains an open research question to determine what kinds of training data best teaches models to solve logical problems.
We investigate the extent to which LLMs can already perform graph reasoning tasks without being trained explicitly to do so.
\begin{quote}
    \textit{Do transformers trained explicitly to solve graph reasoning tasks outperform prompt-tuning approaches on much larger LLMs?}
\end{quote}

\input{table-llm}

In \Cref{table:llm}, we contrast the capabilities of trained transformer models with several prompt-based approaches to querying LLMs.
Task-specific transformers---including the fine-tuned 11B transformers---consistently dominated the prompt-based approaches, despite the vast difference in parameter count and the almost certain presence of graph reasoning in the LLM's corpus.%and the fact that the trained LLMs were almost certainly exposed to both task instances and analysis in its corpus.
% This includes the large-scale transformer model fine-tuned on just 1k samples, however, it does not always outperform specialized models, especially in our hardest problem of triangle counting.

\paragraph*{Theoretical interpretation:}
While the representational capabilities of LLMs to solve reasoning tasks is much greater than that of small transformers, this performance gap suggests that their effective reasoning capacity is much weaker and that it may be improved by a richer training corpus that includes synthetic tasks.

Finally, we observe that the near-perfect performance of trained transformers on the node count, edge count, and edge existence is consistent with the representational easy of those tasks, as suggested by the existence of efficient \depthone{} transformer implementations.

%% file: connectivity_figure.tex
\begin{figure}[t]
%\footnotesize
\centering
%\hspace{-16mm}
    % \hspace{-16mm}
    \begin{subfigure}[b]{0.45\textwidth}
        \centering
        %\resizebox{\textwidth}{!}{%
    % % \setlength{\tabcolsep}{4pt}
        \begin{tabular}{lcc}
        %\small
        \toprule
                % \footnotesize
                & \multicolumn{2}{c}{\textbf{\# of training samples}} \\
                \cmidrule(lr){2-3}
                \textbf{Model} & \textbf{1K}& \textbf{100K}\\ 
                     \midrule
                    GCN~\cite{kipf2016semi} & 83.8 & 83.8\\
                    MPNN~\cite{gilmer2017neural} & 94.0 & 94.4\\
                    GIN~\cite{xhlj18} & 93.8 & 94.0\\
                %     % GraphToken~\cite{pfztkah24} & 93.6 & 93.8\\
                    \midrule
                    % Transformer
                    60M transformer
                    & 92.9 & 98.0 \\
                    % Transformer (pretrain)
                    11B transformer (FT)
                    & \textbf{98.4} & --- \\
                    \bottomrule
        \end{tabular}
        %}
        \caption{Connectivity classification accuracy for trained transformers and GNNs.}
        \label{tab:connectivty}    
    \end{subfigure}
   \qquad\quad
    \begin{subfigure}[b]{0.45\textwidth}
        \centering
       \input{figures/chart_classification_accuracy}
        \vspace{-4ex}
        \caption{Connectivity classification accuracy of 60M transformers by training set size.}
        \label{fig:connectivity}
    \end{subfigure}
    % \hfill
    \caption{Accuracy of a variety of trained transformers and GNNs on the connectivity task.}
    \label{fig:connectivity_both}
\end{figure}
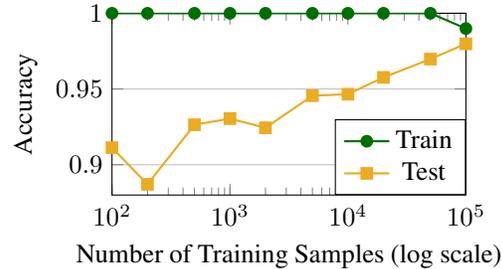

%% file: figures/chart_classification_accuracy.tex
\begin{tikzpicture}
\begin{axis}[height=4cm,width=\linewidth,ymajorgrids=true,xmin=100,xmax=100000,ymax=1,ymin=0.88,xmode=log,legend style={at={(1,0)},anchor=south east,xlabel={Number of Training Samples (log scale)},ylabel={Accuracy}}
]
\addplot[thick,color=cycle4,mark=*] coordinates {
(100, 1.0)
(200, 1.0)
(500, 1.0)
(1000, 1.0)
(2000, 1.0)
(5000, 1.0)
(10000, 1.0)
(20000, 1.0)
(50000, 1.0)
(100000, 0.9899193048477173)
};
\addplot[thick,color=cycle1,mark=square*] coordinates {
(100, 0.9112902879714966)
(200, 0.8870967626571655)
(500, 0.9264112710952759)
(1000, 0.930443525314331)
(2000, 0.9243951439857483)
(5000, 0.9455645084381104)
(10000, 0.9465725421905518)
(20000, 0.9576612710952759)
(50000, 0.9697580337524414)
(100000, 0.9798386693000793)
};
\legend{Train,Test}
\end{axis}
\end{tikzpicture}

%% file: shortest_path_table.tex
\begin{wraptable}[15]{r}{0.5\textwidth}
\centering
\vspace*{-1.5ex}
%\resizebox{\textwidth}{!}{%
\setlength{\tabcolsep}{4pt}
\begin{tabular}{lcc}
\toprule
& \multicolumn{2}{c}{\textbf{\# of training samples}} \\
\cmidrule(lr){2-3} 
     \textbf{Model} & \textbf{1K}& \textbf{100K}\\ 
     \midrule
    GCN~\cite{kipf2016semi} & 50.2 & 55.0\\
    MPNN~\cite{gilmer2017neural} & 66.8 & 72.6\\
    GIN~\cite{xhlj18} & 54.0 & 58.6 \\
    % GraphToken~\cite{pfztkah24}  & 70.2 & 72.0 \\
    \midrule
    % Transformer 
    60M transformer
    &  57.4 & \textbf{97.1} \\
    % Transformer (pretrain) 
    11B transformer (FT) & 92.8 & --- \\
    \bottomrule
\vspace*{-2.5ex}
\caption{Transformers vs GNNs on shortest path: Fine-tuned large transformers outperform other transformers and GNNs, even the alternatives are trained on much larger training sets.}
\label{tab:shortest_path}
\end{tabular}
%}
\end{wraptable}

%% file: table-llm.tex
\begin{table}[t]
% \vspace{-2em}
% \footnotesize
\resizebox{\textwidth}{!}{%
\setlength{\tabcolsep}{3pt}
\begin{tabular}{llcccccccc} 
\toprule
& & \multicolumn{4}{c}{\textbf{Retrieval tasks}} & \multicolumn{2}{c}{\textbf{Parallelizable Tasks}} & \multicolumn{1}{c}{\textbf{Search Tasks}} & \multicolumn{1}{c}{\textbf{Subgraph Counting}}\\
\cmidrule(lr){3-6} \cmidrule(lr){7-8} \cmidrule(lr){9-9} \cmidrule(lr){10-10}
% \textbf{Method} & \textbf{Node count} & \textbf{Edge count} & \textbf{Cycle check} & \textbf{Triangle counting} & \textbf{Node degree} & \textbf{Connected nodes} & \textbf{Reachability} & \textbf{Edge existence} & \textbf{Shortest path}\\ \midrule
% \acr{zero-shot} & 0.217 & 0.124 & 0.760 & 0.015 & 0.140 & 0.147 & 0.849 & 0.445 & 0.115 \\
% \acr{zero-cot} & 0.146 & 0.094 & 0.323 & 0.127 & 0.104 & 0.088 & 0.735 & 0.335 & 0.336 \\
% \acr{few-shot} & 0.253 & 0.120 & 0.374 & 0.030 & 0.174 & 0.124 & 0.794 & 0.368 & 0.227 \\
% \acr{cot} & 0.276 & 0.128 & 0.580 & 0.081 & 0.292 & 0.131 & 0.452 & 0.428 & 0.386 \\
% \acr{cot-bag} & 0.269 & 0.125 & 0.521 & 0.081 & 0.280 & 0.158 & 0.452 & 0.373 & 0.404 \\
% \acr{soft-prompt} & 0.056 & 0.018 & 0.832 & 0.162 & 0.098 & 0.068 & 0.838 & 0.544 & 0.462 \\
% GraphToken & 0.996 & 0.426 & 0.956 & 0.348 & 0.962 & 0.264 & 0.932 & 0.738 & 0.638 \\
 & \textbf{Method} & \textbf{Node count} & \textbf{Edge count} & \textbf{Edge existence} & \textbf{Node degree} &  \textbf{Connectivity} & \textbf{Cycle check} & \textbf{Shortest path} &  \textbf{Triangle counting}  \\ \midrule
\multirow{5}{*}{\begin{sideways}\scriptsize Prompting \end{sideways}} & \acr{zero-shot}~\cite{fhp23} & 21.7 & 12.4 & 44.5 & 14.0 & 84.9 & 76.0 & 11.5 & 1.5 \\
% & 0.147 \\
& \acr{zero-cot}~\cite{fhp23} & 14.6 & 9.4 & 33.5 & 10.4 & 73.5 & 32.3 & 33.6 & 12.7 \\
% & 0.088 \\
& \acr{few-shot}~\cite{fhp23} & 25.3 & 12.0 & 36.8 & 17.4 & 79.4 & 37.4 & 22.7 & 3.0 \\
% & 0.124 \\
& \acr{cot}~\cite{fhp23} & 27.6 & 12.8 & 42.8 & 29.2 & 45.2 & 58.0 & 38.6 & 8.1 \\
% & 0.131 \\
& \acr{cot-bag}~\cite{fhp23} & 26.9 & 12.5 & 37.3 & 28.0 & 45.2 & 52.1 & 40.4 & 8.1 \\
% & 0.158 \\
% & \acr{soft-prompt}~\cite{levine2022standing} & 5.6 & 1.8 & 54.4 & 9.8 & 83.8 & 83.2 & 46.2 & 16.2 \\
% & 0.068 \\
% \midrule
% \multirow{3}{*}{\begin{sideways}\scriptsize Graph-based\end{sideways}}  & GCN~\cite{kipf2016semi} & 6.4 & 1.2 & 47.0 & 9.8 & 83.8 & 83.2 & 50.2 & 4.0 \\
% % & MPNN~\cite{gilmer2017neural} & 89.8 & 66.2 & 71.4 & \textbf{99.2} & 93.4 & \textbf{98.8} & 64.8 & 33.6\\
% & MPNN~\cite{gilmer2017neural} & 19.4 & 16.2 & 69.2 & \textbf{99.4} & 94.0 & \textbf{99.0} & 66.8 & 30.6\\
% & GIN~\cite{xhlj18} & 71.2 & 4.4 & 71.2 & 36.2 & 93.8 & \underline{98.8} & 54.0 &	30.4\\
% & GraphToken~\cite{pfztkah24} & 99.6 & 42.6 & 73.8 & \underline{96.2} & 93.2 & 95.6 & 63.8 & \underline{34.8} \\
% & 0.264 \\
\midrule
\multirow{3}{*}{\begin{sideways}\scriptsize Ours \end{sideways}} &
% Transformer-1K 
60M transformer-1K
& \underline{100.0} & \underline{100.0} &	67.6 &	31.5 &	92.9 &	\underline{97.1} &	57.4 &	\underline{33.4} \\
& 
% Transformer-100K
60M transformer-100K
& \textbf{100.0} &	\textbf{100.0} &	\underline{96.1} &	\textbf{91.7} &	\underline{98.0} &	\textbf{98.0} &	\underline{97.2} &	\textbf{40.5} \\
& 
% Transformer-1K~(pretrained) 
11B transformer (FT)-1K
& \textbf{100.0} & 45.0 & \textbf{100.0} & \underline{68.8} & \textbf{98.4} & \textbf{98.0} & \textbf{92.8} & 26.0\\
% & PaLM~2~1B & 100.0 & 70.6 & 73.0 & 31.0 & 93.6 & 98.0 & 60.4 & 29.0 \\
% & PaLM~2~8B & 100.0 & 73.2 & 98.6 & 50.6 & 96.6 & 96.8 & 60.0 & 28.6 \\
\bottomrule
\end{tabular}
}
\caption{
Comparison of GraphQA task accuracies of transformers explicitly trained for graph reasoning and LLMs with a variety of prompting strategies.
}\label{table:llm}

\end{table}

%% file: conclusion.tex
\section{Conclusion}
This paper provides a comprehensive evaluation of transformer models' graph reasoning capabilities, shedding light on their effectiveness across diverse graph reasoning tasks. By introducing a novel representational hierarchy, the study distinguishes between retrieval, parallelizable, and search reasoning tasks and offers insights into the performance of transformers at varying levels of granularity.
The empirical investigation reveals that transformers exhibit strong performance in graph-based reasoning problems, often matching or surpassing specialized graph models. 
% Notably, transformers excel at retrieval tasks and regular demonstrate superior performance on search and parallelizable tasks that require examining long-range node relationships in graphs. 
Furthermore, the study highlights transformers' exceptional ability to capture global graph patterns effectively, showcasing their capability in understanding long-range dependencies, a critical factor in solving tasks involving global graph structures. 
% Overall, this research contributes to a deeper understanding of the fundamental capabilities of transformers and bridges a gap between empirical and theoretical
Overall, this work crystallizes precise representational trade-offs that reflect the fundamental reasoning capabilities of transformers and demonstrates that the tasks used to quantify those capabilities are indeed learnable in a sample-efficient and parameter-efficient manner.

% While the hierarchy introduced by this work sharply separates graph algorithmic tasks into several equivalence classes with significant implications for their computability by transformers, several questions remain for future work to answer.
% We opted to focus on graph reasoning tasks due to their relevance to a broader story about transformers, GNNs, and parallel algorithms, but there is no reason that the complexity classes presented here cannot be extended to a broader family of algorithmic problems.
% % On a more granular level, theoretical and empirical investigatations into whether the \logdepthwide{} regime is efficiently learnable 
% Furthermore, while our assumption of unbounded-size MLPs yields powerful lower bounds, future work into whether parallelizable tasks can be represented by transformers with \textit{constant-depth bounded-width MLP units} would complement existing work on the distributed computing lens on transformers.
% Broader experimental results that empirically evaluate the scaling laws would more directly assess the relevance of representational theoretical results to learnability. 

While the hierarchy introduced by this work effectively separates graph algorithmic tasks into distinct equivalence classes with significant implications for their computability by transformers, several questions remain for future research. 
We focused on graph reasoning tasks due to their relevance to the broader context of transformers, GNNs, and parallel algorithms. 
However, the complexity classes presented here could potentially be extended to a wider range of algorithmic problems.
While our assumption of unbounded-size MLPs provides strong lower bounds, further research into whether parallelizable tasks can be represented by transformers with bounded-size MLP units would complement this existing work. 
Broader experimental results that empirically evaluate the scaling laws would more directly assess the relevance of representational theoretical results to learnability.

%% file: prelims.tex
\section{Theoretical preliminaries}\label{sec:prelims}

For some vector $v \in \R^N$, let \[\sm(v) = \frac{1}{\sum_{i=1}^N \exp(v_i)} \paren{\exp(v_1), \dots, \exp(v_N)} \in \R^N.\]
Let $e_i \in \R^m$ denote the $i$th elementary vector, $\vec1 = (1, \dots, 1) \in \R^m$ and $\vec0 = (0, \dots, 0) \in \R^m$.
Let $[n] = \{1, \dots, n\}$.

\subsection{Transformer models}\label{assec:prelim-trans}

We use a similar theoretical definition of a multi-layer transformer model to the one employed by \cite{sht24}.
We use a \textit{bounded communication} theoretical model of the standard bidirectional transformer of \cite{vsp17}, which assumes that the principal limitation of the transformer's representational power is the low rank of the self-attention matrix.
Several theoretical assumptions are necessary to model transformers in this regime: 
\begin{itemize}
    \item the interleaved element-wise multi-layer perceptron (MLP) units compute arbitrary functions;
    \item the embedding dimension $m$ and number of heads per layer $H$ are much smaller than the sequence length $N$; and
    \item all model parameters, inputs, and intermediate products can be represented with $O(\log N)$-bit numbers.
\end{itemize}
These assumptions are justified in greater detail by \cite{sht24}.

Furthermore, we also model transformers as having $N'$ additional ``pause token'' inputs, which are embeddings that contain no information relevant information about the input instance, but which provide a longer ``work tape'' that can be utilized for computational purposes \cite[see e.g.][]{goyal2023think, pfau2024let}.
While structurally similar to chain-of-thought reasoning \cite{wei2022chain}, these pause tokens are not determined by multiple auto-regressive passes over the transformer and do not encode useful information as input.
Such pause tokens exist both in the theoretical results of \Cref{sec:hierarchy} (as blank tokens at the end of the sequence) and in the empirical results of \Cref{sec:transformers-gnns} (as placeholders between the graph representation and the task query for graphs that can be represented in fewer tokens than the maximum sequence length).

\begin{definition}
Let $\trans{m, H, L}{N, N'}$ denote the family of all bidirectional \textit{transformers} with embedding dimension $m$, number of heads $H$, and depth $L$, which operate on input sequences on length $N$ with $N'$ blank pause tokens appended to the end\footnote{If $N'=0$, we denote the family as $\trans{m, H, L}{N}$.}.
Any transformer $f \in \trans{m, H, L}{N, N'}$ is a function of the form $f: \R^{N \times d} \to \R^N$ for some input dimension $d$, which parameterized by query, key, and value matrices \[Q_{\ell, h}, K_{\ell, h}, V_{\ell, h} \in \R^{m \times m}, \text{ for $\ell \in [L], \ h \in [H]$},\]
and arbitrary element-wise MLPs with positional embeddings $\phi^0, \dots, \phi^L$ and $\psi$ with 
\begin{align*}
    \phi^0: &\R^d \times \N \to \R^m, \\
    \phi^\ell: &\R^m \times \N \to \R^m, \ \text{for $\ell\in\set{1, \dots, L}$}, \\
    \psi: &\R^m \times \N \to \R,
\end{align*}
where for sequence $Y$ of length $N + N'$, \[\phi^\ell(Y) = (\phi^\ell(y_1, 1), \dots, \phi^\ell(y_{N+N'}, N+ N')).\]
To evaluate $f(X)$ on some input sequence $X = (x_1, \dots, x_N) \in \R^{N \times d}$, we define intermediate embeddings $Y^0, \dots, Y^L \in \R^{(N+N') \times m}$ as follows:
\begin{itemize}
    \item The initial embedding is computed by applying the first MLP $\phi^0$ to all elements of the input $X$, along with $N'$ ``blank'' inputs: \[Y^0 = \phi^0(X) = (\phi^0(x_1, 1), \dots, \phi^0(x_N, N), \phi^0(0, N+1), \dots, \phi^0(0, N + N')).\]% \in \R^{(N+N') \times m}.\]
    \item The intermediate embeddings $Y^{\ell}$ for $\ell \in \set{1, \dots, L}$ are computed by applying a unit of \textit{multi-headed attention} (parameterized by $Q_{\ell, h}, K_{\ell, h}, V_{\ell, h}$ for all $h \in [H]$) to the previous embedding $Y^{\ell-1}$, followed by an element-wise application of the MLP $\phi^\ell$.
    That is,
    \begin{align*}
        Y^\ell 
        &= \phi^\ell\paren{\sum_{h=1}^H \sm\paren{\phi(Y^{\ell-1}) Q_h K_h^\T \phi(Y^{\ell-1})^\T} \phi(Y^{\ell-1}) V_h}.
    \end{align*}
    \item The output $f(X) \in \R^N$ is computed to be the first $N$ outputs of $\psi(Y^L)$. That is,
    \[f(X) = (\psi(Y^L_1, 1), \dots, \psi(Y^L_N, N)).\]
\end{itemize}
\end{definition}

While our theoretical results pertain to the bidirectional regime (where no masking is applied to the self-attention layer), all results in \Cref{asec:1L} and all lower bounds in \Cref{asec:transformer-mpc} apply to autoregressive transformers as well.
Our empirical results utilize causal masking.

Furthermore, an $L$-layer bidirectional transformer that operates on sequences of length $N$ with $N'$ pause tokens can be simulated by and $L$-layer autoregressive transformer with $O(L \cdot (N + N'))$ pause tokens.
In the regime where $N' = \poly(N)$ and $L = O(\log N)$, any positive results for bidirectional models likewise apply to autoregressive models.

% We model all MLPs as arbitrary functions from $\R^m \to \R^m$, and assume throughout that all transformers have logarithmic bit-precision, that is that all real numbers are encoded using $p = O(\log N)$ bits.
% If $N' = 0$, we abbreviate the class as $\trans{m, H, L}{N}$.

% \todo{Provide full transformer setting}

\subsection{Graph reasoning tasks}\label{ssec:appendix-graphreasoning}

We exclusively consider graph reasoning tasks that apply to undirected and unweighted graphs $G = (V, E)$ of bounded size, i.e. $|V| + |E| = O(N)$ for some size parameter $N$.
We define the tasks used for experimentation as follows:
\begin{itemize}
    \item \textbf{Node count}: Given graph $G$, compute $|V|$.
    \item \textbf{Edge count}: Given graph $G$, compute $|E|$.
    \item \textbf{Edge existence}: Given graph $G$ and vertices $u, v \in V$, return $\indicator{(u, v) \in E}$.
    \item \textbf{Node degree}: Given graph $G$ and vertex $u \in V$, return $\deg(u)= \abs{(u, v) \in E}$.
    \item \textbf{Connectivity}: Given graph $G$ and vertices $u, v \in V$, return 1 if there exists a path of edges between $u$ and $v$ and 0 otherwise.
    \item \textbf{Cycle check}: Given graph $G$, return 1 if there exists a cycle of any length $\geq 3$ and 0 otherwise.
    \item \textbf{Shortest path}: Given graph $G$ and vertices $u, v \in V$, return the smallest number of edges that forms a path between $u$ and $v$ if one exists and $-1$ otherwise.
    \item \textbf{Triangle counting}: Given graph $G$, return the number of triangles in the graph: $\abs{\set{(u, v, w): (u,v), (v, w), (u, w) \in E}}$.
\end{itemize}

These tasks are foundational to computer science and are solved by famous polynomial-time algorithms, such as Dijkstra's algorithm for shortest path and Kruskal's for spanning trees.
Several of these tasks reap significant algorithmic benefits from parallel computing models, including graph connectivity, which can be evaluated in a logarithmic number of rounds by employing an ``iterative merging'' strategy~\cite{asswz18}.  

The amenability of some of these tasks to parallel computing is the core principle underlying our theoretical results on transformer model capacity and the resulting contrasts with GNNs.
For instance, the iterative merging algorithm for connectivity can be simulated by a logarithmic-depth transformer, but message-passing GNNs cannot do without a substantial width scaling (see \Cref{asec:gnn_limitations}).
Furthermore, this parallel computing algorithm is widely believed to be optimal, and a crystallization of this as \Cref{conj:1vs2} is the bedrock of our transformer optimality conjectures.

% \paragraph*{Algorithms for graph reasoning tasks.}
% \todo{CLAYTON: this was moved from the related work Beyond introductory algorithms (APSP dynamic program, Dijkstra's shortest path algorithm, Kruskal's spanning tree, etc), graph reasoning tasks are well-studied in parallel computing settings:}

% \begin{itemize}
%     \item Ability to parallelize graph reasoning tasks well-studied \cite{ns22, asswz18}
%     \item Well-known conjecture on hardness of parallel computation related to graph reasoning/connectivity: one-cycle-vs-two-cycle conjecture \cite[e.g.,][]{gku19}
% \end{itemize}

The multi-layer results of \Cref{asec:transformer-mpc} are representation-agnostic; the bounds apply to any fixed encoding of vertices and edges as a transformer input sequence $X \in \R^{N \times d}$.
The input representation for \Cref{asec:1L} must satisfy the more specified \textit{node/edge encoding scheme}, which represents each vertex and edge as a single embedding, followed by a query embedding (with optional blank embeddings as needed).
This encoding scheme is reflected by \Cref{fig:transformer_architecture} and closely resembles the graph instance encoding used for the experimental results\footnote{The empirical encoding scheme uses two tokens---not one---to encode each edge and has fixed size $N$. For all graphs whose encodings do not require all $N$ tokens, the remaining tokens are blank placeholders that precede the query tokens.}.

Throughout the paper, we partition these and other graph reasoning tasks into three representational categories: retrieval, parallelizable, and search tasks.
The retrieval category contains \textit{only} four of the tasks used for experiments:
\begin{itemize}
    \item \textbf{Retrieval tasks} include node count, edge count, edge existence, and node degree.
\end{itemize}
On the other hand, the other categories reflect two equivalence classes introduce by \cite{ns22}:
\begin{itemize}
    \item \textbf{Parallelizable tasks} are defined to be $\L$-complete and equivalent to connectivity in $O(1)$ rounds of MPC. 
    These tasks include connectivity, cycle check, planarity testing, minimum cut, bipartiteness testing, minimum spanning forest, connected components, one-cycle versus two-cycles (see \Cref{conj:1vs2}), and \# connected components.
    \item \textbf{Search tasks} are defined to be $\NL$-complete and equivalent to shortest path in $O(1)$ rounds of MPC.
    These tasks include shortest path, strong connectivity (for directed graphs), all-pairs shortest path, median, diameter, radius, directed cycle detection, and $st$-reachability.
\end{itemize}
Triangle counting does not appear in any of these classes and is separately analyzed in \Cref{ssec:triangle}.

% \todo{Introduce graph problems (graph connectivity, shortest path, node count, edge count, node degree, edge existence, cycle check, triangle count) and edge list representation}

% \todo{Note that graph inputs to transformers are representation agnostic for MPC results, need node/edge encoding for single-layer}

%% file: transformer_mpc.tex
\section{Multi-layer transformers and parallelizable/search graph reasoning tasks}\label{asec:transformer-mpc}

In this appendix, we formalize and prove all results in \Cref{ssec:med-tasks,ssec:hard-tasks} that pertain to the parallelizable and search categories of graph reasoning tasks.

The primary technical tool used to establish these results is an improved relationship between the MPC model of distributed communication of~\cite{ksv10} and our theoretical model of a transformers (\Cref{assec:prelim-trans}). 
This result is presented informally as \Cref{thm:tight-mpc-informal}, and formally as \Cref{thm:tight-mpc}.
Because graph reasoning tasks are well studied in the MPC model of computation \citep[e.g.][]{asswz18,ns22}, this theorem enables the transfer of those positive results to transformer architectures.
Similarly, negative results that pertain to the MPC model, including those conditioned on the well-known \textit{one-cycle versus two-cycle conjecture} (\Cref{conj:1vs2}), imply negative results for transformers.

\begin{restatable}
[Formal version of \Cref{thm:tight-mpc-informal}; transformers simulate MPC]
{theorem}{thmtightmpc}\label{thm:tight-mpc}
For constants $0 < \delta < \delta' < 1$ and $\gamma > 0$, an $R$-round deterministic $(\gamma, \delta)$-MPC protocol with $n$ inputs can be simulated by a transformer $f \in \trans{m, H, L}{N, N'}$ with depth $L = O(R)$, single heads $H = 1$, embedding dimension $m = O(n^{\delta'})$, context length $N = n$, and blank chain-of-thought tokens $N' = \max(0, O(n^{1 + \gamma - \delta}) - n)$.
\end{restatable}

We formally introduce the MPC distributed computing model in \Cref{assec:mpc-prelim}, along with the one-cycle versus two-cycle conjecture and a collection of positive results from \cite{ns22}.
In \Cref{ssec:log-depth}, we formally present the task-specific graph reasoning results introduced in \Cref{ssec:med-tasks,ssec:hard-tasks} and prove them as implications of \Cref{thm:tight-mpc}. 
We finally contextualize and prove \Cref{thm:tight-mpc} in \Cref{assec:tight-mpc} by modifying the proof strategy of a similar result \cite{sht24} and by showing that MPC protocols can be simulated by a weaker model of distributed computation.

% The primary technical tool used to establish these results is the following theorem, which proves that any MPC protocol can be simulated by a transformer.
% Since these graph problems are well studied in the MPC model of computation by \cite{ns22}, this theorem enables those results to be imported to transformer architectures.

\subsection{MPC preliminaries}\label{assec:mpc-prelim}

The massively parallel computation (MPC) model of~\cite{ksv10} formalizes distributed computing frameworks such as MapReduce~\cite{dg04} as theoretical models that are amenable to rigorous analysis.
MPC pertains to a regime where an input that consists of $n$ ``words'' is distributed across a very large number of machines $q$ (e.g. $q \approx n^{0.95}$), each of which contains a bounded local memory $s$ (e.g. $s \approx n^{0.1}$) where computation on individual machines is inexpensive but communication between machines is costly.
We use the definition of MPC of \cite{asswz18}, which quantifies the complexity of a protocol by the local memory $s = O(n^\delta)$, the global memory $sq = O(n^{1 + \gamma})$, and the number of communication rounds $R$.
% \todo{more discussion of MPC: theoretical model of mapreduce, pertains to regime with very large number of machines with limited local memory on each, limitations on message sizes rather than total computation, difference between MPC/PRAM/CONGEST/etc}

\begin{definition}\label{def:mpc}
For global and local memory constants $\gamma, \delta > 0$ and input size $n$, an \emph{$R$-round $(\gamma, \delta)$-MPC protocol} specifies a distributed computing protocol over $q = \Theta(n^{1 + \gamma -\delta})$ machines, each of which contains a local memory of size $s = O(n^\delta)$. 
\begin{itemize}[nosep]
    \item An input to the protocol, which is encoded as a length-$n$ sequence of $p = \Theta(\log n)$-bit words, is distributed across the local memories of the first $\ceil{n / s}$ machines.
    \item In each of the $R$ rounds, each machine computes an arbitrary function its local memory at the time, which specifies at most $s$ words to be transmitted to other machines.
    \item Messages are simultaneously transmitted (subject to the constraint that each machine sends and receives at most $s$ words of information), and each machine's new local memory is set equal to the messages received.
    \item After the final round, the output of the protocol is a concatenation of the local memories of the first $\ceil{n / s}$ machines.
\end{itemize}
We say that an MPC protocol computes some function $f: \Z_{2^p}^n \to \Z_{2^p}^n$ if for any $X \in \Z_{2^p}^n$ encoded as input to the protocol, the output of the protocol is $f(X)$.
\end{definition}

% \todo{provide positive results on MPC from \cite{ns22}}

% \todo{want tight LBs on things; note that one-cycle vs two-cycle belongs to parallizability family of tasks}

% \todo{Introduce alternate MPC model with bounded number of senders and recipients}

% \todo{Introduce theorem (to be proved) relating the two models}

\subsubsection{MPC and graph reasoning tasks}\label{asssec:mpc-graph}

In this section, we introduce previously known results about the abilities of Massively Parallel Computation protocols to solve graph reasoning tasks.
The novel results presented in the subsequent section are the primarily the immediate implications of these results and \Cref{thm:tight-mpc}.

\paragraph*{MPC round equivalence for parallelizable and search tasks}

These results largely reflect the contributions of \cite{ns22}, which establish a hierarchy of graph reasoning tasks that is reflected in our distinction between parallelizable and search tasks.
They show that all parallelizable tasks are equivalent to connectivity in $O(1)$ rounds of MPC and that all search tasks are likewise equivalent to shortest path. (See their Figure~1.)
Concretely, they prove the following for all graphs $G = (V, E)$ with $|V| + |E| = O(n)$.

\begin{theorem}[Theorem 4 of \cite{ns22}; round-equivalence of parallelizable tasks]\label{thm:ns22-med}
If there exists an $R$-round $(\gamma, \delta)$-MPC protocol for any parallelizable task---including graph connectivity, $st$-connectivity, cycle detection, formula evaluation, planarity testing, graph bipartiteness, connected components, minimum spanning forest, and minimum cut, one-cycle versus two-cycle testing---for some $\gamma > 0$ and $\delta \in (0, 1)$, then there exists an $R'$-round $(\gamma', \delta)$-MPC protocol for any other other parallelizable task where $R' = R + O(1)$ and $\gamma' = \max(\gamma, 3)$.
\end{theorem}

\begin{theorem}[Theorem 5 of \cite{ns22}; round-equivalence of search tasks]\label{thm:ns22-hard}
If there exists an $R$-round $(\gamma, \delta)$-MPC protocol for any search task---including $st$-reachability (for directed graphs), strong connectivitity, directed cycle detection, shortest path, all-pairs shortest path, diameter, radius, and median---for some $\gamma > 0$ and $\delta \in (0, 1)$, then there exists an $R'$-round $(\gamma', \delta)$-MPC protocol for any other other search task where $R' = R + O(1)$ and $\gamma' = \max(\gamma, 2)$.
\end{theorem}

The equivalance of \Cref{thm:ns22-med} has more immediate implications for the round complexity of parallelizable tasks than \Cref{thm:ns22-hard} does for search tasks because the round complexity of graph connectivity is well-understood to be $O(\log n)$ and thought to be optimal.
We first present a deterministic MPC positive result for graph connectivity.

\begin{theorem}[Theorem 6.2 of \cite{cc22}; log-round connectivity MPC algorithm]\label{thm:cc22-connectivity}
For any $\gamma > 0$ and $\delta \in (0, 1)$, there exists a deterministic $O(\log n)$-round $(\gamma, \delta)$-MPC protocol that solves graph connectivity on any graph $G = (V, E)$ of size $|V| + |E| = O(n)$.
\end{theorem}

Hence, all parallelizable tasks can be solved with a logarithmic number of MPC rounds with arbitrarily small polynomial local memory.

\begin{corollary}[Log-round parallelizable MPC algorithms]\label{cor:log-med}
For any $\gamma > 0$ and $\delta \in (0, 1)$ and any parallelizable task, there exists a deterministic $O(\log n)$-round $(\gamma', \delta)$-MPC protocol that solves the task on any graph $G = (V, E)$ of size $|V| + |E| = O(n)$ for $\gamma' = \max(\gamma, 3)$.
\end{corollary}

The optimality of these logarithmic-round protocols is suggested by a conjecture about the hardness of distinguishing between two graphs with $n$ vertices and $n$ edges, one of whose edges are arranged in a single cycle of length $n$ and the other in two disjoint cycles of length $\frac{n}2$.
This is the well-known ``one-cycle versus two-cycle conjecture,'' which is widely employed as a condition for distributed computing hardness \cite[see e.g.][]{bks17,rvw18,gku19}.

\begin{conjecture}[One-cycle versus two-cycle conjecture, see e.g. \cite{gku19}]\label{conj:1vs2}
For any $\gamma > 0$ and $\delta \in (0, 1)$, any $(\gamma, \delta)$-MPC protocol that distinguishes a single length-$n$ cycle from two disjoint length-$\frac{n}2$ cycles uses $R = \Omega(\log n)$ rounds.
\end{conjecture}

Under this conjecture, \Cref{thm:ns22-med} immediately implies the optimality of \Cref{cor:log-med}.

\begin{corollary}[Optimality of log-round parallelizable MPC algorithms]\label{cor:log-med-opt}
Conditional on \Cref{conj:1vs2}, for any $\gamma > 0$ and $\delta \in (0, 1)$, any $(\gamma, \delta)$-MPC protocol that solves a parallelizable graph task on all graphs $G = (V, E)$ of size $|V| + |E| = O(n)$ uses $R = \Omega(\log n)$ rounds.
\end{corollary}

The round complexity of search tasks is less understood, and it is unknown if search tasks can be solved by $O(\log n)$-round $(\gamma, \delta)$-MPC protocols if $\delta \in (0, \frac12]$.

% However, we can import a proof of sub-linear round complexity for the shortest path task.

% \begin{theorem}[]

% \end{theorem}

\paragraph*{MPC protocols for problems with bounded circuit size}

More general MPC constructions are known for problems that solved by bounded-size boolean circuits, which include both parallelizable and search tasks. 
The well-known $\NC^i$ classes of Boolean circuits that take $N$ inputs and have $\poly(n)$ gates and depth $O(\log^i n)$ have been shown to be computable by bounded-round MapReduce-like computational protocols \cite{fw19} and by MPC protocols in particular \cite{ns22}.
% \cite{fw19} proves that the classes of bounded-depth Boolean circuits can be computed with bounded-round MapReduce computational protocols, and the result is adapted to the MPC setting by \cite{ns22}.

\begin{theorem}[Theorem 1 of \cite{ns22}; log-round circuit MPC algorithms]\label{thm:nc-mpc-ns22}
For any problem in $\NC^{i+1}$ and any $\gamma > 0$ and $\delta \in (\frac12, 1)$, there exists a deterministic $O(\log^i n)$-round $(\gamma, \delta)$-MPC protocol that solves the problem.
\end{theorem}

Since $\L$ and $\NL$ are known to belong to $\NC^2$, the following corollary is an immediate consequence.

\begin{corollary}[Log-round parallelizable and search MPC algorithms with high local memory]\label{cor:log-wide}
For any parallelizable or search graph reasoning task and any $\gamma > 0$ and $\delta \in (\frac12, 1)$, there exists a deterministic $O(\log n)$-round $(\gamma, \delta)$-MPC protocol that solves the task on all graphs $G = (V, E)$ of size $|V| + |E| = O(N)$.
\end{corollary}

Note that these results pertain exclusively to the ``large local memory regime,'' where each machine has memory at $s = \omega(\sqrt{n})$.
Therefore, this does not guarantee the existence of a $O(\log n)$-round MPC solution for any search task or for any parallelizable task with $\gamma < 1$.

\paragraph*{MPC protocol for triangle counting}

Finally, the triangle counting task can be solved in the MPC framework by utilizing a special case of a parallel algorithm for clique counting.
These pertain to graphs with bounded \textit{arboricity} $\alpha$, a quantity that corresponds to the branching factor of a node that is bounded by the degree of the graph; these apply to arbitrary graphs by noting that $\alpha \leq |V|$.

\begin{theorem}[Theorem 1.5 of \cite{belmr21}; loglog-round triangle counting MPC algorithm]
For any $k \geq 3$, $\delta \in (0, 1)$, and $\gamma > 0$, there exists an $O(\log\log n)$-round $(\gamma, \delta)$-MPC protocol that computes the number of $k$-cliques in any graph $G = (V,E)$ with $|V| + |E| = O(n)$ and arboricity $\alpha = O(n^{\gamma/(k-2)})$.
\end{theorem}

% As an immediate consequence of Theorem~1 of \cite{ns22}, which bounds the number of rounds of MPC sufficient to solve problems in the $\NC$ hierarchy, and \Cref{thm:tight-mpc}, we have the following immediate result for bounded-depth transformers.

% \subsection{Parallelizability of graph reasoning tasks and logarithmic-depth transformers}\label{ssec:log-depth}

\subsection{Positive and negative graph reasoning results for multi-layer transformers}\label{ssec:log-depth}

We state formalized versions of the statements of \Cref{ssec:med-tasks,ssec:hard-tasks,ssec:triangle} and prove them by invoking \Cref{thm:tight-mpc} jointly with the relationships between MPC and graph reasoning of \Cref{asssec:mpc-graph}.

\subsubsection{Parallelizable task results (\Cref{ssec:med-tasks})}\label{assec:med-tasks}

For the duration of this section, the class of parallelizable tasks includes all of those that are deemed equivalent to graph connectivity in $O(1)$ rounds of MPC by \cite{ns22}, as stated in \Cref{thm:ns22-med}.

% \paragraph*{Positive results (\Cref{thm:parallelizable-pos})}

We first present a statement that formalizes the existence of logarithmic-depth transformer constructions for solving parallelizable tasks.

\begin{theorem}
[Formal version of \Cref{thm:parallelizable-pos}; log-depth transformers compute parallelizable tasks]
\label{thm:parallelizable-pos-formal}
For any $\epsilon \in (0, 1)$ and any parallelizable task, there exists a transformer $f \in \trans{m, H, L}{N, N'}$ such that $f(X)$ computes the task where $X \in \R^N$ is some encoding of any graph $G = (V, E)$ with $|V| + |E| = O(N)$ and $f$ has depth $L = O(\log N)$ and heads $H = O(1)$ and embedding dimension $m$ and pause tokens $N'$ satisfying either
\begin{itemize}[nosep]
    \item \logdepthpause{}: $m = O(N^\epsilon)$ and $N' = O(N^{4 - \epsilon'})$, where $\epsilon' \in (0, \epsilon)$; or
    \item \logdepthwide{}: $m = O(N^{\epsilon})$ and $N' = 0$, if $\epsilon > \frac12$.
\end{itemize}
\end{theorem}
\begin{proof}
The first bullet is an immediate implication of \Cref{cor:log-med} and \Cref{thm:tight-mpc} with $\gamma = 3$, $\delta' = \epsilon$, and $\delta = \epsilon'$.
The second bullet follows from \Cref{cor:log-wide} and \Cref{thm:tight-mpc} with $\gamma = \delta = \frac{\epsilon}2$ and $\delta' = \epsilon$.
\end{proof}

We then establish that sub-logarithmic-depth solutions to any parallelizable task are impossible without having linear embedding dimension $m$ or super-polynomial number of pause tokens $N'$, under the assumption that the one-cycle versus two-cycle conjecture holds.

\begin{theorem}
[Formal version of \Cref{thm:graph-log-lb-informal}; log-depth optimality for parallelizable tasks]
\label{thm:graph-log-lb}
Conditional on \Cref{conj:1vs2}, for any $\epsilon \in (0, 1)$ and $\gamma > 0$ and any parallelizable task, if there exists a transformer $f \in \trans{m, H, L}{N, N'}$ that solves the task and has width $mH = O(N^\epsilon)$ and pause tokens $N + N' = O(N^{1 + \gamma})$, then its depth satisfies $L = \Omega(\log N)$.
\end{theorem}

\begin{proof}
The proof is a consequence of \Cref{cor:log-med-opt} and a result of \cite{sht24} that proves the to simulate transformers with MPC protocols, an inversion of \Cref{thm:tight-mpc}.
We restate that result as follows.

\begin{theorem}[Theorem~3.4 of \cite{sht24}; MPC simulates transformers]
Fix any transformer $f \in \trans{m, H, L}{N, N'}$ with width $mH = O(N^{\delta})$ for some $\delta \in (0, 1)$ and total sequence length $N + N' = O(N^{1 + \gamma})$ for some $\gamma \geq 0$.
Then, for any $\delta' \in (\delta, 1)$ and $\gamma' = 1 + 2\gamma + \delta'$, there exists an $O(\frac{L(1 + \gamma)}{\delta' - \delta})$-round $(\gamma', \delta')$-MPC protocol that simulates $f$.
\end{theorem}

Consider a transformer $f \in \trans{m, L, H}{N, N'}$ that solves the task with width $mH = O(N^\epsilon)$ and total sequence length $N + N' = O(N^{1 + \gamma}$ for some constant $\epsilon \in (0, 1)$ and $\gamma > 1$.
Then, there exists an $O_{\epsilon, \gamma}(L)$-round $(1 + 2\gamma + \sqrt{\epsilon}, \sqrt{\epsilon})$-MPC protocol that solves the parallelizable task.
If \Cref{conj:1vs2} is true, then $L = \Omega(\log N)$.
% This contradicts \Cref{cor:log-med-opt}, unless $L = \Omega(\log N)$ or \Cref{conj:1vs2} does not hold.
\end{proof}

% \begin{theorem}\label{thm:tran-nc}
% For any $\epsilon \in (0, \frac12)$, $i \geq 0$, and any problem $P \in  \NC^{i+1}$, there exists a transformer $f \in \trans{m, H, L}{n}$ with embedding dimension $m = O(n^{1/2 + \epsilon})$, number of heads $H = O(1)$, and depth $L = O(\log^i n)$ that solves $P$.
% \end{theorem}

% Echoing Corollary 3 of \cite{ns22} and applying the fact that $\NC^1 \subseteq \L \subseteq \NL \subseteq \NC^2$, we have the following implication.

% \begin{corollary}\label{cor:tran-l}
% For any $\epsilon \in (0, \frac12)$ and any problem $P \in \L$ (or $\NL$), there exists a transformer $f \in \trans{m, H, L}{n}$ with embedding dimension $m = O(n^{1/2 + \epsilon})$, number of heads $H = O(1)$, and depth $L = O(\log n)$ that solves $P$.
% \end{corollary}

% Results of this generality were not attainable by \cite{sht24}, due to the fact that their Theorem~3.1 is unable to simulate MPC protocols with local memory $s = \omega(n^{1/2})$ with transformers of sublinear embedding dimension.

\subsubsection{Search task results (\Cref{ssec:hard-tasks})}\label{assec:hard-tasks}

We present the main result of \Cref{ssec:hard-tasks} and show an equivalence between the minimum depth transformer needed to solve search tasks in a width-efficient manner.
As before, the class of searchable tasks includes tasks that are equivalent to shortest path in $O(1)$ MPC rounds (\Cref{thm:ns22-hard}).

\begin{theorem}
[Formal version of \Cref{thm:search-pos}; \logdepthwide{} computes search tasks]
\label{thm:search-pos-formal}
For any $\epsilon \in (\frac12, 1)$ and any search task, there exists a transformer $f \in \trans{m, H, L}{N, N'}$ such that $f(X)$ computes the task where $X \in \R^N$ is some encoding of any graph $G = (V, E)$ with $|V| + |E| = O(N)$, and $f$ has depth $L= O(\log N)$, heads $H = O(1)$, embedding dimension $m = O(N^{\epsilon}$, and no pause tokens ($N' = 0$).
\end{theorem}
\begin{proof}
As before, the proof is an immediate consequence of \Cref{cor:log-wide} and \Cref{thm:tight-mpc}.
\end{proof}

While both search and parallelizable tasks are computable logarithmic depth of considerable width, the minimum depth transformer of width $O(N^\epsilon)$ for small $\epsilon$ that computes search tasks is unknown.
Despite this deficit, we can still establish the representation similarity of all search tasks by showing that their minimum depths vary by at most a constant additive factor.

\begin{theorem}
[Depth-equivalence of search tasks]
\label{thm:search-equiv}
Suppose for some $\epsilon \in (0, 1)$ and $\gamma > 0$ there exists a transformer $f \in \trans{m, H, L}{N, N'}$ with embedding dimension $m = O(N^\epsilon)$ and total sequence length $N + N' = O(N^{1+\gamma})$ that computes the some search task on all graphs $G = (V, E)$ of size $|V| + |E| = O(N)$.
Then, for any other search task, there exists some transformer $f' \in \trans{\bar{m}, H, \bar{L}}{N, \bar{N}'}$ with embedding dimension $\bar{m} = O(m)$, depth $\bar{L}= L + O(1)$, and pause tokens $\bar{N}' = N' + O(N^3)$ that computes that search task.
% , then any other search task can be solved by some transformer of depth $L + O(1)$, embedding dimension $O(m)$, and pause tokens $N' + O(N^3)$.
\end{theorem}

\begin{proof}
This statement is an immediate consequence of \Cref{thm:ns22-hard} and \Cref{thm:tight-mpc}.
\end{proof}

\subsubsection{Clique counting task results (\Cref{ssec:triangle})}\label{assec:triangle}

We prove the existing of doubly-logarithmic-depth transformer that solves the triangle counting task giving a more general result that counts the number of $k$-cliques in any graph of bounded arboricity.

\begin{theorem}[Generalization of \Cref{cor:triangle-informal}; loglog-depth computes clique counting]\label{cor:clique}
For any fixed $\epsilon \in (0, 1)$ and $k \geq 3$, there exists a transformer $f \in \trans{m, H, L}{N, N'}$ with embedding dimension $m = O(n^\epsilon)$, heads $H = O(1)$, depth $L = O(\log\log n)$, and chain-of-thought tokens \[N' = \begin{cases} O(\alpha^{k-2} N^{1-\epsilon}) & \text{if $\alpha = \Omega(N^{\epsilon / (k-2)})$,} \\ 0 & \text{otherwise.} \end{cases}\] %\max(0, O(N^{1-\epsilon}\alpha^{k-2}) - N)\]
that counts the number of $k$-cliques in all graphs $G = (V, E)$ of arboricity $\alpha$ and size $|V| + |E| = O(N)$.
\end{theorem}

\Cref{cor:triangle-informal} follows immediately from taking $k = 3$.

\begin{proof}
This is the immediate implication of \Cref{thm:tight-mpc} and \Cref{cor:clique}.
\end{proof}

\subsection{Proof of \Cref{thm:tight-mpc}}\label{assec:tight-mpc}

% \thmtightmpc*

% This theorem offers an improvement over Theorem~3.1 of \cite{sht24}, which only applies to MPC protocols with local memory $s = O(n^{1/4 - \epsilon})$.
% This improved result makes the results of \Cref{sssec:large-width} possible, since the corresponding MPC protocols provided by \cite{ns22} rely on each machine having a large local memory.
% We prove \Cref{thm:tight-mpc} in two stages in \Cref{assec:mpc-sim,assec:mpc-sht24}.
% % We restate and prove \Cref{thm:tight-mpc} in \Cref{asec:transformer-mpc}.

% In \Cref{assec:mpc-prelim}, we formally introduce the MPC model and share relevant results in that setting.

Our principal theoretical result establishes that any MPC protocol with sublinear local memory can be simulated by a transformer with sublinear embedding dimension.

\thmtightmpc*

This offers an improvement on Theorem~3.1 of \cite{sht24}, which only guarantees that MPC protocols with local memory $s=O(n^{1/4 - \epsilon})$ (or with $\delta < \frac14$) can be simulated by transformers with sublinear embedding dimension.
This distinction is significant because several positive results for the MPC protocol (e.g. the ability to solve all problems in $\mathsf{L}$ and $\mathsf{NL}$ in a logarithmic number of MPC rounds) require $s = \Omega(N^{1/2})$ local memory, and hence could not be shown to be simulated by transformers of sublinear embedding dimension previously\footnote{
For this theoretical model of the transformer to be non-vacuous, the transformer must have embedding dimension $m = o(Nd)$.
Without this, any function over $X \in \R^{N \times d}$ could be solved by a single-layer transformer that concatenates all inputs into a single vector in an attention unit and passes that as input to an MLP that solves the problem.
}.

Including an allowance of $N'$ blank pause tokens permits the simulation of MPC protocol with $\gamma \geq \delta$, i.e. where number of machines $q$ grows super-linearly with $n$. 
If $\gamma < \delta$, then the protocol can be simulated without pause tokens (i.e. $N' = 0$) for sufficiently large input sizes $n$.

To prove \Cref{thm:tight-mpc}, we define a restriction of the MPC computational model that disallows each machine from communicating with more than $k$ other machines in each round of the protocol.

\begin{definition}
For constants $\gamma, \delta, \rho > 0$, an \emph{$R$-round $(\gamma, \delta, \rho)$-MPC protocol} on input sequences of length $n$ is a $(\gamma, \delta)$-MPC protocol (\Cref{def:mpc}) that obeys an additional constraint on the number of outgoing and incoming messages. 
Namely, for some capacity $k = O(n^\rho)$, in each round, every machine can send its local memory to and receive information from at most distinct $k$ machines.
% For global and local memory constants $\gamma, \delta, \rho > 0$ and input size $n$, an \emph{$R$-round $(\gamma, \delta, \rho)$-MPC protocol} specifies a distributed computing protocol over $q = \Theta(n^{1 + \gamma -\delta})$ machines, each of which contains a local memory of size $s = O(n^\delta)$ and a degree of $b = O(n^\rho)$. 
% An input to the protocol, which is encoded as a length-$n$ sequence of $p = \Theta(\log n)$-bit words, is distributed across the local memories of the first $\ceil{n / s}$ machines.
% In each of the $R$ rounds, each machine computes an arbitrary function its local memory at the time, which specifies at most $s$ words to be transmitted \emph{to at most $b$ other machines}.
% Messages are simultaneously transmitted (subject to the constraint that each machine sends and receives at most $s$ words of information \emph{from at most $b$ other machines}), and each machine's new local memory is set equal to the messages received.
% After the final round, the output of the protocol is a concatenation of the local memories of the first $\ceil{n / s}$ machines.
% We say that an MPC protocol computes some function $f: \Z_{2^p}^n \to \Z_{2^p}^n$ if for any $X \in \Z_{2^p}^n$ encoded as input to the protocol, the output of the protocol is $f(X)$.
\end{definition}

\Cref{thm:tight-mpc} is an immediate consequence of \Cref{prop:mpc-sim}---which establishes that any $(\gamma, \delta)$-MPC protocol can be simulated by a $(\gamma, \delta, \rho)$-MPC protocol---and \Cref{cor:mpc-sht24}---which shows that any $(\gamma, \delta, \rho)$-MPC protocol can be simulated by a transformer.

\begin{restatable}
[$(\gamma, \delta, \rho)$-MPC simulates $(\gamma, \delta)$-MPC]{proposition}{propmpcsim}\label{prop:mpc-sim}
For constants $\gamma, \delta > 0$ and $\rho \in (0, \frac\delta2)$, if $f$ can be computed by an $R$-round $(\gamma, \delta)$-MPC protocol, then there exists a $O(\frac{R(1 + \gamma)^2}{\rho^2})$-round $(\gamma, \delta, \rho)$-MPC protocol that computes $f$ as well.
\end{restatable}

\begin{restatable}[Transformers simulate $(\gamma, \delta, \rho)$-MPC]{corollary}{cormpcsht}\label{cor:mpc-sht24}
For constants $\delta \in (0, 1)$ and $\gamma, \rho > 0$, an $R$-round deterministic $(\gamma, \delta, \rho)$-MPC protocol with $n$ inputs can be simulated by a transformer $f \in \trans{m, H, L}{N, N'}$ with depth $L = R + 1$, heads $H = 1$, embedding dimension $m = O(n^{\delta + 4\rho} \log n)$, context length $N = n$, and blank chain-of-thought tokens $N' = \max(0, O(n^{1 + \gamma - \delta}) - n)$. 
\end{restatable}

\begin{proof}[Proof of \Cref{thm:tight-mpc}]
Let $\rho := \min(\frac\delta2, \frac{\delta' - \delta}4) - \epsilon$ for some small constant $\epsilon$ (e.g. $\epsilon := \min(\frac\delta4, \frac{\delta' - \delta}8)$).
By \Cref{prop:mpc-sim}, we can simulate the target $R$-round $(\gamma, \delta)$-MPC protocol with an $R'$-round $(\gamma, \delta, \rho)$-MPC protocol for
\[R'=O\paren{\frac{R(1 + \gamma^2)}{\min((\delta' - \delta)^2, \delta^2)}}.\]
We apply \Cref{cor:mpc-sht24} to conclude that this latter protocol can be simulated by a transformer of depth $L = R' + 1$ and embedding dimension \[m = O(n^{\delta + 4 \rho} \log n) = O(n^{\delta'}).\qedhere\]
\end{proof}

We prove \Cref{prop:mpc-sim} in \Cref{assec:mpc-sim} by simulating a single round of a standard MPC with multiple rounds of restricted MPC, where messages are sent to intermediate ``neighborhoods''---which contain collections of machines with similar destinations---of increasing fine granularity.
We prove \Cref{cor:mpc-sht24} in \Cref{assec:mpc-sht24} by a minor adaptation of the proof of Theorem~3.1 of \cite{sht24}.

\subsubsection{Proof of \Cref{prop:mpc-sim}}\label{assec:mpc-sim}

% \propmpcsim*

Fix an $R$-round $(\gamma, \delta)$-MPC protocol $\pi$ with input size $n$, $q  = \Theta(n^{1 + \gamma - \delta})$ machines, and $s = O(n^\delta)$ local memory. 
To prove \Cref{prop:mpc-sim}, it suffices to show that a single round of MPC communication can be simulated an $O(\frac{(1 + \gamma)^2}{\rho^2})$-round $(\gamma, \delta, \rho)$-MPC protocol.

To formalize the communication procedure to simulate, we let $\outbox = (\outbox_1, \dots, \outbox_q)$ denote an encoding of the outgoing ``message packets'' from each source machine that obeys $\pi$'s local memory constraints. Concretely,
\[\outbox_i = \set{(\msg, \src, \dst): \src = i, \dst \in [q]} \ \st \ \sum_{\msg \in \outbox_i} |\msg| \leq s.\] 
Let $\inbox = (\inbox_1, \dots, \inbox_q)$ denote the same collection of message packets, organized by their destination machines.
\[\inbox_i = \set{(\msg, \src, \dst) \in \outbox_\src: \dst = i} \ \st \ \sum_{\msg \in \inbox_i} |\msg| \leq s.\] 
It suffices to prove \Cref{lemma:mpc-sim-round}.

\begin{lemma}
[$(\gamma, \delta, \rho)$-MPC simulates one communication round of $(\gamma, \delta)$-MPC]
\label{lemma:mpc-sim-round}
If $\rho < \frac\delta2$, there exists an $O(\frac{(1 + \gamma)^2}{\rho^2})$-round $(\gamma, \delta, \rho)$-MPC protocol that takes as input any $\outbox$ and returns the corresponding $\inbox$.
\end{lemma}

\begin{proof}

We define $r = O(\frac{1 + \gamma}\rho)$ intermediate steps and prove that each of those can be simulated.
The intuition is that each an intermediate step routes each packet $(\msg, \src, \dst)$ to a machine that belongs to the same ``neighborhood'' as $\dst$. 
Each step maps each packet to a neighborhood of smaller radius than the step before, until all packets have been transmitted to their proper destination location.

We define a tree of neighborhoods as follows.
Fix some branching factor $\ell = \Theta(n^{\rho/2})$ and number of intermediate steps $r = \ceil{\frac{\log q}{\log \ell}} = O(\frac{1 + \gamma}\rho)$.
For any step $t = 0, \dots, r$ and neighborhood index $j = 1, \dots, b_t := \ceil{\frac{q}{\ell^{r-t}}}$, we define
\[\nbhd^t_j = \set{(j-1) \ell^{r-t} + 1, (j-1) \ell^{r-t} + 2, \dots, j \cdot \ell^{r-t}} \cap [q]\]
and observe that it satisfies the size constraint $|\nbhd^t_j| \leq \ell^{r-t}$.
We let $\nbhd^t(\dst)$ denote the unique neighborhood $\nbhd^t_j$ satisfying $\dst \in \nbhd^t_j$.
We say that $\nbhd^t_j < \nbhd^t_{j'}$ if $j < j'$ (and equivalently, for all $i \in \nbhd^t_j$ and $i' \in \nbhd^t_{j'}$, $i < i'$).
Let 
\begin{align*}
    \children(\nbhd^t_j) &= \set{\nbhd^{t+1}_{(j-1)\ell + 1}, \dots, \nbhd^{t+1}_{j\ell}}, \\
    \desc^\tau(\nbhd^t_j) &= \set{\nbhd^{\tau}(i): i \in \nbhd^t_j},\\
    \parent(\nbhd^t_j) &= \nbhd^{t-1}_{\ceil{j / \ell}}, \\ 
    \ancestor^t(\nbhd^\tau_j) &= \nbhd^t(i) \ \text{for $i \in \nbhd^\tau_j$}.% \\
    % \firstchild(\nbhd^t_j) &= \nbhd^{t+1}_{(j-1)\ell + 1} = \min \children(\nbhd^t_j),
    % \firstchild(\nbhd^t_j) &= \min \nbhd^t_j = (j-1)\ell^{r-t},
\end{align*}
for some $\tau \geq t$.

Note that the following are true of neighborhoods.
\begin{itemize}
    \item The initial ``step 0'' neighborhood contains all machines (i.e. $\nbhd^0_1 = [q]$), and the final ``step $r$'' neighborhoods contain a single machine (i.e. $\nbhd^r_j = \{j\}$ for all $j \in [q]$).
    \item For any $t < r$ and $j \in [b_t]$, $\nbhd^t_j$ is the disjoint union of all sets in $\children(\nbhd^t_j)$ and $|\children(\nbhd^t_j)| \leq \ell$.
    \item $\set{\nbhd^t_1, \dots, \nbhd^t_{b_t}}$ comprise a disjoint union of $[q]$.
    \item For any $\dst \in [q]$, there exist unique sets $\nbhd^0(\dst), \dots, \nbhd^r(\dst)$ that satisfy \[\nbhd^r(\dst) \subset \dots \subset \nbhd^0(\dst)\]
    and $\nbhd^{t-1}(\dst) = \parent(\nbhd^{t}(\dst))$.
\end{itemize}

%---- Neighborhood figure
\begin{figure}[t]
    \centering
\resizebox{!}{0.8\textwidth}{
\begin{tikzpicture}[scale=0.4,->,>=stealth',shorten >=1pt,auto,node distance=3cm,
                    thick,main node/.style={circle,draw,font=\sffamily\small\bfseries}]
% Step 0
\node[align=center](0) at (3,30){Step $0$\\ \tiny 1 $\nbhd$\\ \tiny $q$ machines per $\nbhd$};
% Neigbhorhood 0-1
\draw[rounded corners=5pt, fill=blue!20] (0,0) rectangle (6,28); 
\node[align=center](0) at (3,27.25){$\nbhd^0_1$};
\node[draw, rounded corners, align=center, fill=pink] (1) at (3,25.5){Machine 1\\$\outbox^0_1$};
\node[draw, rounded corners, align=center, fill=pink] (2) at (3,22){Machine 2\\$\outbox^0_2$};
\node[draw, rounded corners, align=center, fill=pink] (3) at (3,1.5){Machine $q$\\$\outbox^0_q$};
  \foreach \y in {7,9,11,13,15,17,19}  % Adjust values for spacing
    \filldraw[black] (3,\y) circle (2pt); 
% Step 1
\node[align=center](0) at (13,30){Step $1$\\ \tiny $\ell$ $\nbhd$s\\ \tiny $q/\ell$ machines per $\nbhd$};
%Nbhd 1-1
\draw[rounded corners=5pt,fill=blue!20] (10,17) rectangle (16,28); 
\node[align=center](4) at (13,27.25){$\nbhd^1_1$};
\node[draw, rounded corners, align=center,fill=pink] (5) at (13,25.5){Machine 1\\$\outbox^1_1$};
  \foreach \y in {22,22.5,23}  % Adjust values for spacing
    \filldraw[black] (13,\y) circle (1pt); 
\node[draw, rounded corners, align=center,fill=pink] (6) at (13,19){Machine $q/\ell$\\$\outbox^1_{q/\ell}$};
  \foreach \y in {12,14,16}  % Adjust values for spacing
    \filldraw[black] (13,\y) circle (2pt); 
% Nbhd 1-\ell
\draw[rounded corners=5pt,fill=blue!20] (10,0) rectangle (16,11); 
\node[align=center](0) at (13,10){$\nbhd^1_\ell$};
  \foreach \y in {5,6,7,8,9}  % Adjust values for spacing
    \filldraw[black] (13,\y) circle (1pt); 
\node[draw, rounded corners, align=center,fill=pink] (2) at (13,1.5){Machine $q$\\$\outbox^1_{q}$};
% Arrows between Step 0 and Step 1

\draw[->] (5.25,25) -- (12,2.8); 
\draw[->] (5.25,25) -- (10.5,20);

\draw[->] (5.25,21) -- (12,2.8); 
\draw[->] (5.25,21) -- (10.5,20);
\draw[->] (5.25,21) -- (10.8,24);

\draw[->] (5.25,2) -- (10.8,2.4); 
\draw[->] (5.25,2) -- (10.2,20);
\draw[->] (5.25,2) -- (10.5 ,23);

% Step 2
%Nbhd 2-1
\node[align=center](0) at (23,30){Step $2$\\ \tiny $\ell^2$ $\nbhd$s\\ \tiny $q/\ell^2$ machines per $\nbhd$};
\draw[rounded corners=5pt,fill=blue!20] (20,24) rectangle (26,28); 
\node[align=center](0) at (23,27.25){$\nbhd^2_1$};
 \foreach \y in {25,25.6,26}  % Adjust values for spacing
    \filldraw[black] (23,\y) circle (1pt); 
\foreach \y in {22,22.5,23}  % Adjust values for spacing
    \filldraw[black] (23,\y) circle (1pt);
%Nbhd 2-\ell
\draw[rounded corners=5pt,fill=blue!20] (20,17) rectangle (26,21); 
\node[align=center](0) at (23,20){$\nbhd^2_\ell$};
 \foreach \y in {18,18.5,19}  % Adjust values for spacing
    \filldraw[black] (23,\y) circle (1pt);     
  \foreach \y in {13,14,15}  % Adjust values for spacing
    \filldraw[black] (23,\y) circle (2pt); 
   
%Nbhd 2-\ell^2-\ell+1
\draw[rounded corners=5pt,fill=blue!20] (20,7) rectangle (26,11); 
\node[align=center](0) at (23,10){$\nbhd^2_{\ell^2-\ell+1}$};
 \foreach \y in {8,8.5,9}  % Adjust values for spacing
    \filldraw[black] (23,\y) circle (1pt); 
\foreach \y in {5,5.5,6}  % Adjust values for spacing
    \filldraw[black] (23,\y) circle (2pt);
%Nbhd 2-\ell^2
\draw[rounded corners=5pt,fill=blue!20] (20,0) rectangle (26,4); 
\node[align=center](0) at (23,3){$\nbhd^2_{\ell^2}$};
     \foreach \y in {1,1.5,2}  % Adjust values for spacing
    \filldraw[black] (23,\y) circle (1pt); 

% Arrows between Step 1 and Step 2
\draw[->] (15.25,24) -- (22,25.4); 
\draw[->] (15.25,24) -- (20.5,20);
\draw[->] (15.7,19) -- (22,24.4); 
\draw[->] (15.7,19) -- (20.5,19);
\draw[->] (15.25,2) -- (22,2); 
\draw[->] (15.25,2) -- (22,8); 
\foreach \x in {27,28,29}  % Adjust values for spacing
  \filldraw[black] (\x,14) circle (2pt);

\foreach \x in {27,28,29}  % Adjust values for spacing
  \filldraw[black] (\x,30) circle (1 pt);

% Step r
\node[align=center](0) at (33,30){Step $r$\\ \tiny $q$ $\nbhd$s\\ \tiny $1$ machine per $\nbhd$};
%Nbhd r-1
\draw[rounded corners=5pt,fill=blue!20] (30,24) rectangle (36,28); 
\node[align=center](0) at (33,27.25){$\nbhd^r_1$};
\node[draw, rounded corners, align=center,fill=pink] (6) at (33,25.5){Machine $1$\\$\outbox^r_{1}$};
\foreach \y in {7,9,11,13,15,17,19,21}  % Adjust values for spacing
    \filldraw[black] (33,\y) circle (2pt);
%Nbhd r-q
\draw[rounded corners=5pt,fill=blue!20] (30,0) rectangle (36,4); 
\node[align=center](0) at (33,3.3){$\nbhd^r_q$};
\node[draw, rounded corners, align=center,fill=pink] (6) at (33,1.5){Machine $q$\\$\outbox^r_{q}$};

% Arrows into Step r
\draw[->] (28,26) -- (31,26);
\draw[->] (28,25) -- (31,25);

\draw[->] (28,2) -- (31,2);
\draw[->] (28,1) -- (31,1);
\end{tikzpicture}
}
\caption{The neighborhood routing structure.}
\label{fig:neighborhoods}
\end{figure}
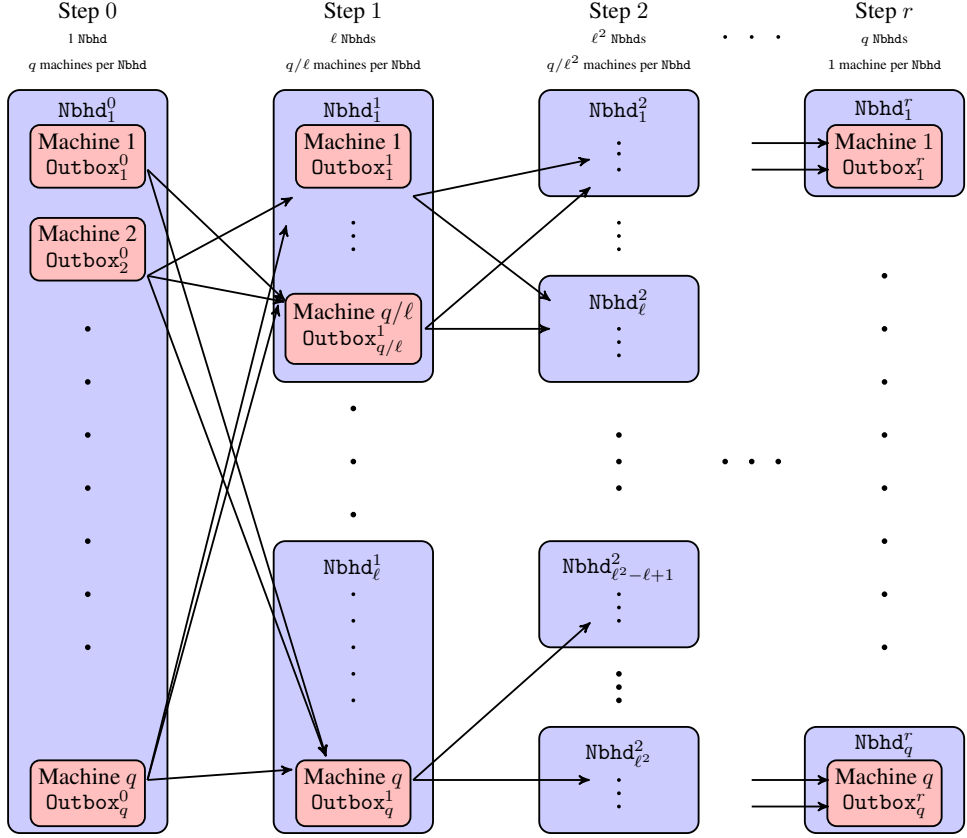

We define a collection of MPC machine states as ``intermediate outboxes'' $\outbox^t = (\outbox^t_1, \dots, \outbox^t_q)$ for each step $t$ to represent an assignment of message packets to machines in the same $t$th-level neighborhood as the packet destination.
That is, we require that each $\outbox^t_i$ satisfy
% \[\outbox^t_i = \set{(\msg, \src, \dst): \set{\dst, i} \in \nbhd^t_j \ \text{for some} \ j \in [b_t].}\]
\begin{equation}\label{eq:outbox}
\outbox^t_i = \set{(\msg, \src, \dst): i \in \nbhd^t(\dst)}
\end{equation}
We say that $\outbox^t$ are \emph{valid intermediate outboxes} for $\outbox$ if:
\begin{enumerate}
    \item $\outbox^t$ satisfies \Cref{eq:outbox};
    \item there is a one-to-one mapping between packets in $\outbox$ and $\outbox^t$; and
    \item local memory constraints are maintained up to a factor of two\footnote{For the sake of simplicity, we assume that the encoding $\msg$ also contains all relevant metadata about the packet $\src$ and $\dst$. That is, we only track the message size $|\msg|$.}
\[\sum_{\msg \in \outbox^t_i} |\msg| \leq 2s.\]
\end{enumerate}
% (1)  (2)  and (3) 

Note that $\outbox$ satisfies the conditions for $\outbox^0$, and the only sequence of embeddings that satisfies the conditions for $\outbox^r$ is $\inbox$.
An inductive argument that applies \Cref{lemma:mpc-sim-step} $r$ times completes the proof.
\end{proof}

% \todo{resolve $\rho < \delta$}

\begin{lemma}
[$(\gamma, \delta, \rho)$-MPC simulates one intermediate step]\label{lemma:mpc-sim-step}
For any $t \in \set{ 0, \dots, r-1}$ and $\rho < \frac\delta2$, there exists an $O(\frac{1 + \gamma}{\rho})$-round $(\gamma, \delta, \rho)$-MPC protocol that takes as input any satisfactory $\outbox^t$ and returns some satisfactory $\outbox^{t+1}$.
\end{lemma}

\begin{proof}
The protocol operates in two phases:
\begin{enumerate}
    \item The \emph{pre-computation phase}, where relevant message size metadata is computed using $O(r - t)$ rounds of communication.
    \item The \emph{message-passing phase}, where all messages are propagated in $r - t$ rounds of communication in order to convert $\outbox^t$ to $\outbox^{t+1}$.
\end{enumerate}
Since $r = O(\frac{1 + \gamma}\rho)$, the bound on total number of rounds is satisfied.

The algorithm maintains the invariant that all communication occurs within $t$-level neighborhoods $\nbhd^t_j$ without any mixing between neighborhoods; this is possible because of the assumed validity of $\outbox^t$, which implies that all messages whose packets appear in $\nbhd^t_j$ of $\outbox^t$ have ultimate destinations in $\nbhd^t_j$.
Concretely, if $(\msg, \src, \dst) \in \outbox^t_i$ and $i \in \nbhd^t_j$, then $\dst \in \nbhd^t_j$.

We first explain the routing strategy for the message-passing phase by describing an itinerary of machines that each message packet will be routed to and proving that this itinerary meets the conditions needed to be executed by an $(r-t)$-round $(\gamma, \delta, \rho)$-MPC protocol.
We then describe how the metadata necessary for the itinerary computations can be obtained during the pre-computation phase.

\paragraph*{Message-passing phase.}
We first introduce some \textit{packet} notation to track all relevant metadata about any message in the $t$th intermediate step.
Fix some particular input $\outbox^t$.
\begin{itemize}
    \item Let $\packet = (\msg, \src, \dst, \src^t)$ denote a \textit{packet} that contains a message $\msg$ and metadata concerning its ``global'' source $\src$ and destination $\dst$ (i.e. $(\msg, \src, \dst) \in \outbox_\src, \inbox_\dst$) and its ``local'' source $\src^t$ from within $\outbox^t$ (i.e. $(\msg, \src, \dst) \in \outbox^t_{\src^t}$).
    \item We write $\packet \in \outbox^t_i$ if $\src^t = i$ %$(\msg, \src, \dst) \in \outbox^t_i$ 
    and $\packet \in \nbhd$ if $\packet \in \outbox^t_i$ for some $i \in \nbhd$.
    \item Let $\srcnbhd^{t, \tau}(\packet) = \nbhd^\tau(\src^t)$ represent the neighborhood of size $\ell^{r-\tau}$ contains $\packet$ (i.e. $\packet \in \srcnbhd^{t, \tau}(\packet)$) and $\dstnbhd^{t}(\packet) = \nbhd^{t+1}(\dst)$ denote the neighborhood of size $\ell^{r-t-1}$ that contains the ultimate destination $\dst$. Because $\packet \in \nbhd^t_j$ if and only if $\dst \in \nbhd^t_j$, it follows that $\dstnbhd^t(\packet) \subset \srcnbhd^{t, t}(\packet)$.
    \item Let $|\packet| = |\msg|$ be the total number of bits needed to encode the packet.
    \item We establish a lexicographical ordering that depends first on $\src^t$ and next on $\dst$.
    That is, we say that $P' < P$ for some $P' = (\msg', \src', \dst', \src^{t \prime})$ if $\src^{t\prime} < \src^t$, or if $\src^{t \prime} = \src^t$ and $\dst' < \dst$.

\end{itemize}

We develop a \emph{packet scoring function} $z^{t, \tau}(\packet)$, which in turn induces an \emph{itinerary function} $b^{t, \tau}(\packet) \in [q]$ that assigns an intermediate machine to hold packet $\packet$ after $r-\tau$ communication steps.
These functions are carefully chosen in order to ensure that the local memories of each individual machine are bounded by $O(s)$ and that the number of machine each sends messages to and receives messages from is bounded by $O(\ell^2)$.
Critically, we ensure that $b^{t, \tau}(\packet) \in \srcnbhd^{t, \tau}(\packet)$; that is, we initially rearrange packets within solely within the smallest neighborhoods of $\outbox^t$ and gradually propagate packets more widely.
That way, each packet $\packet$ obeys a trajectory of the following form over $r-t$ MPC rounds:
\begin{multline*}
    b^{t, r}(\packet) = \src^t \in \srcnbhd^{t, r}(\packet) \Longrightarrow b^{t, r-1}(\packet) \in \srcnbhd^{t, r-1}(\packet) \Longrightarrow \dots \\  \Longrightarrow b^{t, t+1}(\packet) \in \srcnbhd^{t, t+1}(\packet)  \Longrightarrow b^{t, t}(\packet) \in \dstnbhd^t(\packet).
\end{multline*}

To define these functions, we first define partial sums of message sizes in order to later determine an itinerary of machines that $\packet$ will be routed to in between $\src^t$ and some destination machine in $\dstnbhd^t(\packet)$.
The first term, $\ps^{t, \tau}(\packet)$, sums the size of all ``lesser'' packets that share a destination neighborhood $\dstnbhd^t(\packet)$ and a $\tau$th level input neighborhood $\srcnbhd^{t, \tau}(\packet)$:
% \begin{align*}
%     \ps^{t, \tau}(\packet) &= \sum_{\packet' \in \srcnbhd^{t, \tau}(\packet)} |\packet'| \cdot \indicator{\dstnbhd^t(\packet') = \dstnbhd^t(\packet), \ \packet' < \packet}.
% \end{align*}
\begin{align*}
    \ps^{t, \tau}(\packet) &= \sum_{\substack{\packet' \in \srcnbhd^{t, \tau}(\packet) \\ \dstnbhd^t(\packet') = \dstnbhd^t(\packet) \\ \packet' < \packet}} |\packet'|.
\end{align*}
The second term, $\lds^{t, \tau}(\packet)$, sums the size of all packets that share an input neighborhood but have a ``smaller'' destination neighborhood than $\packet$:
\begin{align*}
 \lds^{t, \tau}(\packet) &= \lds^t(\srcnbhd^{t, \tau}(\packet), \dstnbhd^{t}(\packet)),
\end{align*}
where
\begin{align*}
    % \lds^t(\srcnbhd, \dstnbhd) &= \sum_{\packet' \in \srcnbhd} |\packet'| \cdot \indicator{\dstnbhd^{t}(\packet') < \dstnbhd}.
    \lds^t(\srcnbhd, \dstnbhd) &= \sum_{\substack{\packet' \in \srcnbhd \\ \dstnbhd^{t}(\packet') < \dstnbhd}} |\packet'|.
\end{align*}

We now define the packet scoring and itinerary functions for any $\tau \in \set{t, \dots, r}$:
\[z^{t, \tau}(\packet) = \begin{cases}
% 2s \cdot \firstchild(\srcnbhd^{t, \tau}(\packet)) + \lds^{t, \tau}(\packet) + \ps^{t, \tau}(\packet) & \tau \geq t + 1, \\
% 2s \cdot \firstchild(\dstnbhd^{t}(\packet))  + 2 \cdot \ps^{t, \tau}(\packet) & \tau = t.
2s \cdot \min \srcnbhd^{t, \tau}(\packet) + \lds^{t, \tau}(\packet) + \ps^{t, \tau}(\packet) & \tau \geq t + 1, \\
2s \cdot \min \dstnbhd^{t}(\packet)  + 2 \cdot \ps^{t, \tau}(\packet) & \tau = t.
\end{cases}\]
\[b^{t, \tau}(\packet) = \floor{\frac{z^{t, \tau}(\packet)}{2s}}.\]
We prove a series of claims to establish that the packet scoring and itinerary functions are properly defined.

\begin{claim}[Itinerary range]\label{claim:range}
For any $\tau \in \set{t, \dots, r}$, the itinerary function satisfies
\[b^{t, \tau}(\packet) \in \begin{cases}
\srcnbhd^{t, \tau}(\packet) & \tau \geq t + 1 \\
\dstnbhd^t(\packet) & \tau = t.
\end{cases}\]
As an immediate consequence, $b^{t, r}(\packet) = \src^t$ for $\packet = (\msg, \src, \dst, \src^t)$.
\end{claim}
\begin{proof}
We first bound the scoring function $z^{t, \tau}$.
Note that 
\begin{align*}
\lds^{t, \tau}(\packet) + \ps^{t, \tau}(\packet) 
&\leq \sum_{\packet' \in \srcnbhd^{t, \tau}(\packet)} |\packet'| - |\packet| \\
&\leq |\srcnbhd^{t, \tau}(\packet)| \cdot 2s - |\packet|.
\end{align*}
Therefore, for $\tau > t$, 
\begin{equation}\label{eq:src-nbhd-gap}
    % z^{t, \tau}(\packet) \in \left[2s \cdot \firstchild(\srcnbhd^{t, \tau}(\packet)), 2s \cdot (\firstchild(\srcnbhd^{t, \tau}(\packet)) +  |\srcnbhd^{t, \tau}(\packet)|) - |\packet|\right],
    z^{t, \tau}(\packet) \in \left[2s \cdot \min\srcnbhd^{t, \tau}(\packet), 2s \cdot (\min\srcnbhd^{t, \tau}(\packet) +  |\srcnbhd^{t, \tau}(\packet)|) - |\packet|\right],
\end{equation}
and
% \[b^{t, \tau}(\packet) \in [\firstchild(\srcnbhd^{t, \tau}(\packet)), \firstchild(\srcnbhd^{t, \tau}(\packet)) +  |\srcnbhd^{t, \tau}(\packet)|),\]
\[b^{t, \tau}(\packet) \in [\min\srcnbhd^{t, \tau}(\packet), \min\srcnbhd^{t, \tau}(\packet) +  |\srcnbhd^{t, \tau}(\packet)|),\]
which proves the first case of the claim.
The second case follows by observing that 
\[\ps^{t, t}(\packet) \leq s \cdot |\dstnbhd^t(\packet)| - |\packet|.\]
Were it not true, there would exist at least one machine $i \in \dstnbhd^t(\packet)$ that must receive more an $s$ quantity of messages at the end of the entire round of the protocol $\pi$ (i.e. $\sum_{\msg \in \inbox_i} |\msg| > s$), which contradicts the MPC assumption.
\end{proof}

\begin{claim}[Gaps between scores]\label{claim:gaps}
If $\packet_1 \neq \packet_2$ and $z^{t, \tau}(\packet_1) \leq z^{t, \tau}(\packet_2)$, then \[ z_{t, \tau}(\packet_1) + |\packet_1| \leq z^{t, \tau}(\packet_2).\]
\end{claim}
\begin{proof}
First, let $\tau > t$.
Consider the case where $\srcnbhd^{t, \tau}(\packet_1) \neq \srcnbhd^{t, \tau}(\packet_2)$.
By \Cref{claim:range} and our assumption that $z^{t, \tau}(\packet_1) \leq z^{t, \tau}(\packet_2)$, it must be the case that $\srcnbhd^{t, \tau}(\packet_1) < \srcnbhd^{t, \tau}(\packet_2)$.
Hence, we have the following by applying \Cref{eq:src-nbhd-gap}.
\begin{align*}
    &z^{t, \tau}(\packet_2) - z_{t, \tau}(\packet_1)\\
    % &\quad\geq 2s \cdot (\firstchild(\srcnbhd^{t, \tau}(\packet_2)) - \firstchild(\srcnbhd^{t, \tau}(\packet_1)) - |\srcnbhd^{t, \tau}(\packet_1)| + |\packet_1|  \\
    &\quad\geq 2s \cdot (\min\srcnbhd^{t, \tau}(\packet_2) - (\min\srcnbhd^{t, \tau}(\packet_1) + |\srcnbhd^{t, \tau}(\packet_1)| - |\packet_1|))  \\
    &\quad\geq |\packet_1|.
\end{align*}
Otherwise, if $\srcnbhd^{t, \tau}(\packet_1) = \srcnbhd^{t, \tau}(\packet_2)$, then 
\begin{align*}
    z^{t, \tau}(\packet_2) - z_{t, \tau}(\packet_1)
    &= (\lds^{t, \tau}(\packet_2) + \ps^{t, \tau}(\packet_2)) - (\lds^{t, \tau}(\packet_1) + \ps^{t, \tau}(\packet_1))  \\
    &= \sum_{\packet' \in S_2} |\packet'| - \sum_{\packet' \in S_1} |\packet'|,
\end{align*}
for some packet subsets $S_1, S_2 \subset \srcnbhd^{t, \tau}(\packet_1)$.
By further inspecting the respective $\lds^{t,\tau}$ and $\ps^{t, \tau}$ terms, we observe that $S_1 \subset S_2$ and $\packet_1 \in S_2 \setminus S_1$.
The claim immediately follows.

The argument for the case $\tau = t$ is nearly identical.
\end{proof}

\begin{claim}[Local memory bound]\label{claim:local}
For any $b \in \N$, 
\begin{align*}
    % \sum_{\packet} |\packet| \indicator{b^{t, \tau}(\packet) = b} \leq \begin{cases}
    \sum_{\packet: b^{t, \tau}(\packet) = b} |\packet| \leq \begin{cases}
    2s & \tau \in \set{t, r} \\
    3s & \tau \in \set{t + 1, \dots, r}.\\
    \end{cases}
\end{align*}
\end{claim}
\begin{proof}
The case $\tau = r$ is an immediate consequence of the inductive assumption that $\outbox^t$ satisfies the desired intermediate properties.

For all other cases, let $\set{\packet_1, \dots, \packet_n}$ denote all packets with $b^{t, \tau}(\packet_i) = b$ and let $z^{t, \tau}(\packet_1) \leq \dots \leq z^{t, \tau}(\packet_n)$ without loss of generality.
We use \Cref{claim:gaps}, the assumption that all $|\packet_i| \leq s$, and the boundedness of $z^{t, \tau}(\packet_i)$ from \Cref{claim:range} to conclude the proof.
\begin{align*}
    \sum_{i=1}^n |\packet_i|
    &\leq \sum_{i=1}^{n-1} (z^{t, \tau}(\packet_{i+1}) - z^{t, \tau}(\packet_i)) + |\packet_n| \\
    &\leq z^{t, \tau}(\packet_{n}) - z^{t, \tau}(\packet_1) + s \\
    &\leq \begin{cases}
    2s & \tau = t, \\
    3s & \tau > t.
    \end{cases}\qedhere
\end{align*}
\end{proof}

\begin{claim}[Intra-class distance preservation]\label{claim:preserve}
If $\packet_1$ and $\packet_2$ satisfy $\srcnbhd^{t, \tau+1}(\packet_1) = \srcnbhd^{t, \tau+1}(\packet_2)$ and $\dstnbhd^t(\packet_1) = \dstnbhd^t(\packet_2)$, then
\[z^{t, \tau}(\packet_1) - z^{t, \tau}(\packet_2) = z^{t, \tau+1}(\packet_1) - z^{t, \tau+1}(\packet_2).\]
\end{claim}
\begin{proof}
Since $\srcnbhd^{t, \tau+1}(\packet_1) = \srcnbhd^{t, \tau+1}(\packet_2)$, it follows that $\srcnbhd^{t, \tau}(\packet_1) = \srcnbhd^{t, \tau}(\packet_2)$ and therefore,
\begin{align*}
    z^{t, \tau}(\packet_1) - z^{t, \tau}(\packet_2) &= \ps^{t, \tau}(\packet_1) - \ps^{t, \tau}(\packet_2) \\
    % &= \sum_{\substack{\packet' \in \srcnbhd^{t, \tau}(\packet_1) \\ \dstnbhd^t(\packet') = \dstnbhd^t(\packet_1)}} |\packet'| \cdot \indicator{\packet_1 \leq \packet' < \packet_2}\\
    &= \sum_{\substack{\packet' \in \srcnbhd^{t, \tau}(\packet_1) \\ \dstnbhd^t(\packet') = \dstnbhd^t(\packet_1) \\ \packet_1 \leq \packet' < \packet_2}} |\packet'| \numberthis \label{eq:tau} \\
    z^{t, \tau+1}(\packet_1) - z^{t, \tau+1}(\packet_2) &= \ps^{t, \tau+1}(\packet_1) - \ps^{t, \tau+1}(\packet_2) \\
    % &= \sum_{\substack{\packet' \in \srcnbhd^{t, \tau+1}(\packet_1) \\ \dstnbhd^t(\packet') = \dstnbhd^t(\packet_1)}} |\packet'| \cdot \indicator{\packet_1 \leq \packet' < \packet_2}.
    &= \sum_{\substack{\packet' \in \srcnbhd^{t, \tau+1}(\packet_1) \\ \dstnbhd^t(\packet') = \dstnbhd^t(\packet_1) \\ \packet_1 \leq \packet' < \packet_2}} |\packet'| \numberthis \label{eq:tau1}.
\end{align*}
Because $\packet_1, \packet_2 \in \srcnbhd^{t, \tau+1}(\packet_1)$, the defined packet ordering implies that any $\packet' \in [\packet_1, \packet_2)$ must satisfy $\packet' \in \srcnbhd^{t, \tau+1}(\packet_1)$.
Therefore, \Cref{eq:tau,eq:tau1} are equal and the claim holds.
\end{proof}

\begin{claim}[Distinct recipients bound]\label{claim:recipients}
For any $b$,
\[|\set{b^{t, \tau}(\packet): b^{t, \tau+1}(\packet) =  b}| \leq 3\ell.\]
\end{claim}
\begin{proof}
Within each root neighborhood $\nbhd^\tau_j$, there exist at most $\ell$ destination neighborhoods in $\dstnbhd^{t}(\packet)$ for $\packet \in \nbhd^\tau_j$.

Fix some such $\dstnbhd$ and let $\packet_1, \dots, \packet_n$ denote all packets with $b^{t, \tau+1}(\packet) = b$ and $\dstnbhd^{t}(\packet_i) = \dstnbhd$.
Without loss of generality, assume that $z^{t, \tau}(\packet_1) \leq \dots \leq z^{t, \tau}(\packet_n)$.
Because all such packets belong to the same machine in step $r-\tau-1$ (i.e. $b^{t, \tau+1}(\packet_i) = b$ for all $i \in [n]$), they belong share the same source neighborhood of size $\ell^{r - \tau-1}$ (i.e. $\srcnbhd^{t, \tau+1}(\packet_i) = \srcnbhd^{t, \tau+1}(\packet_1)$).
By \Cref{claim:preserve} and the definition of $b^{t, \tau}$,
\begin{align*}
    b^{t, \tau}(\packet_i) - b^{t, \tau}(\packet_1)
    &\leq 1 + \frac1{2s}(z^{t, \tau}(\packet_i) - z^{t, \tau}(\packet_1)) \\
    &= 1 +\frac1{2s}( z^{t, \tau+1}(\packet_i) - z^{t, \tau+1}(\packet_1)) \\
    &\leq 2 + b^{t, \tau+1}(\packet_i)  - b^{t, \tau+1}(\packet_1) = 2.
\end{align*}
Therefore, there are at most three possible values of $b^{t, \tau}(\packet_i)$.

The claim follows by considering each of the $\ell$ destination neighborhoods separately:
\begin{align*}
\abs{\set{b^{t, \tau}(\packet): b^{t, \tau+1}(\packet) =  b}}
&= \sum_{\dstnbhd} \abs{\set{b^{t, \tau}(\packet): b^{t, \tau+1}(\packet) =  b, \ \dstnbhd^t(\packet) = \dstnbhd}} \\
&\leq 3\ell.\qedhere
\end{align*}
\end{proof}

\begin{claim}[Distinct senders bound]\label{claim:senders}
For any $b$,
\[|\set{b^{t, \tau+1}(\packet): b^{t, \tau}(\packet) =  b}| \leq 3\ell^2.\]
\end{claim}
\begin{proof}
Within each $\nbhd^\tau_j$, there exist at most $\ell^2$ distinct pairs of destination neighborhoods $\dstnbhd^t(\packet)$ and source neighborhoods $\nbhd^{t, \tau}(\packet)$ for $\packet \in \nbhd^\tau_j$.

As before, we fix some $\dstnbhd$ and $\srcnbhd$ and let $\packet_1, \dots, \packet_n$ all satisfy $b^{t, \tau}(\packet_i) = b$, $\dstnbhd^{t}(\packet_i) = \dstnbhd$, and $\srcnbhd^{t, \tau+1} = \srcnbhd$.
Using the same argument, we show that \[\abs{\set{b^{t, \tau+1}(P_i): i \in [n]}} \leq 3.\]
We conclude by considering all such pairs.
\end{proof}

As a result of Claims~\ref{claim:local}, \ref{claim:recipients}, and \ref{claim:senders}, we conclude that each packet $\packet = (\msg, \src, \dst, \src^t)$ can be equipped with some itinerary \[\src^t = b^{t,r}(\packet), \ b^{t, r-1}(\packet), \dots, b^{t, t+1}(\packet), \ b^{t, t}(\packet) \in \dstnbhd^t(\packet)\]
that properly translates an instances of $\outbox^t$ to $\outbox^{t+1}$ and does so without ever requiring local memory more than $3s = O(N^\delta)$ on any intermediate step or any machine to send or receive messages from more than $3\ell^2 = O(N^\rho)$ other machines.
This itinerary can be executed using an $(r - t)$-round $(\gamma, \delta, \rho)$-MPC protocol.

\paragraph*{Pre-computation phase.}
It remains to show that each $b^{t, \tau}$ can be computed for each packet.
To do so, we prove that there exists an $O(r-t)$-round $(\gamma, \delta, \rho)$-MPC protocol that ends with each machine $i$ knowing $\ps^{t,\tau}(\packet)$ and $\lds^{t, \tau}(\packet)$ for each $\packet \in \outbox^t_i$ and $\tau \in \{t, \dots, r\}$.
We design an MPC protocol that uses the tree structure to propagate information about individual child neighborhoods to their parents and vice versa.
We describe recursive relationships that elucidate how to compute the two salient quantities.

First, we introduce a useful intermediate term.
% Let \[\ps^{t, \tau}(i, \dstnbhd) = \sum_{\substack{i' \in \nbhd^\tau(i) \cap [0, i) }} \sum_{\packet' \in \outbox^t_i} |\packet'|\cdot \indicator{\dstnbhd^t(\packet') = \dstnbhd},\]
Let \[\ps^{t, \tau}(i, \dstnbhd) = \sum_{\substack{i' \in \nbhd^\tau(i) \\ i' < i  }} \sum_{\substack{\packet' \in \outbox^t_{i'} \\ \dstnbhd^t(\packet') = \dstnbhd}} |\packet'|,\]
denote the sum of all packet that are contained by ``lesser'' machines that share source and destination neighborhoods.
Note that $\ps^{t, \tau}(\packet)$ can be computed for any $\packet \in \outbox^t_i$ locally by machine $i$ given prior knowledge of $\ps^{t, \tau}(i, \dstnbhd^t(\packet))$. 
Thus, the pre-computation phase need only compute the latter term.

We also introduce a term that represents sum of the sizes of all packets that share source and destination neighborhoods:
% \[\ns^t(\srcnbhd, \dstnbhd) = \sum_{\packet' \in \srcnbhd} |\packet'| \cdot \indicator{\dstnbhd^t(\packet') = \dstnbhd}.\]
\[\ns^t(\srcnbhd, \dstnbhd) = \sum_{\substack{\packet' \in \srcnbhd \\ \dstnbhd^t(\packet') = \dstnbhd}} |\packet'|.\]
Now, we provide the recurrences for any $\tau < r$ (or for any $\srcnbhd$ satisfying $|\srcnbhd| > 1$:
\begin{align*}
    &\ps^{t, \tau}(i, \dstnbhd) 
    \\&\quad= 
    % \begin{cases}
    % \ps^{t, \tau + 1}(i, \dstnbhd) + \ps^{t, \tau}(\firstchild(\nbhd^{\tau+1}(i)), \dstnbhd) \\ %& \tau < r \\
    \ps^{t, \tau + 1}(i, \dstnbhd) + \ps^{t, \tau}(\min\nbhd^{\tau+1}(i), \dstnbhd) \\ %& \tau < r \\
    % 0 & \tau = r\end{cases}\\
    \\&\quad= 
    %\begin{cases}
    \ps^{t, \tau + 1}(i, \dstnbhd) + \sum_{\substack{\srcnbhd \in \children(\nbhd^\tau(i)) \\ \srcnbhd < \nbhd^{\tau+1}(i)}} \ns^t(\srcnbhd, \dstnbhd), \\%& \tau < r \\
    % 0 & \tau = r\end{cases}\\
    &\ns^t(\srcnbhd, \dstnbhd) =%\\&\quad= 
    \sum_{\srcnbhd' \in \children(\srcnbhd)} \ns^{t}(\srcnbhd, \dstnbhd), \numberthis \label{eq:ns-rec} \\
    &\lds^{t}(\srcnbhd, \dstnbhd) =%\\&\quad= 
    % \begin{cases}
    \sum_{\srcnbhd' \in \children(\srcnbhd)} \lds^{t}(\srcnbhd, \dstnbhd) \numberthis \label{eq:lds-rec}.% & \tau < r \\
    % \sum_{\packet' \in \srcnbhd} |\packet'| \cdot \indicator{\dstnbhd^{t}(\packet') < \dstnbhd} & \tau = r.
    % \end{cases}
\end{align*}
When $\tau = r$, the terms $\ps^{t, \tau}(i, \dstnbhd)$, $\ns^t(\nbhd^r(i), \dstnbhd)$, and $\lds^t(\nbhd^r(i), \dstnbhd)$ can be computed locally within machine $i$.% cases can be determined by looking solely at the relevant machine $i$ (where $\srcnbhd = \{i\}$ for the latter).

We follow a tree-like communication pattern to compute all relevant sums.
Each machine $i \in [q]$ computes \[\set{\ps^{t, \tau}(i, \dstnbhd): \tau \geq t, \ \dstnbhd \in \children(\nbhd^t(i))}\]
and
\[\set{\lds^{t}(\nbhd^\tau(i), \dstnbhd): \tau \geq t, \ \dstnbhd \in \children(\nbhd^t(i))}\]
by completing $r-t$ \textit{propagate-up} rounds, $r-t$ \textit{aggregation} rounds, and $r-t$ \textit{propagate-down} rounds.
\begin{itemize}
    \item The propagate-up rounds compute the neighborhood-wide message-size summations $\ns^t(\nbhd^\tau(i), \dstnbhd)$ and $\lds^t(\nbhd^\tau(i), \dstnbhd)$ in each machine $i$ satisfying
    % $\firstchild(\nbhd^\tau(i))$ 
    $i = \min\nbhd^\tau(i)$
    for each $\tau$ and $\dstnbhd$.
    \item The aggregation rounds accumulate $\ns$ terms of the same level into specific $\ps$ terms.
    \item The propagate-down rounds iteratively compute and share the $\ps$ and $\lds$ terms with all relevant machines.
\end{itemize}

\textit{Propagate-up rounds:} Fix some neighborhood $\nbhd_j^t$.
After $r-\tau$ propagate-up rounds, the goal is to compute the terms $\ns^t(\srcnbhd, \dstnbhd)$ and $\lds^t(\srcnbhd, \dstnbhd)$ for each relevant destination neighborhood $\dstnbhd \in \children(\nbhd_j^t)$ and source neighborhood $\srcnbhd \in \desc^{\tau}(\nbhd_j^t)$ within a single machine in each source neighborhood, 
% $\firstchild(\srcnbhd)$.
$\min \srcnbhd$.
We argue that this is possible inductively.

Before performing any computation (that is, after ``round 0''), each machine $i$ individually computes $\ns^t(\nbhd^r(i), \dstnbhd)$ and $\lds^t(\nbhd^r(i), \dstnbhd)$ by aggregating the messages encoded in its own representation of $\outbox_i^t$.

We assume that the procedure works as specified for $r - \tau $ rounds of communication. 
Fix some $\srcnbhd^{\tau-1} \in \desc^{\tau-1}(\nbhd^t_j)$.
Then, for every $\srcnbhd^\tau \in \children{\srcnbhd^{\tau-1}}$, the quantities
\begin{align*}
    &\{\ns^{t}(\srcnbhd^\tau, \dstnbhd): \dstnbhd \in \children(\nbhd^t_j)\} \\
    &\quad \cup \{ \lds^{t}(\srcnbhd^\tau, \dstnbhd): \dstnbhd \in \children(\nbhd^t_j)\}
\end{align*}
have already been computed and stored in $\min\srcnbhd^\tau$.
% Observe that $\ns^{\tau + 1}$ can be rewritten as follows:
% \begin{multline*}
%     \ns^{\tau - 1}(\innbhd^{\tau-1}, \outnbhd) \\= \sum_{\innbhd^\tau \in \children(\innbhd^{\tau-1})} \ns^{\tau}(\innbhd^{\tau}, \outnbhd).
% \end{multline*}
By the recurrence relations of \Cref{eq:ns-rec,eq:lds-rec}, $\ns^t(\srcnbhd^{\tau-1}, \dstnbhd)$ and $\lds^t(\srcnbhd^{\tau-1}, \dstnbhd)$ are functions of those quantities.
Thus, it suffices to have each machine $\min \srcnbhd^\tau$ machine transmit its relevant terms to $\min \srcnbhd^{\tau-1}$.
A round of MPC communication that transmits such messages involves each machine sending most one message of size $\ell$ and receiving at most $\ell$ messages, each of size $O(\ell)$.

Inductively, we ensure that all neighborhood sums are computed after $r-t$ propagate-up rounds. 
Because each machine handles at most $\ell$ distinct messages having total size $O(\ell^2)$ per MPC round, this protocol does not violate the bounded message size and bounded distinct message constraints (so long as $\ell^2 \ll s$), which can be guaranteed for sufficiently large $n$, so long as $\ell = O(n^{\rho/2})$.

\textit{Aggregation rounds:} 
After the completion of the aggregation rounds, each machine $i$ computes terms of the form \[\set{\ps^{t, \tau}(\min\nbhd^{\tau+1}(i), \dstnbhd): \dstnbhd \in \children(\nbhd^t_j)}\]
from relevant $\ns^t$ terms if $i = \min\nbhd^{\tau+1}(i)$.
By the recurrence, it is sufficient for machine $i$ to following $\ell^2$ distinct terms 
\[\set{\ns^t(\srcnbhd, \dstnbhd) : \srcnbhd \in \children(\nbhd^{\tau}(i)),  \dstnbhd \in \children(\nbhd^t_j)}.\]
Since all machines $i$ already knows such terms for $\srcnbhd = \nbhd^{\tau+1}(i)$, it can obtain the remaining $\ns^t$ terms by simultaneously sharing information with its ``cousin machines:'' $\min \srcnbhd$, for all $\srcnbhd \in \children(\nbhd^{\tau}(i))$.
% Now, we explain how each desired partial sum can be computed from relevant neighborhood sums.
% Note the following alternate representation of each partial sum:
% \begin{align*}
%     &\ps_i^{t, \tau}(\outnbhd) \\
%     &\quad= \ps_i^{t, \tau+1}(\outnbhd) + \sum_{\substack{\innbhd \in \children(\nbhd^\tau(i))\\ \innbhd < \nbhd^{\tau+1}(i)}} \ns^{\tau+1}(\innbhd, \outnbhd) \\
%     &\quad= \ps_i^{t, \tau+1}(\outnbhd) + \ps_{\firstchild(\nbhd^{\tau+1}(i))}^{t, \tau}(\outnbhd).
% \end{align*}
% By the recurrence, we can compute all such sums for a given machine $i$ if the machine is provided with \[\set{\ps_{\firstchild(\nbhd^{\tau+1}(i))}^{t, \tau}(\outnbhd) : \tau \geq t, \ \outnbhd \in \children(\nbhd^t(i))}.\]
% At the end of the propagate-up rounds, the term $\ns^{\tau+1}(\innbhd, \cdot)$ for any $\innbhd \in \desc^\tau(\nbhd^t_j)$ is known by machine $\firstchild(\innbhd)$.
% Therefore, the machine $\firstchild(\nbhd^{\tau+1}(i))$ can compute $\ps_{\firstchild(\nbhd^{\tau+1}(i))}^{t, \tau}$ after a single round if all $\ns^{\tau+1}(\innbhd, \cdot)$ terms for $\innbhd \in \children(\nbhd^\tau(i))$ are shared with it.
This can be handled by a single round of communication where each ``$(\tau+1)$th-level neighborhood representative'' machine forwards its sums sums to up to $\ell$ other first representatives, for a total messaging cost of $O(\ell^2)$. 

We use $r - t$ separate rounds to repeat this process for each $\tau \geq t$.

\textit{Propagate-down rounds:}
It remains to compute each $\ps^{t, \tau}(i, \cdot)$ and $\lds^t(\nbhd^\tau(i), \cdot)$ term at each machine $i$.
The relevant $\lds^t$ terms have already been computed by each respective $\min\nbhd^\tau(i)$ machine and can be propagated to machine $i$ by using $r - t$ rounds of tree-propagation through intermediate $\min\nbhd^{\tau'}(i)$ machines.

The same is possible for $\ps^{t, \tau}$ terms, although the individual terms need to be added in order to follow the recurrence. 
We propagate the terms in the same way as the $\lds^t$ terms, but we take special care to carry out the extra additions.
This can be accomplished simultaneously to the other propagate-down rounds.
This protocol involves each first-child node sharing at most $\ell$ distinct messages, each of size at most $O((r-t) \ell)$. 
As before, every node sends and receives at most $O((r-t)\ell^2) \ll s$ words.

After these $3(r - t)$ rounds have elapsed, each machine $i$ computes $b^{t, \tau}(\packet)$ for each $\packet \in \outbox^t_i$.
Using this itinerary, the machines routes tuples of the form \[(\packet, b^{t, r}(\packet), \dots, b^{t,t}(\packet))\]
to the respective machine $b^{t, \tau}(\packet)$ in round $r - \tau$.
Due to the arguments presented at the start of the section, this procedure terminates with each $(\msg, \src, \dst)$ tuple being held by some machine $i$ such that the resulting $\outbox^{t+1}_i$ is valid, and the procedure involves at most $O(r - t) = O(\frac{1 + \gamma}\rho)$ rounds of $(\gamma, \delta, \rho)$-MPC computation.
\end{proof}

\subsubsection{Proof outline of \Cref{cor:mpc-sht24}}\label{assec:mpc-sht24}

\cormpcsht*

The proof of \Cref{cor:mpc-sht24} involves an adaptation to the proof of Theorem~3.1 of \cite{sht24}.
To avoid restating the proof in its entirety, we provide a brief outline of the proof of Theorem~3.1 and explain which modification is necessary.

Theorem~3.1 is a consequence of Lemmas~B.4, B.5, and B.6 of \cite{sht24}, which establish that there exist single-layer transformers that simulate the initialization, a round of computation and communication, and the output formatting of any fixed $(\gamma, \delta)$-MPC protocol. 
The input and output steps of $(\gamma, \delta)$, and $(\gamma, \delta, \rho)$-MPC protocols are identical, only Lemma~B.5 needs to be examined.

To simulate a single round of an MPC protocol with a transformer, all local computations are simulated in the element-wise multi-layer perceptron (MLP) units, and all communication is handled in a single multi-headed self-attention layer (Lemma~B.7)
Since $(\gamma, \delta, \rho)$-MPC protocols add no restrictions related to local computation, the former can be simulated in exactly the same manner with identical MLPs, and it remains to analyze the construction of Lemma~B.7.
We restate Lemma~B.7 and provide an replacement lemma that suffices to prove \Cref{cor:mpc-sht24}.

\begin{lemma}[Lemma~B.7 of \cite{sht24}; multi-headed attention simulates MPC communication]\label{lemma:route-sht24}
For any $R$-round MPC protocol with local memory $s$ and $q$ machines and any round $r \in [R-1]$, there exists a single-layer transformer $f \in \trans{m, H, 1}{q, 0}$ with $H = O(\log\log q)$ and $m = O(s^4 \log q)$, which, given as input a length-$q$ encoding of each machine's outgoing messages in round $r$, returns an encoding of each machine's incoming messages in round $r + 1$.
\end{lemma}

\begin{lemma}[Single-headed attention simulates bounded-message MPC communication]\label{lemma:route-new}
For any $R$-round MPC protocol with local memory $s$, $q$ machines, and a $k$-machine communication limit and any round $r \in [R-1]$, there exists a single-layer single-headed transformer $f \in \trans{m, H, 1}{q, 0}$ with $H = 1$ and $m = O(k^4 s \log q)$, which, given as input a length-$q$ encoding of each machine's outgoing messages in round $r$, returns an encoding of each machine's incoming messages in round $r + 1$.
\end{lemma}

\Cref{lemma:route-new} is an immediate consequence of Lemma~3.2 of \cite{sht24}, their main technical result, which already applies to the regime with limits on the number of machines in communication.

By replacing \Cref{lemma:route-sht24} with \Cref{lemma:route-new}, applying the remainder of the proof of Theorem~3.1 of \cite{sht24}, and letting $k = O(n^\rho)$, the proof of \Cref{cor:mpc-sht24} is complete.

%% file: graphs_old.tex
\section{Single-layer transformers and graph reasoning tasks}\label{asec:1L}

\newcommand{\isvertex}{\mathtt{isVertex}}
\newcommand{\isedge}{\mathtt{isEdge}}
\newcommand{\isquery}{\mathtt{isTask}}

This appendix presents the results of \Cref{ssec:easy-tasks}, which separates the collection of graph reasoning tasks into those \textit{retrieval tasks} that can be efficiently solved by single-layer parameter-efficient transformers---including node count, edge count, edge existence, and node degree---and those \textit{parallelizable} or \textit{search tasks} that require deeper constructions---including connectivity, shortest path, cycle check, and triangle count.
Taken together, these results establish that the single-layer transformers of the \depthone{} regime are capable of solving simple aggregation-based tasks, but that their known limitations in capacity as communication protocols of \cite{sht23} apply to non-trivial graph reasoning tasks.

We specific a particular \emph{node/edge encoding} of an input graph $G = (V, E)$ and a graph reasoning task using a consistent encoding scheme that closely resembles the encoding used in our graph reasoning experiments and those of \cite{knmcllh22}.
This encoding is distinguished by the fact that each node and vertex of the graph $G$ is represented by exactly one token, rather than the pair of tokens utilized in our experiments.
This choice ensures that any trivial pre-processing of graph inputs (e.g. using a positional embedding to associate each edge token pair) need not count towards the single-layer transformer model.

\begin{definition}
The \emph{node/edge encoding} of a graph $G = (V, E)$ with $V \subseteq [n]$ and $|V| + |E| \leq N-1$ and a graph reasoning task $P$ is a sequence \[X = X(G, P) = (x_1, \dots, x_N) \in \R^{N \times d}\]
where $d = 5$ and each \[x_i = (x_i^1, x_i^2, \isvertex_i, \isedge_i, \isquery_i) \in \{0, \dots, n\}^2 \times \bit^3 \]
satisfies the following conditions:
\begin{itemize}
    \item For each $v \in V$, there exists exactly one $i \in [N-1]$ with $x_i = (v, 0, 1, 0, 0)$.
    \item For each $(u, v) \in E$, there exists exactly $i \in [N-1]$ with $x_i = (u,v,0, 1, 0)$.
    \item The token $x_N$ encodes a particular instance of the task $P$, by encoding $\isquery_N = 1$ with an optional edge or node encoding. 
    That is, for tasks without arguments (such as triangle count), $x_N = (0,0,0, 0, 1)$. For tasks with a single node argument $v \in [n]$ (such as node degree), $x_N = (v, 0, 1, 0, 1)$. For tasks with a pair of node arguments $u, v \in [n]$ (such as shortest path and connectivity), $x_N = (u, v,0, 1, 1)$.
    \item All other tokens satisfy $x_i = (0, 0, 0, 0, 0)$.
\end{itemize}
We say that a single-layer transformer $f \in \trans{m, H, 1}N$ \emph{solves} task $P$ on graph $G$ if the output corresponding to the task description $f(X(G, P))_N$ encodes the output of the task. Since $f$ is a single-layer transformer, we can write this output as 
\begin{equation}\label{eq:single-output}
    f(X)_N = \psi\paren{\sum_{h=1}^H\sm\paren{\phi(x_N)^\T Q_h K_h^\T \phi(X)^\T} \phi(X) V_h}
\end{equation}
for element-wise multi-layer perceptrons \[\phi: \R^d \to \R^m, \psi: \R^m \to \R\]
broadcasted across each input (i.e. $\phi(X) = (\phi(x_1), \dots, \phi(x_N)) \in \R^{N \times m}$) and weight matrices \[Q_1, \dots, Q_H, K_1, \dots, K_H, V_1, \dots, V_H \in \R^{m \times m}.\] 
\end{definition}

% - Introduction of node/edge encoding: each input is a tuple $x_i = (\isvertex_i, \isedge_i, \isquery_i, x_i^1, x_i^2) \in \bit^3 \times \{0, \dots, n\}^2$. Graph input is sequence of vertices, then edges, then a single query. Query is always final token. Output of query is final output of sequence.
% Corresponds to graphs $G = (V,E)$ with $|V| + |E| \leq N+1$. (Allow tokens to be blank as needed.)

% - Since we're dealing with single-layer transformers, we're only interested in functions of the form
% \[X \mapsto \psi\paren{\sum_{h=1}^H\sm\paren{\phi(X) Q_h^\T K_h \phi(x_N)} \phi(X) V_h}\]
% for element-wise MLPs \[\phi: \R^d \to \R^m, \psi: \R^m \to \R\] and weight matrices \[Q_1, \dots, Q_H, K_1, \dots, K_H, V_1, \dots, V_H \in \R^{m \times m}.\] 

Throughout, we assume that all parameters in the transformer model and intermediate numerical quantities can be written using $O(\log N)$-bit floating point numbers.
This assumption can be satisfied for the positive results of \Cref{asec:pos-1L} and is necessary to obtain the negative results of \Cref{asec:neg-1L}.

We permit the element-wise MLPs $\phi$ and $\psi$ to be arbitrary functions for the negative results, while restricting them to be MLPs that can be approximated using bounded-size multi-layer ReLU networks for the positive results. 
While we do not make these ReLU networks explicit, we restrict ourselves to simple operations that can be computed using linear transformations and the application of smooth univariate functions.

Finally, we acknowledge that the negative results in are largely superseded by those of \cite{ms23}, which establishes that $\L$-complete languages (including graph connectivity and cycle check) cannot be efficiently solved by constant-depth transformers, let alone single-layer transformers.
We include these bounds anyway to mark a contrast with our positive results; draw further connections to the communication complexity lens on transformers; and exhibit simple task instances $(G, P)$ that require greater depth, including constant diameter and constant degree graphs.

\subsection{Positive results for single-layer transformers}\label{asec:pos-1L}

\begin{restatable}[Formal version of \Cref{thm:pos-1L-informal}; \depthone{} computes retrieval tasks]{theorem}{thmposoneL}\label{thm:pos-1L}
Fix any graph reasoning task among node count, edge count, edge existence, and node degree and any graph size $N$.
Then, there exists a single-layer single-headed transformer $f \in \trans{m, 1, 1}{N}$ with embedding dimension $m = O(\log N)$ that solves the task on all graphs $G = (V, E)$ of size $|V| + |E| \leq N-1$ formatted as node/edge embedding sequences.
\end{restatable}

% \thmposoneL*

\begin{proof}

We prove that a single head of self-attention with input and output MLPs can solve these retrieval and aggregation tasks in a parameter-efficient manner by first carefully designing a universal input MLP $\phi: \R^d \to \R^{m}$ for some $m = O(\log N)$ to produce embeddings that encode useful graph properties. 
Then, we define task-specific query, key, and value matrices $Q, K, V: \in \R^{m \times m}$ and output MLPs $\psi: \R^m \to \R$ that produce the correct answers.
While we do not explicitly account for finite-precision computations, all of the operations utilized can be carefully implemented to respect $O(\log N)$-bit floating-point numerical representations using the technical approaches of \cite{sht23, sht24}.

\paragraph*{Shared sinusoidal embedding MLP.}
We denote the embedding output of the input MLP as
\[\phi(x_i) = (\isquery_i, \isvertex_i, \isedge_i, \phi'(x_i)) \in \R^m\]
for some $\phi'(x_i): \R^d \to \R^{2m'}$ for some $m' = O(\log N)$ and $m = 2m' + 3$.
For some fixed $a_1, \dots, a_{m'} \in [0, 1]$ to be determined, let
\[\phi'(x_i) = \begin{cases}
\eta(x_i^1) & \text{if $\isvertex_i = 1$,}\\
\xi(x_i^1, x_i^2) & \text{if $\isedge_i = 1$,}\\
\vec{0} & \text{otherwise,}
\end{cases}\]
where $\eta$ is a sinusoidal embedding MLP for nodes with 
\begin{equation}\label{eq:eta}
    \eta(v) =(\sin(2 \pi a_1 v), \cos(2 \pi a_1 v ), \dots, \sin(2 \pi a_{m'} v ), \cos(2 \pi a_{m'} v )) \in \R^{2m'},
\end{equation}
and $\xi$ is an edge embedding satisfying
\[\xi(u, v) = \frac{m'}{m' + \eta(u)^\T \eta(v)}(\eta(u) + \eta(v)).\]

We first show that the node embeddings are approximately orthogonal. 
By employing standard trigonometric identities, we obtain \[\eta(u)^\T \eta(v) = \sum_{j=1}^{m'} \cos(2\pi a_j(u - v)),\]
and note that $\norm[2]{\eta(v)}^2 = m'$. 
We use a standard concentration argument to prove that $|\eta(u)^\T \eta(v)| \ll m'$ if $u \neq v$ with high probability.

\begin{claim}[Near-orthogonality of embeddings $\eta$]\label{claim:eta-prop}
There exist coefficients $a_1, \dots, a_{m'} \in [0, 1]$ that comprise the embedding $\eta: [n] \to \R^m$ of \Cref{eq:eta} such that
\[\abs{\eta(u)^\T \eta(v)} \leq 2 \sqrt{m' \log n}, \ \text{if $u \neq v$.}\]
\end{claim}
\begin{proof}
We employ the probabilistic method.
Consider $m'$ iid random variables $\ba_1, \dots, \ba_{m'} \sim \unif([0, 1])$ and let $\bfeta$ represent the respective node embedding.
Fix some arbitrary $u, v \in [n]$ with $u \neq v$ and note that 
\[\bfeta(u)^\T \bfeta(v)= \sum_{j=1}^{m'} \cos(2\pi \ba_j(u - v)).\]
For any $j \in [m']$, the integrality of $u-v$ implies that
\[\EE[\ba_j]{\cos(2\pi \ba_j(u - v))} = 0.\]
Hoeffding's inequality provides the following bound:
\begin{align*}
    \pr{\abs{\bfeta(u)^\T\bfeta(v)} \geq 2 \sqrt{m' \log n}}
    &\leq \exp\paren{-\frac{4 m' \log n}{2m'}} \leq \frac{1}{n^2}.
\end{align*}
By applying a union bound to all ${n \choose 2}$ choices of $u$ and $v$, we have the following:
\[\pr{\forall u, v \in [n],\  u \neq v: \ \abs{\bfeta(u)^\T\bfeta(v)} \geq 2 \sqrt{m' \log n}} \leq \frac{n(n-1)}{2n^2} < 1.\]
Hence, there exists a satisfactory set of coefficients $a_1, \dots, a_{m'}$.
\end{proof}

For some $\eta$ satisfying \Cref{claim:eta-prop}, let $\chi = \max_{u, v \in [n], u \neq v} |\eta(u)^\T \eta(v)| \leq 2 \sqrt{m' \log n}$.
By taking $m' = O(\log n)$ be sufficiently large, we guarantee that $\chi \leq \frac{m'}c$ for any constant $c \geq 2$.

We then bound all relevant inner products between vertex and edge embeddings for sufficently large $m'$ (and sufficiently small $\chi$).
We assume throughout that $u, u', v, v'$ are distinct and note that $\xi(u, v) = \xi(v, u)$ (and omit symmetric inequalities).

\begin{align*}
    \norm[2]{\eta(v)}^2 &= m'. \numberthis \label{eq:eta-same}\\
    \abs{\eta(u)^\T \eta(v)} &\leq \chi \leq \frac{m'}{4} . \numberthis \label{eq:eta-diff}\\
    \abs{\eta(u)^\T \xi(u, v) - m'}
    &=\abs{\frac{m'(m' + \eta(u)^\T \eta(v))}{m' + \eta(u)^\T \eta(v)} - m'} = 0. \numberthis \label{eq:eta-xi-same}\\
    \abs{\eta(u)^\T \xi(u', v')} 
    &= \abs{\frac{m' (\eta(u)^\T \eta(u') + \eta(u)^\T \eta(v'))}{m' + \eta(u')^\T \eta(v')}}
    \leq \frac{2m'\chi}{m' - \chi} \leq 4\chi \leq \frac{m'}{4}.
    \numberthis \label{eq:eta-xi-diff} \\
    \abs{\norm[2]{\xi(u, v)}^2 - 2m'}
    &=\abs{\frac{m^{\prime 2}\cdot (2m' + 2 \eta(u)^\T \eta(v)) - 2m'(m' + \eta(u)^\T \eta(v))^2}{(m' + \eta(u)^\T \eta(v))^2} } \\
    &\leq \abs{\frac{2m' \eta(u)^\T \eta(v)}{m' + \eta(u)^\T \eta(v)}} \leq \frac{2 m' \chi}{m' - \chi} \leq 4 \chi \leq \frac{m'}4. \numberthis \label{eq:xi-same}\\
    \abs{\xi(u, v)^\T \xi(u, v') - m'}
    &= \abs{\frac{m^{\prime2}(m' + \eta(u)^\T \eta(v') + \eta(u)^\T \eta(v) + \eta(v)^\T \eta(v'))}{(m' + \eta(u)^\T \eta(v))(m' + \eta(u)^\T \eta(v'))} - m'} \\
    &= \abs{\frac{m^{\prime 2}\eta(v)^\T \eta(v') - m' \eta(u)^\T \eta(v) \eta(u)^\T \eta(v')}{(m' + \eta(u)^\T \eta(v))(m' + \eta(u)^\T \eta(v'))}} \\
    &\leq \frac{m^{\prime 2} \chi + m' \chi^2}{(m' - \chi)^2}
    \leq 4 \chi + 4 \frac{\chi^2}{m'} \leq \frac{m'}{4}. \numberthis \label{eq:xi-diff1}\\
    |\xi(u, v)^\T \xi(u', v')| 
    &=\abs{\frac{m^{\prime 2}(\eta(u)^\T \eta(u') + \eta(u)^\T \eta(v') + \eta(v)^\T \eta(u') + \eta(v)^\T \eta(v')}{(m' + \eta(u)^\T \eta(v))(m' + \eta(u')^\T \eta(v'))}} \\
    &\leq \frac{4m^{\prime 2} \chi}{(m' - \chi)^2} \leq 16 \chi \leq \frac{m'}{4}. \numberthis \label{eq:xi-diff2}
\end{align*}

Therefore, we conclude the following bounds:
\begin{align*}
      \eta(u)^\T \eta(v), \ \eta(u)^\T \xi(u', v'), \ \xi(u, v)^\T \xi(u', v') &\in \left[-\frac{m'}{4}, \frac{m'}{4}\right], \\
      \norm[2]{\eta(v)}^2, \ \eta(u)^\T \xi(u, v), \ \xi(u, v)^\T \xi(u, v') &\in \left[\frac{3m'}{4}, \frac{5m'}{4}\right], \\
      \norm[2]{\xi(u, v)}^2 & \in \left[\frac{7m'}{4}, \frac{9m'}4 \right].
\end{align*}

We now apply the above bounds to prove that each task can be solved by a single-headed transformer with weights $Q, K, V \in \R^{m \times m}$, output MLP $\psi: \R^m \to \R$, and shared input MLP $\phi$ discussed previously.

\paragraph*{Node count task.}
If $P$ encodes the node count task, we specify weights matrices $Q, K, V$ in order to ensure that $Q^\T \phi(x_N) = m' e_1$\footnote{Note that $Q^\T \phi(x_N)$ is the only relevant query embedding to the output $f(X)_N$; see \Cref{eq:single-output}.}; $K^\T\phi(x_i) = (\isvertex_i + \isquery_i) \cdot e_1$; and $V^\T\phi(x_i) = \isquery_i \cdot e_1$\footnote{While we do not specify the entries of $Q, K, V$, they are clearly linear transformations.}.
Then, the following is true of exponentiated key/query inner products and scalings products of value vectors:
\begin{align*}
    \exp(\phi(x_N)^\T Q K^\T \phi(x_i)) 
    &= \begin{cases}e^{m'} & \text{if $\isvertex_i=1$ or $i=N$,} \\ 1 & \text{otherwise.}\end{cases} \\
    \exp(\phi(x_N)^\T Q K^\T \phi(x_i)) V \phi(x_i) &= 
    \begin{cases}
        e^{m'} \cdot e_1 & \text{if $i = N$,} \\
        \vec{0} & \text{otherwise.}
    \end{cases}
\end{align*}

The output of the softmax is then
\begin{align*}
    \sm\paren{\phi(x_N)^\T Q K^\T \phi(X)^\T} \phi(X) V 
    &= \frac{\sum_{i=1}^N \exp(\phi(x_N)^\T Q K^\T \phi(x_i)) V \phi(x_i) }{\sum_{i=1}^N  \exp(\phi(x_N)^\T Q K^\T \phi(x_i)) } \\
    % &= \frac{\sum_{i=1}^{N-1} \exp(m'\cdot \isvertex_i) \cdot \vec{0} + \exp(m') \cdot e_1}{\sum_{i=1}^{N-1} \exp(m'\cdot \isvertex_i) + \exp(m')} \\
    &= \frac{e^{m'}}{e^{m'}\cdot (1 + |V|) + 1 \cdot (N - |V| - 1)} \cdot e_1.
\end{align*}
Let \[y_N := (\sm\paren{\phi(x_N)^\T Q K^\T \phi(X)} \phi(X) V)_1 \in \R\]
denote the first coordinate of the softmax output.
By taking $m' \geq \log(4N)$, we guarantee that
\begin{align*}
    % (\sm\paren{\phi(x_N)^\T Q^\T K \phi(X)} \phi(X) V)_1 
    y_N
    &\in \left[ \frac{1}{1 + |V| + N / e^{m'}}, \frac{1}{1 + |V|}\right]
    \subseteq  \left[ \frac{1}{\frac54 + |V|}, \frac{1}{1 + |V|}\right].
\end{align*}
By letting the output MLP $\psi$ approximate the function $\psi(z) = \floor{\frac{1}{z_1} - 1}$ for $z_1 \in [1, N+2]$, we can ensure that $f(X)_N = |V|$.

\paragraph*{Edge count task.}
We similarly let $Q^\T \phi(x_N) = m' e_1$; $K^\T\phi(x_i) = (\isedge_i + \isquery_i) \cdot e_1$; and $V^\T\phi(x_i) = \isquery_i \cdot e_1$.
By an identical analysis of the attention matrix and $\psi$ construction, we ensure that $f(X)_N = |E|$.

\paragraph*{Edge existence task.}
Under the node/edge encoding, we assume that the edge existence task is encoded as $x_N = (x_N^1, x_N^2, 0, 1, 1)$ for some $x_N^1, x_N^2\in V$ and should return $f(X)_N = \indicator{(x_N^1, x_N^2) \in E}$.
We choose our weight matrices to ensure that $Q^\T\phi(x_N) = \phi'(x_N) = \xi(x_N^1, x_N^2)$; $K^\T\phi(x_i) = \phi'(x_i)$, and $V^\T\phi(x_i) = 2(1 - \isquery_i)e_1$. 
By applying \Cref{claim:eta-prop} and letting $m' = O(\log N)$ to be sufficiently large, the following is true the query/key inner products:
\begin{align*}
    \exp(\phi(x_N)^\T Q K^\T \phi(x_i)) 
    &= \exp(\norm[2]{\xi(x_N^1, x_N^1)}^2)
    \geq e^{7m'/4} & \text{if $\set{x_i^1, x_i^2} = \set{x_N^1, x_N^2}$,}\\
    \exp(\phi(x_N)^\T Q K^\T \phi(x_i)) 
    &\leq e^{5m'/4} \leq \frac{1}{8N} e^{7m'/4} & \text{otherwise.}
\end{align*}
We can therefore analyze the softmax output $y_N$ to obtain the following bound:
\begin{align*}
    %  (\sm\paren{\phi(x_N)^\T Q^\T K \phi(X)} \phi(X) V)_1 
    y_N
    %  &\leq \frac{2e^{2m'}\indicator{(x_N^1, x_N^2) \in E} + N \cdot \frac1{2N}e^{2m'}}{e^{2m'}(1 + \indicator{(x_N^1, x_N^2) \in E})} \\
    &\leq \frac{2\exp(\norm[2]{\xi(x_N^1, x_N^1)}^2)\indicator{(x_N^1, x_N^2) \in E} + 2N \cdot \frac1{8N}e^{7m'/4}}{\exp(\norm[2]{\xi(x_N^1, x_N^1)}^2)(1 + \indicator{(x_N^1, x_N^2) \in E})} \\
    &\leq \frac{2 \indicator{(x_N^1, x_N^2) \in E} + \frac14}{1 + \indicator{(x_N^1, x_N^2) \in E}}
     \leq \indicator{(x_N^1, x_N^2) \in E} + \frac14, \\
    %  (\sm\paren{\phi(x_N)^\T Q^\T K \phi(X)} \phi(X) V)_1 
    y_N
    %  &\geq \frac{2e^{2m'}\indicator{(x_N^1, x_N^2) \in E}}{e^{2m'}(1 + \indicator{(x_N^1, x_N^2) \in E}) + N \cdot \frac1{4N} e^{2m'}} \\
    &\geq \frac{2\exp(\norm[2]{\xi(x_N^1, x_N^1)}^2)\indicator{(x_N^1, x_N^2) \in E}}{\exp(\norm[2]{\xi(x_N^1, x_N^1)}^2)(1 + \indicator{(x_N^1, x_N^2) \in E}) + N \cdot \frac1{8N} e^{7m' / 4}} \\
    &\geq \frac{2 \indicator{(x_N^1, x_N^2) \in E}}{1 + \indicator{(x_N^1, x_N^2) \in E} + \frac18}
     \geq \indicator{(x_N^1, x_N^2) \in E} - \frac14.
\end{align*}
Hence, 
\[y_N \in \begin{cases}
\left[-\frac14, \frac14\right] & \text{if $(x^1_N, x^2_N) \not\in E$,} \\
\left[\frac34, \frac54\right] & \text{if $(x^1_N, x^2_N) \in E$.} 
\end{cases}\]
Therefore, it suffices to design a threshold output MLP $\psi$ that satisfies 
\[\psi(z) = \begin{cases}
1 & z_1 \geq \frac34, \\
0 & z_1 \leq \frac14,
\end{cases}\]
in order to distinguish between the two possible output ranges.
This can be easily constructed by taking a linear combination of two ReLU neurons.

\paragraph*{Node degree task.}
We assume that the task is encoded as $x_N = (x_N^1, 0, 1, 0, 1)$ and should return $f(X)_N = \deg(x_N^1) := \abs{\set{(x_N^1, v) \in E}}$.
We use weight matrices with $Q^\T\phi(x_N) = \phi'(x_N) = \eta(x_N)$; $K^\T\phi(x_i) = \phi'(x_i)$; and $V^\T\phi(x_i) = \isquery_i$.
This ensure that the following is true about the query/key inner products for sufficiently large $m'$:
\begin{align*}
    \exp(\phi(x_N)^\T Q K^\T \phi(x_i)) &= e^{m'} & \text{if $i = N$ or $x_N^1 \in \set{x_i^1, x_i^2}$,}\\
    \exp(\phi(x_N)^\T Q K^\T \phi(x_i)) &\leq e^{ m'/4} \leq \frac{1}{4N} e^{m'} & \text{otherwise.}
\end{align*}
We similarly ensure that the following holds about the softmax output $y_N$.
\begin{align*}
    %  (\sm\paren{\phi(x_N)^\T Q^\T K \phi(X)} \phi(X) V)_1 
    y_N
     &\leq \frac{e^{m'}}{e^{m'} \cdot (\deg(x_N^1) + 2)}
     = \frac{1}{\deg(x_N^1) + 2}, \\
    %  (\sm\paren{\phi(x_N)^\T Q^\T K \phi(X)} \phi(X) V)_1 
    y_N
     &\geq \frac{e^{m'}}{(\deg(x_N^1) + 2)e^{m'} + N \cdot \frac1{4N} e^{m'}} \\
     &\geq \frac{1}{\deg(x_N^1) + \frac94}.
\end{align*}
We then use the similar approach employed in the node count task to choose an output MLP that approximately computes $\phi(z) = \floor{\frac{1}{z_1} - 2}$.
% All transformers use the same input embeddings $\phi$. Since they're single-headed, just need to identify $Q, K, V$ and $\psi$ for each task.
% \begin{enumerate}
%     % \item \textbf{Node count:} $Q x_N = e_1$. $Kx_i = e_1 \isvertex_i$. $Vx_i = \isquery_i = 1$.
%     % Then, $N$th outcome of self-attention is $\frac1{|V| + 1}$. MLP $\psi$ inverts this.
    
%     % \item \textbf{Edge count:} $Q x_N = e_1 $. $Kx_i = e_1 \isedge_i$ (note that $\isedge_N = 1$). $Vx_i = \isquery_i$.
%     % Then, $N$th outcome of self-attention is $\frac1{|E| + 1}$. MLP $\psi$ inverts this.
    
%     % \item \textbf{Edge existence:} $Q x_N = \xi(x_N^1, x_N^2)$. $K x_i = \xi(x_i^1, x_i^2)$. $V x_i = 2(\isedge_i - \isquery_i)$.
%     % Then, $N$th outcome is 1 if there exists some other edge with same route, 0 otherwise.
    
%     \item \textbf{Node degree:} $Qx_N = \xi(x_N^1)$. $K x_i = embed$. $V x_i = \isvertex_i$.
%     Then, $N$th outcome is $\frac{2}{2 + degree}$. MLP inverts.
% \end{enumerate}

We thus conclude that there exist single-layer transformers that solve all of the described tasks.
\end{proof}

\subsection{Negative results for single-layer transformers}\label{asec:neg-1L}

\begin{restatable}
[Formal version of \Cref{thm:neg-1L-informal}; \depthone{} cannot compute search or retrieval tasks]
{theorem}{thmnegoneL}\label{thm:neg-1L}
% [Regime (A) does not solve harder tasks]
Fix any graph reasoning task among graph connectivity, shortest path, and cycle detection.
Any single-layer transformer $f \in \trans{m, H, 1}{N}$ with $O(\log N)$-bit precision that solves the task on all graphs $G = (V, E)$ of size $|V| + |E| \leq N-1$ formatted as node/edge embedding sequences has width satisfying $mH = \Omega(N / \log N)$.
\end{restatable}

% \thmnegoneL*

Our negative results generalizes and applies the approach of \cite{sht23} to prove negative results for single-layer transformers by communication complexity reductions.
The bounds hinge on the following fundamental fact about the hardness of a two-player game where two agents jointly attempt to compute a set disjointness quantity that depends on both of their inputs with bounded communication.
\begin{fact}[Disjointness communication complexity lower bound \cite{yao79}]\label{fact:disj}
Suppose Alice and Bob are given inputs $a, b \in \bit^r$ and wish to jointly compute $\disj(a, b) = \max_i a_i b_i$ by alternately sending single-bit messages to one another.
Any deterministic protocol that computes $\disj(a, b)$ requires at least $r$ rounds of communication, or $r$ bits of information.
\end{fact}

We first generalize Theorem~7 of \cite{sht23} to show that no transformer can efficiently solve an embedded set disjointness problem without having a width that scales linearly in $r$. 
We prove the theorem and later apply to it prove negative results about relevant graph reasoning tasks.
To demonstrate that the graph reasoning tasks do not require pathological input graphs $G$ to be hard, we exhibit particular input graph instances with constant graph diameter or constant degree where the task cannot be efficiently solved.
% - Compare to \cite{ms23}: already well known that $\L$-complete languages cannot be solved using $O(1)$ depth, so this isn't entirely knew. But these results are nice complements since they show that these limitations apply to very simple graphs (constant-degree or constant-diameter). Also, nice comparison to positive results: clear separation in difficulty among transformer tasks.

% - Idea: encode disjointness communication complexity problem, similar to \cite{sht23}.

\begin{lemma}[Generic \depthone{} communication complexity negative result]%[Generalization of Theorem~7 of \cite{sht23}]
\label{thm:disj-1L}
For some sequence length, fix two disjoint subsets $A, B \subset [N-1]$, and consider a single-layer transformer $f \in \trans{m, H, 1}{N}$ with $O(\log N)$-bit precision that solves set disjointness, i.e. $f(X)_N = \disj(a, b)$ for any input $X$ where $X_A$ is a function of Alice's input $a \in \bit^r$, $X_B$ is a function of Bob's input $b \in \bit^r$, and $X_{[N] \setminus (A \cup B)}$ is fixed regardless of $a$ and $b$.
% Consider any encoding of Alice and Bob's size-$n$ disjointness input $a, b \in \bit^n$ into a length-$N$ transformer input $X \in \R^N$ with disjoint indices $A, B \subset [N-1]$ such that $X_A$ and $X_B$ are encoded by Alice and Bob respectively and $X_{[N] \setminus (A \cup B)}$ is constant.
Then, $f$ has width satisfying $mH = \Omega(r / \log N)$.
\end{lemma}
\begin{proof}
Consider any transformer $f$ that solves $\disj$ as specified in the theorem statement.
We show that this implies the existence of a $O(mH \log N)$-round communication protocol that solves $\disj(a, b)$ for any inputs $a, b \in \bit^{r}$.
An application of \Cref{fact:disj} immediately proves the theorem statement.

If such a transformer $f$ exists, then the following is true for some $\phi, \psi$, and $Q, K, V$:
\begin{align*}
    \disj(a, b) 
    &= \psi\paren{\sum_{h=1}^H\sm\paren{\phi(x_N)^\T Q_h K_h^\T \phi(X)} \phi(X) V_h} \\
    &=\psi\paren{\sum_{h=1}^H \frac{Z_{h,A}\exp(L_{h,A}) + Z_{h,B}\exp(L_{h,B}) + Z_{h,[N] \setminus (A \cup B)} \exp(L_{h,[N] \setminus (A \cup B)})}{\exp(L_{h,A}) + \exp(L_{h,B}) + \exp(L_{h,[N] \setminus (A \cup B)})}},
\end{align*}
for partial softmax and normalization terms\footnote{
These terms are designed in order to facilitate computation with $\log(N)$-bit floating-point numbers. Note that each $Z_{h,S}$ is a convex combination of value vectors $V_h^\T \phi(x_i)$, which means that $Z_{h,S}$ requires no more bits of precision to accurately approximate than $V_h^\T \phi(x_i)$. 
Similarly, each $L_{h, S}$ is within an $O(\log N)$ additive factor of $\max_{i \in S} \phi(x_N)^\T Q_h K_h^\T \phi(x_i)$.
Finally, the output of the softmax term for head $h$ is a convex combination of $Z_{h, A}$, $Z_{h, B}$, and $Z_{h, [N] \setminus (A \cup B)}$, which makes accurate computation of each output possible without requiring greater precision.
} defined as follows for $S \subset [N]$ and $h \in [H]$:
\begin{align*}
    Z_{h, S} &= \frac{\sum_{i \in S} \exp(\phi(x_N)^\T Q_h K_h^\T \phi(x_i)) V_h \phi(x_i)}{\sum_{i \in S}  \exp(\phi(x_N)^\T Q_h K_h^\T \phi(x_i))} \in \R^m,\\
    L_{h, S} &= \log\paren{\sum_{i \in S}  \exp(\phi(x_N)^\T Q_h K_h^\T \phi(x_i))} \in \R_+.
\end{align*}
The definition of input instances $X$ implies that Alice can compute $Z_{h, A}$ and $L_{h, A}$ as a function of her input $a$; Bob can similarly compute $Z_{h, B}$ and $L_{h, B}$ from $b$; and $Z_{h, [N] \setminus (A \cup B)}$, $L_{h, [N] \setminus (A \cup B)}$, and all implementation details of $f$ are known by all players.
Therefore, Bob can compute $\disj(a, b)$ by an $O(mH \log N)$-round protocol where Alice sends him $(Z_{h, A}, L_{h,A})$ bit-by-bit.
\end{proof}

It remains to apply \Cref{thm:disj-1L} to each graph reasoning task by defining graph instance encodings $X$ that encode $\disj(a, b)$ in the solution to the task for some $r = \Theta(N)$.

\begin{proof}[Proof of \Cref{thm:neg-1L}]

% The following corollaries jointly prove \Cref{thm:neg-1L}.

% \begin{corollary}\label{cor:conn-1L}
% Any single-layer transformer $f \in \trans{m, H, 1}{N}$ that solves the graph connectivity task on graphs encoded as node/edge embedding sequences has width satisfying $mH = \Omega(N / \log N)$, even if the input graph has constant degree or constant diameter.
% \end{corollary}

% \begin{corollary}\label{cor:shortest-1L}
% Any single-layer transformer $f \in \trans{m, H, 1}{N}$ that solves the shortest path task on graphs encoded as node/edge embedding sequences has width satisfying $mH = \Omega(N / \log N)$, even if the input graph has constant degree or constant diameter.
% \end{corollary}

% \begin{corollary}\label{cor:cycle-1L}
% Any single-layer transformer $f \in \trans{m, H, 1}{N}$ that solves the cycle check task on graphs encoded as node/edge embedding sequences has width satisfying $mH = \Omega(N / \log N)$, even if the input graph has constant degree or constant diameter.
% \end{corollary}

% \begin{corollary}\label{cor:cycle-1L}
% Any single-layer transformer $f \in \trans{m, H, 1}{N}$ that solves the triangle counting task on graphs encoded as node/edge embedding sequences has width satisfying $mH = \Omega(N / \log N)$, even if the input graph has constant degree and constant diameter.
% \end{corollary}

For each task, we provide a pair of ``hard instances:'' one with constant degree and one with constant diameter, in order to obey different notions of graph simplicity.

\paragraph*{Graph connectivity task.}

For both instances, on any disjointness input $a, b \in \bit^r$, we define a graph $G = (V, E)$ with $|V| + |E| = O(r)$ whose edges encode the input, i.e. $E = E(a, b)$.
We define three fixed disjoint sets of potential edges $\bar{E}_A, \bar{E}_B, \bar{E}_*$ with $|\bar{E}_A| +  |\bar{E}_B| + |\bar{E}_*| = O(r)$ such that $E = E_A(a) \cup E_B(b) \cup \bar{E}_*$, where $E_A(a) \subset \bar{E}_A$ is a function of only Alice's input $a$ and $E_B(b) \subset \bar{E}_B$ is a function of Bob's $b$.

We define an enumeration of all vertices in $V$ and potential edges in $\bar{E}_A, \bar{E}_B, \bar{E}_*$. 
We first fix their sizes $|V|, |\bar{E}_A|, |\bar{E}_B|, |\bar{E}_*|$ and then index the vertices and edges in the order $V, \bar{E}_A, \bar{E}_B, \bar{E}_*$.
That is,
% For simplicity, we enumerate the fixed vertices as $V = [|V|]$ and denote the edges as 
\begin{align*}
    V &= \set{1, \dots, |V|}, \\
    \bar{E}_* &= \set{(u_i, v_i): i \in S_*}, \ \text{for} \ S_* = \set{|V| + 1, \dots, |V| + |\bar{E}_*|},\\
    \bar{E}_A &= \set{(u_i, v_i): i \in A}, \ \text{for} \ A = \set{|V|+|\bar{E}_*| + 1, \dots, |V| + |\bar{E}_*| + |\bar{E}_A|}, \\
    \bar{E}_B &= \set{(u_i, v_i): i \in B}, \ \text{for} \ B = \set{|V|+|\bar{E}_*| + |\bar{E}_A| + 1, \dots, |V| + |\bar{E}_*| + |\bar{E}_A| + |\bar{E}_B|}.
    % \set{(u_{|V| +1}, v_{|V|+1}) \dots, (u_{|V| + |\bar{E}|}, v_{|V| + |\bar{E}|})}\subset V^2 \\
    % \bar{E}_A &= \set{(u_{|V| + |\bar{E}| +1}, v_{|V| + |\bar{E}| +1}) \dots, (u_{|V| + |\bar{E}| + |\bar{E}_A|}, v_{|V| + |\bar{E}|+ |\bar{E}_A|})}\subset V^2 \\
\end{align*}
% for 
% \begin{align*}
%     S_* &= \set{|V| + 1, \dots, |V| + |\bar{E}_*|} \\
%     A &= \set{|V|+|\bar{E}_*| + 1, \dots, |V| + |\bar{E}_*| + |\bar{E}_A|},\\
%     B &= \set{|V|+|\bar{E}_*| + |\bar{E}_A| + 1, \dots, |V| + |\bar{E}_*| + |\bar{E}_A| + |\bar{E}_B|}.
% \end{align*}
This implies that the following transformer input $X \in R^{N \times d}$ for $N = |V| + |S_*| + |A| + |B| + 1$ is a valid node/edge encoding such that inputs $X_A$ is a function of $a$, $X_B$ is a function of $b$, and $X_{[N] \setminus (A \cup B)}$ is constant:
\begin{itemize}
    \item For $i \in V$, $x_i = (i, 0, 1, 0, 0)$ encodes a fixed vertex.
    \item For $i \in S_*$, $x_i = (u_i, v_i, 0, 1, 0)$ encodes a fixed edge.
    \item For $i \in A$, $x_i$ represents edge $(u_i, v_i)$ if it exists:
    \[x_i = \begin{cases}(u_i, v_i, 0, 1, 0) & \text{if $(u_i, v_i) \in E_A(a)$,} \\ (0, 0, 0,0,0) & \text{otherwise.}\end{cases}\]
    \item For $i \in B$, $x_i$ represents edge $(u_i, v_i)$ if it exists:
    \[x_i = \begin{cases}(u_i, v_i, 0, 1, 0) & \text{if $(u_i, v_i) \in E_B(b)$,} \\ (0, 0, 0,0,0) & \text{otherwise.}\end{cases}\]
    \item $x_N = (u, v,0, 1, 1)$ encodes the task token for some fixed $(u, v) \in V^2$.
\end{itemize}

\textit{Constant-diameter instance:}
We define a graph $G$ of diameter at most 8 that encodes an disjointness instance $a, b \in \bit^r$ in an instance of connectivity.
\begin{itemize}
    \item Let $|V| = 3r + 2$, and let nodes 1 and $3r+2$ denote the ``source'' and ``sink'' node respectively.
    That is, $x_N = (1, 3r+2, 0, 1, 1)$.
    \item Edges between the source node 1 and nodes $\set{2, \dots, r+1}$ and between the end node $3r + 2$ and $\set{2r + 2, \dots, 3r+1}$ always are included. That is, \[E_* = \set{(1, i): i \in \set{2, \dots, r+1}} \cup \set{(3r+2, i): i \in \set{2r+2, \dots, 3r+1}}.\]
    \item Alice's input is encoded as edges $(i + 1, i + r + 1)$ for $i \in [r]$. That is,
    \[E_A(a) = \set{(i + 1, i + r + 1) : a_i = 1} \subset \bar{E}_A = \set{(i + 1, i + r + 1): i \in [r]} .\]
    \item Bob's inputs are similarly encoded as $(i + r+ 1, i + 2r)$:
    \[E_B(b) = \set{(i + r + 1, i + 2r + 1) : b_i = 1} \subset \bar{E}_B = \set{(i + r+ 1, i + 2r + 1): i \in [r]} .\]
\end{itemize}
We visualize this construction of $G$ in \Cref{fig:constant-diameter}.

There exists a path between node 1 and node $3r + 2$ if and only if there exists some $i\in[r]$ such that $(i + 1, i + r+1), (i + r + 1, i + 2r + 1) \in E$, which corresponds to $a_i = b_i = 1$.
Thus, the graph $G$ is connected if and only if $\disj(a, b) = 1$.
Any transformer $f$ that computes \[f(X)_N = \indicator{1 \text{ and } 3r+2 \text{ are connected in} G}\]
also solves disjointness. 
Since $N = |V| + |S_*| + |A| + |B| + 1 = \Theta(r)$, we conclude the proof of hardness of solving graph connectivity on this instance by applying \Cref{thm:disj-1L}.

% Graph DISJ(A,B)=1 A = (1,0,1), B= (1,1,0), r=3
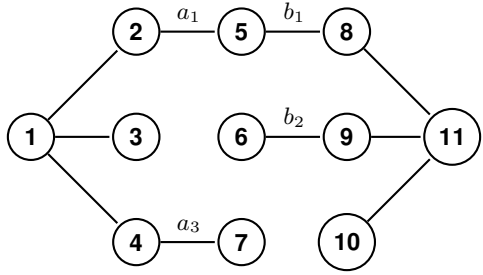
\begin{figure}[h!]
    \centering
    \begin{tikzpicture}[scale=0.7,-,>=stealth',shorten >=1pt,auto,node distance=3cm,
                    thick,main node/.style={circle,draw,font=\sffamily\small\bfseries}]

  \node[main node] (1) at (0,2) {1};
  \node[main node] (2) at (2,4) {2};
  \node[main node] (3) at (2,2) {3};
  \node[main node] (4) at (2,0){4};
  \node[main node] (5) at (4,4){5};
  \node[main node] (6) at (4,2) {6};
  \node[main node] (7) at (4,0){7};
  \node[main node] (8) at (6,4){8};
  \node[main node] (9) at (6,2){9};
  \node[main node] (10) at (6,0){10};
  \node[main node] (11) at (8,2){11};
  \path[every node/.style={font=\sffamily\small}]
  (1) edge node[above] {} (2)
  (1) edge node [left] {} (3)
  (1) edge node [left] {} (4)
  (2) edge node [above] {$a_1$ } (5)
  (4) edge node [above] {$a_3$ } (7)
  (5) edge node [above] {$b_1$ } (8)
  (6) edge node [above] {$b_2$ } (9)
  (8) edge node [left] { } (11)
  (9) edge node [left] { } (11)
  (10) edge node [left] { } (11);
\end{tikzpicture}
    \caption{The constant diameter graph construction for $r=3$, $A=(1,0,1)$ and $B=(1,1,0)$. The source $1$ and sink $11$ are connected which is equivalent to $\disj(A,B)=1$ by construction. }
    \label{fig:constant-diameter}
\end{figure}

\textit{Constant-degree instance:}
We can define a degree-3 graph $G$ that similarly encodes a disjointness instance $a, b$.
To do so, we modify the previous construction by replacing the fixed edges $E_*$ with two binary trees, each of depth $O(\log r)$, between the source and end nodes as roots and Alice and Bob's nodes incident to $\bar{E}_A$ and $\bar{E}_B$ respectively as leaves.
The remainder of the construction---including the encoding of $E_A(a)$ and $E_B(b)$ and the connectivity analysis---is identical.
Since a binary tree with $r$ leaves has $O(r)$ nodes and edges, the graph $G$ can be similarly encoded as a length-$N$ transformer input for $N = O(r)$.
See \Cref{fig:constant-degree} for a visualization.

% A =(0,1,0,1), B= (1,1,0,0) DISJ(A,B)=1 r = 4, log(r)=2
\begin{figure}[h!]
    \centering
\begin{tikzpicture}[scale=0.5,-,>=stealth',shorten >=1pt,auto,node distance=3cm,
                    thick,main node/.style={circle,draw,font=\sffamily\small\bfseries}]               
\node[main node] (1) at (0,4) {$S$};
\node[main node] (2) at (2,2) {};
\node[main node] (3) at (2,6) {};
\node[main node] (4) at (4,1) {};
\node[main node] (5) at (4,3) {};
\node[main node] (6) at (4,5) {};
\node[main node] (7) at (4,7) {};
\node[main node] (8) at (12,4) {$X$};
\node[main node] (9) at (10,2) {};
\node[main node] (10) at (10,6) {};
\node[main node] (11) at (8,1) {};
\node[main node] (12) at (8,3) {};
\node[main node] (13) at (8,5) {};
\node[main node] (14) at (8,7) {};
\node[main node] (15) at (6,1) {};
\node[main node] (16) at (6,3) {};
\node[main node] (17) at (6,5) {};
\node[main node] (18) at (6,7) {};
  \path[every node/.style={font=\sffamily\small}]
  (1) edge node[above]{}(2)
  (1) edge node[above]{}(3)
  (2) edge node[above]{}(4)
  (2) edge node[above]{}(5)
  (3) edge node[above]{}(6)
  (3) edge node[above]{}(7)
  (4) edge node[above]{$a_4$}(15)
  (6) edge node[above]{$a_2$}(17)
  (18) edge node[above]{$b_1$}(14)
  (17) edge node[above]{$b_2$}(13)
  (8) edge node[above]{}(9)
  (8) edge node[above]{}(10)
  (9) edge node[above]{}(11)
  (9) edge node[above]{}(12)
  (10) edge node[above]{}(13)
  (10) edge node[above]{}(14) ;
\end{tikzpicture}    
    \caption{The degree 3 graph construction for $r=4$, $A=(0,1,0,1)$, $B=(1,1,0,0)$. The source node $S$ and sink node $X$ are connected, which is equivalent to $\disj(A,B)=1$ by construction.}
    \label{fig:constant-degree}
\end{figure}
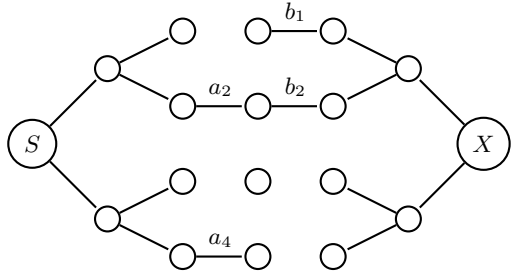

\paragraph*{Shortest path task.}

% \begin{proof}[Proof of \Cref{cor:shortest-1L}]
We can modify the graph connectivity constructions $G$ to create a decision variant of shortest path, rather than graph connectivity.
Let $D(G)$ be the length of a path between the source and sink node in each connectivity construction, and note that $D(G) = 4$ for the constant-diameter instance and $D(G) = O(\log r)$ for constant-degree.
We design a new graph $G'$ of size $O(r)$ then appends a path of $D(G) + 1$ vertices to $G$.
Then, the shortest path between the source and sink node is $D(G)$ if they are connected and $D(G) + 1$ if not.

If there exists a transformer $f'$ solves shortest path on graphs of the form $G'$, then there exists a minor modification $f$ that solves graph connectivity on graph $G$.
% \end{proof}

% \begin{proof}[Proof of \Cref{cor:cycle-1L}]

\paragraph*{Cycle check task.}

We again modify the graph connectivity constructions.
We produce $G'$ by adding a single edge to $G$ between the source node 1 and the sink node $3r+2$, which ensures that $G'$ has a cycle if and only if $G$ is connected between the source and sink nodes.
Therefore, a transformer $f'$ can only solve check check if there exists a transformer $f$ that solves connectivity.
\end{proof}

%% file: gnn_limitations.tex
\section{Representational gaps between transformers and GNNs}\label{asec:gnn_limitations}

Our experiments in \Cref{sec:transformers-gnns} highlight the differences in graph reasoning abilities between vanilla transformers with a naive input graph tokenization and graph neural networks (GNNs).
These distinctions can be understood as the consequences of model capacity or inductive bias between transformers and GNNs.
This appendix contrasts the novel analysis of the capacities of vanilla transformers in \Cref{sec:hierarchy} with previously established limitations on GNNs.
We present two theoretical tools for deriving negative results on GNN capabilities: \congest{} distributed computing model and the Weisfeiler-Leman isomorphism test (WL-test).
We discuss the results implied by these frameworks and contrast them with transformer capabilities and limitations.

% We use a variety of results to highlight the representational gaps between transformers and GNNs on problems with global structure.

% \subsection{Additional experiment}

\subsection{Limitations of GNNs via \congest{} analogy}

The bidirectional relationship between transformers and the massively parallel computation (MPC) distributed computing model of \cite{sht24} was partially inspired by a similar analogy between GNNs and the \congest{} model by \cite{loukas2020graph}. 

The MPC and \congest{} distributed computing protocols differ in their models of how messages are passed between machines.
While MPC protocols permit any machine to send a message to any other machine subject to capacity constraints, \congest{} protocols operate on a graph topology that restricts machines to send messages exclusively to and from neighbors.
As a consequence, two nodes in a \congest{} graph that are separated by a path of length $p$ cannot share information with one another without $\Omega(p)$ rounds of communication; no similar notion of ``long-range communication'' exists for MPC.

All message-passing GNNs where each node initially knows only its own features can be simulated by a \congest{} protocol whose rounds scales linearly in the depth of the GNN~\cite{loukas2020graph}.
This reduction implies a lower bound on solving several parallelizable and search tasks discussed in \Cref{ssec:appendix-graphreasoning}.

% there is no notion of ``long-range communication'' in MPC since all machines can send one another messages

% : MPC protocols permit each machine to send a bounded quantity of messages to any other collection of machines, while \congest{} protocols impose a graph topology on machines and permit only messages to neighbors.

% Just as \cite{sht24} demonstrates a bidirectional connection between transformers models and the massively parallel computation (MPC) distributed computing model, \cite{loukas2020graph} shows that GNNs can be simulated by the \congest{} model.
% MPC and \congest{} differ as distributed computing protocols in how they model the ability of machines to send messages to one another: MPC protocols permit each machine to send a bounded quantity of messages to any other collection of machines, while \congest{} protocols impose a graph topology on machines and permit only messages to neighbors.
% This ``local versus global'' trade-off between \congest{} and MPC can be made explicit for several tasks studied here.

\begin{theorem}\label{cor:loukas}
Any graph neural network $f$ with width (message size) $m$ and depth $L$ that computes any of the following tasks on $n$-node graphs has the following size lower bounds:
\begin{itemize}[nosep]
    \item Subgraph connectivity, minimum spanning forest, minumum cut, shortest path: $L \sqrt{m} = \tilde\Omega(\sqrt{n})$.
    \item Cycle detection, diameter computation: $Lm = \tilde\Omega(n)$.
\end{itemize}
% requires $mL = \Omega(\sqrt{n} / \log n)$.
If $L = O(\log N)$, then for all tasks, $m = \tilde\Omega(n)$.
\end{theorem}

In contrast, the quantitative bounds in \Cref{sec:hierarchy} establish sharp trade-offs between transformers and GNNs for parallelizable tasks and suggest a possible equivalence for search tasks.

All parallelizable tasks---including (subgraph) connectivity, minimum spanning forest, minimum cut, and cycle detection---can be solved by transformers of depth $L = O(\log n)$ and width $m = O(n^{\epsilon})$ for any constant $\epsilon\in (0, 1)$ due to \Cref{thm:parallelizable-pos}.
In contrast, \Cref{cor:loukas} requires that a similar-depth GNNs have width $m = \tilde\Omega(n)$, and a GNN of comparable width requires depth $L= \tilde\Omega(n^{1-\epsilon})$.

On the other hand, search tasks, such as shortest path and diamter, are only guaranteed to be solvable by transformers of depth $O(\log n)$ and width $O(n^{1 + \epsilon})$ (for graphs with $|E|=n^2$) by \Cref{thm:search-pos}.
This statement compares to the GNN negative results of \Cref{cor:loukas}.

% This can be contrasted with the bounds in \Cref{sec:hierarchy}, which show that there exist transformer with depth $L= O(\log n)$ and width $m = O(N^\epsilon)$ for any $\epsilon > 0$ that solve (subgraph) graph connectivity, cycle detection, and minimum spanning forest. 
% Any GNNs with depth $L = O(\log N)$ require width $m = \Omega(\sqrt{n} / \log^2 n)$.

\subsection{Limitations of GNNs via the Weisfeiler-Leman test}

A relationship between GNNs and the Weisfeiler-Leman heuristic graph isomorphism test~\cite{weisfeiler1968reduction} (WL test) further establishes representational limitations of message-passing GNNs. 
This connection and the rich literature surrounding it is presented in greater detail by \cite{morris2023weisfeiler}.

The 1-WL test is a permutation-invariant test for predicting whether two candidate graphs are isomorphic that works by first labeling each node with the empty set $\emptyset$ and then repeatedly replacing each label with a multiset of its neighbors' labels.
A hierarchy of WL test variants exists where the $k$-WL test maintains a label for each $k$-tuple of vertices.
The failure models of these heuristic solutions are well-understood; critically, the 1-WL test cannot distinguish between connected and disconnected graphs.

The abilities of message-passing GNNs without unique node identifiers to determine whether two graphs are isomorphic are limited by the graphs distinguishable by the 1-WL test~\cite{xhlj18,morris2023weisfeiler}.
As an immediate consequence, such GNNs cannot solve graph connectivity \textit{at all}, unlike transformers, which can do so with a logarithmic-depth parameter-efficient representation.
The relationship is bidirectional; message-passing GNNs (and transformers as well) admit an efficient approximation of the 1-WL test.

The effectiveness of these bounds is tempered by the assumption that no node includes identifying features, which is easily overcome by standard GNN models.
The analogy is further limited by the fact that graph embeddings of transformers must have some (possibly arbitrary) node identifier as input in order to tokenize a graph using the node/edge encoding without losing the ability to associate each node with its incident edges. 
However, the juxtaposition of the 1-WL and \congest{}-based limitations on the abilities of GNNs to solve connectivity-based tasks suggests a fundamental gap in capabilities between models that is apparent in multiple theoretical lenses.

%% file: empirical_appendix.tex
\section{Experimental Details}\label{asec:empirical}

\subsection{Datasets}\label{asec:dataset}

% We conduct our experiments on the graph reasoning tasks proposed in GraphQA~\citep{fhp}. This dataset presents multiple graph reasoning problems with different difficulty levels. These tasks can be categorized as follows.

% \textbf{Datasets.}
We evaluate our model on the diverse graph reasoning tasks presented in GraphQA~\cite{fhp23}. We used the public code of the dataset available at \url{https://github.com/google-research/google-research/tree/master/graphqa}. The code to generate the datasets is licensed under the Apache License, Version 2.0. The tasks in the dataset range in difficulty and encompass the following categories:

\begin{itemize}[leftmargin=*,nosep]
    \item \textbf{Graph-level:}
    node counting~(counting the number of nodes in a graph), edge counting~(counting the number of edges in a graph),
    cycle check~(determining whether a graph contains a cycle), and
    triangle counting~(counting the number of triangles in a graph).
    \item \textbf{Node-level:}
    node degree~(calculating the degree of a given node in a graph).
    % , and connected nodes~(finding all the nodes that are connected to a given node in a graph),
    \item \textbf{Edge-level:}
    connectivity~(finding if there is a path from one node to another),
    edge existence~(whether a given edge exists in a graph, and
    shortest path~(finding the length of the shortest path from one node to another).
\end{itemize}

The graphs used in the experiments in this paper and the corresponding graph reasoning tasks are taken from \cite{fhp23}. 
There are $1,000$ graphs in the original train set, $500$ graphs in the dev set, and $500$ graphs in the test set. The graphs are generated randomly using Erd\H{o}s-R\'{e}nyi~(ER) random graph model~\citep{erdds1959random}. Graph size ranges from 5 to 20 nodes.

\textbf{Train set statistics.} Average number of nodes: 11.90; average number of edges: 37.01; average node degree: 5.43.

\textbf{Test set statistics.} Average number of nodes: 12.37; average number of edges: 39.79; average node degree: 5.70.

\begin{wrapfigure}[18]{R}{0.5\textwidth}
\input{figures/chart_cycle_length}
\vspace*{-2mm}
\caption{Histogram of minimum cycle lengths for cycle check instances.}\label{fig:appendix:cycle_length}
\end{wrapfigure}
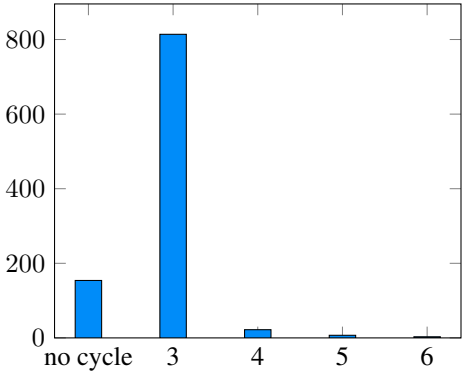

While random instances of graph reasoning tasks provide a valuable assessment of the task complexity on realistic graphs, they do not necessarily reflect the ``worst case'' graph inputs that convey negative results like \Cref{thm:neg-1L-informal} and \Cref{thm:graph-log-lb-informal}.
For example, the reduction that establishes that cycle check is ``as hard as'' graph connectivity and the consequential logarithmic-depth hardness results hinge on the consideration of graph instances with $n$ nodes and polynomial cycle length.
However, as witnessed by \Cref{fig:appendix:cycle_length}, the shortest cycles observed in 1000 instances of GraphQA cycle check is almost always of length three, and only 3.2\% of instances are larger.
As a consequence, identifying the existence of a cycle on the GraphQA dataset is inherently local, which is reflected by a strong performance by heuristic-based GNN solutions (\Cref{tab:local})---despite the fact that efficient GNNs for \textit{worst-case} cycle check do not exist  (\Cref{cor:loukas}).

For our experiments on the effect of the scale of the number of training data points in the final results we obtain, we use the open-source code of GraphQA available to generate a larger training dataset of 100K examples. We follow the original instructions and parameters to create this larger training dataset.

% \todo{CLAYTON: cycle length statistics chart description}

\subsection{Implementation Details}

\textbf{Model Hyperparameters.}
We fixed the number of iterations as 1,000,000 and train standard decoder-only transformers with $L = 12$ layers, $m = 768$ embedding dimension, $H = 12$ heads, learning rate $5 \cdot 10^{-4}$, and dropout $0.1$.
These models have an approximate parameter count of 60,000,000.

We used random search~\cite{bergstra2012random} over the following set of hyperparameters to select a universal architecture for all tasks: The range provided for the learning rate and dropout rate are $[10^{-4}, 10^{-1}]$ and $[0, 0.5]$. The number of layers $L$ and embedding dimension $m$ is selected from $L \in \set{4, 6, 8, 10, 12, 14, 16}$ and $m \in \set{192, 384, 576, 768, 960, 1152, 1344, 1536}$. We %fixed the number of heads as $H = 4$ and 
employed the GLU~\cite{shazeer2020glu} activation as a non-linearity.

\textbf{Model Selection.} 
We implemented our model in JAX~\cite{frostig2018compiling} and used AdamW~\cite{kingma2014adam,loshchilov2017decoupled} as the optimizer. Optimal hyperparameters for each task and model were determined by training on the GraphQA$_{Train}$ dataset and evaluating performance on the GraphQA$_{Dev}$ dataset. The results presented in the paper are based on the held-out GraphQA$_{Test}$ dataset.

\textbf{Hardware Acceleration.} All experiments were conducted using Google's TPUv3 and TPUv5e accelerators~\cite{tpu}.

\subsection{Baseline Results}

To rigorously evaluate the performance of transformers on graph reasoning tasks, we compare them against three established categories of baselines:

\begin{enumerate}[leftmargin=*]
    \item \textbf{Prompting-based methods.} These methods provide the LLM with a textual descriptions of the graph and question within the prompt. We consider the following variations and copy the results from the original papers:
    \begin{itemize}[leftmargin=*]
        \item{\acr{zero-shot}.} In this approach, the model is given a task description and immediately asked to produce the desired output. No additional examples or demonstrations are provided.
        
        \item{\acr{few-shot}.} This approach provides the model with a few examples of the task and their desired outputs~\citep{fewshot}. Unlike traditional training, these examples are included directly in the prompt, allowing the model to learn and adapt during the inference.
        
        \item{\acr{CoT}.} Chain-of-thought~(CoT) prompting~\citep{wei2022chain} provides examples each showing step-by-step reasoning, teaching the LLM to generate its own thought processes for tackling new tasks. 
        
        \item{\acr{zero-cot}.} Zero-shot CoT~\citep{kojima2022large} builds upon Chain-of-Thought (CoT) prompting by eliminating the need for training examples.  The LLM generates its own step-by-step reasoning process using a simple trigger phrase like ``Let's think step by step''.
        
        \item{\acr{cot-bag}.} BAG prompting~\citep{wang2023can} extends \acr{cot} to improve the performance of LLMs on graph-related tasks by appending ``Let's construct a graph with the nodes and edges first'' to the prompt.
        
        % \item{soft-prompt.} This approach uses the standard soft prompt tuning of \citet{lester2021power}. It optimizes a global \textit{static} prompt which is shared across problem instances to improve task performance. Unlike our proposed method, it does not have access to the graph information, making the results of this approach equivalent to that of a majority classifier.
    \end{itemize}

    \item \textbf{Graph-based methods.} These models are specifically designed to process graphs as input and are trained task-specific. They leverage the connections between nodes to learn patterns and make predictions, making them ideal for tasks where a graph is involved. We use GCN~\cite{kipf2016semi}, MPNN~\cite{gilmer2017neural}, and GIN~\cite{xhlj18} from this category. GraphToken~\cite{pfztkah24} is a GNN-based model that processes the graph and feed the output of the GNN as soft-tokens to an LLM.
    
    \item \textbf{Transformer models~(Ours).} The last class of model are task-specific vanilla transformer models~\cite{vsp17}. The \textit{60M transformer-1K} model is the one described above trained on $1,000$ training examples from the GraphQA training set.
     To investigate the impact of training data scale, we generated a larger dataset containing $100,000$ examples, ensuring the same distribution as the original training set by using the official GraphQA code and trained \textit{60M transformer-100K} on that. The \textit{11B transformer (FT)-1K} is a vanilla transformer model that is started with a pre-trained checkpoint of T5~\cite{raffel2020exploring} and is fine-tuned on the 1K training dataset.
     We also include two fine-tuned PaLM~2~\cite{anil2023palm} transformers of size XXS and XS.
     Similar to prompting baselines, this model receives a textual description of the graph as input to leverage its textual reasoning capabilities.
    
\end{enumerate}

The results for \acr{zero-shot}, \acr{zero-cot}, \acr{few-shot}, \acr{cot}, and \acr{cot-bag} are taken from \citet{fhp23}. Results for \acr{soft-prompt} and GraphToken are sourced from \citet{pfztkah24}. 

We independently evaluated GCN, MPNN, and GIN models on these tasks. We used the original architectures proposed in their respective papers and performed hyperparameter tuning on the GraphQA$_{Dev}$ dataset.

\subsection{Further Experimental results}

% We now proceed to the empirical verification of our findings from the previous section.
% We evaluate our model on the diverse graph reasoning tasks from the GraphQA benchmark~\cite{fhp23}, with detailed description of all tasks provided in \Cref{ssec:appendix-graphreasoning}.
% We classify all tasks into four groups: retrieval, parallelizable, search, and subgraph counting tasks.

% According to our theoretical analysis, we expect that:
% \begin{itemize}
%     % most aligned with theory
%     % log path -- pause / wide connections
%     \item $\mathcal{H}1$: Algorithmic capabilities of transformers allows them to beat GNNs, but at a higher computational cost. (Shortest path and connectivity: transformers 100k and PALM vs GNNs)

%     % two tables:  one connectivity, one parallelizable tasks

%     \item $\mathcal{H}2$: GNNs often have better inductive biases for local graph structure in sample-efficient regime. (Table: Node degree, cycle check.)
    
%     \item $\mathcal{H}3$: Trained transformers beat prompting across the board (Table: 1k transformer vs all prompting.)
%     % input: prompting vs trained transformers
% \end{itemize}

% 3. 
% input: scaling results

% also xformers do great at logdepth retrieval tasks but... not node degree, connect to gnns

% 2. 
% input: gnns vs transformers

% 1. 

\input{table1}

\begin{figure*}[t]
\centering
\resizebox{\textwidth}{!}{\input{figures/chart_scaling_results}}
\caption{Comparison of train and test scaling on all tasks.}\label{fig:appendix:scaling_multifig}
\end{figure*}
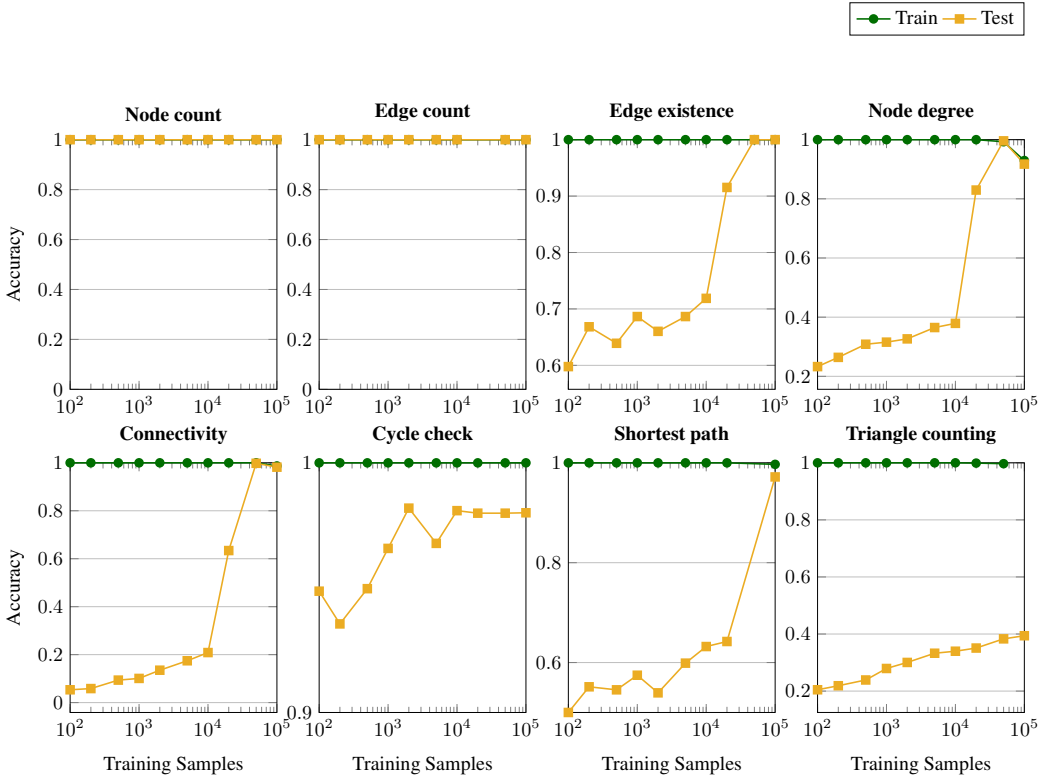

% \textbf{Baselines.}
% We compare a variety of neural architectures and learning regimes across various integration of the transformers and GNN architecture: vanilla transformers trained to solve specific tasks, vanilla GNNs, general LLMs with textual inputs and chain-of-thought prompting, pre-trained LLMs fine-tuned on graph tasks, and LLMs with GraphToken embeddings provided by GNNs.

% \textbf{Implementation.} We defer the implementation details 
% and the best hyperparameter settings 
% for our model on all the datasets to the supplementary material. To accelerate future research, we will open-source our code upon acceptance of the paper.

% \textbf{Results.}
\Cref{table:main_results} presents a comprehensive comparison of the graph reasoning capabilities across various baseline models and our proposed transformer architectures. The results highlight several key findings, which we summarize below:

\textbf{Transformers Exhibit Strong Performance on Graph-based Reasoning Problems.}
While transformers are not explicitly designed for graph reasoning tasks like graph-based models, they demonstrate surprisingly strong performance in this domain. The results of this study indicate that transformers, despite their versatility as a general architecture, can often match or even surpass specialized graph models on a variety of graph reasoning benchmarks.

\textbf{Transformers Excel at Retrieval Tasks.} As proved in~\Cref{thm:pos-1L}, retrieval tasks can be solved by transformers. The obtained results confirm that such tasks are relatively easy for transformers as they obtained the full accuracy on most of such tasks. One exception here is the node degree task that GNNs outperform transformers but still transformers perform relatively well. We discuss why GNNs outperform transformers well for this task.

\textbf{Larger Transformers Excel at Solving Search Tasks.} As discussed in~\Cref{thm:search-pos}, transformers are effective for search tasks, albeit requiring a larger number of parameters compared to retrieval tasks. This is empirically evident in the comparison between Transformer-1K and Transformer-1K~(pretrained). It's worth noting that the pretrained transformer used here has 11 billion parameters, a significant increase from the 1 million parameters in Transformer-1K.

\textbf{Transformers Excel at Capturing Global Patterns.}
An interesting observation here is the performance gap between Transformers and GNNs across tasks with varying emphasis on local versus global graph structure. Notably:
\begin{enumerate}[leftmargin=*]
    \item \textbf{Local Structure:} The node degree task, which relies heavily on local node information, is best handled by MPNN, a GNN-based model.
    \item \textbf{Global Structure:} In contrast, tasks like connectivity, triangle counting and shortest path, which require understanding global graph patterns, are dominated by Transformer models. 
    Notably, vanilla Transformers achieve a remarkable 45\% relative improvement over even much larger LLMs augmented with GNN-generated soft prompts (GraphToken) on the shortest path task. This showcases the exceptional ability of Transformers to capture long-range dependencies, a critical factor in understanding global graph structure.
\end{enumerate}

\subsubsection{Sample complexity ablations}\label{asssec:sc}

To develop a more comprehensive understanding of graph reasoning tasks learnability by small transformers, we train a variety of transformers for 1,000,000 steps on each task on a range of sample sizes.
In \Cref{fig:appendix:scaling_multifig}, we demonstrate how the model performance improves as a function of sample complexity.
By doing so, we witness the relative hardness of each task in terms of the marginal benefit of new samples.
\begin{itemize}
    \item Edge count and node count are ``easy'' retrieval tasks that can be solved perfectly with as few as 100 samples. 
    % This is perhaps unsurprising, since the number of blank tokens in the graph tokenization input coincides with these quantities. 
    \item Edge existence attains near-perfect train and test classification accuracy as the number of training samples approaches 100,000.
    \item Connectivity, shortest path, %connected nodes (\Cref{fig:sc-connected-nodes}),
    and node degree demonstrate a sharp improvement in evaluation error as a function of the sample size. These models perfectly fit the training set in most sample size regimes, but yield a closer correspondence between training and testing error when trained on 100,000 samples.
    \item Cycle check and triangle count have persistent gaps between training and testing error and overfit even in the large sample setting.
\end{itemize}

%% file: figures/chart_cycle_length.tex
\begin{tikzpicture}
\begin{axis}[
    symbolic x coords={no cycle, 3, 4, 5, 6},
    xtick=data,ymin=0,width=\linewidth]%, ymode=log]
    \addplot[ybar,fill=cycle3] coordinates {
        (no cycle,154)
        (3,814)
        (4,22)
        (5,7)
        (6,3)
    };
\end{axis}
\end{tikzpicture}

%% file: table1.tex
\begin{table}[t]
\vspace{-2em}
\caption{
Comparison of various methods in different categories on graph reasoning tasks of GraphQA. Here, we categorize the tasks using the taxonomy proposed in~\Cref{sec:hierarchy}.
}\label{table:main_results}
% \footnotesize
\resizebox{\textwidth}{!}{%
\setlength{\tabcolsep}{3pt}
\begin{tabular}{llcccccccc} 
\toprule
& & \multicolumn{4}{c}{\textbf{Retrieval tasks}} & \multicolumn{2}{c}{\textbf{Parallelizable Tasks}} & \multicolumn{1}{c}{\textbf{Search Tasks}} & \multicolumn{1}{c}{\textbf{Subgraph Counting}}\\
\cmidrule(lr){3-6} \cmidrule(lr){7-8} \cmidrule(lr){9-9} \cmidrule(lr){10-10}
% \textbf{Method} & \textbf{Node count} & \textbf{Edge count} & \textbf{Cycle check} & \textbf{Triangle counting} & \textbf{Node degree} & \textbf{Connected nodes} & \textbf{Reachability} & \textbf{Edge existence} & \textbf{Shortest path}\\ \midrule
% \acr{zero-shot} & 0.217 & 0.124 & 0.760 & 0.015 & 0.140 & 0.147 & 0.849 & 0.445 & 0.115 \\
% \acr{zero-cot} & 0.146 & 0.094 & 0.323 & 0.127 & 0.104 & 0.088 & 0.735 & 0.335 & 0.336 \\
% \acr{few-shot} & 0.253 & 0.120 & 0.374 & 0.030 & 0.174 & 0.124 & 0.794 & 0.368 & 0.227 \\
% \acr{cot} & 0.276 & 0.128 & 0.580 & 0.081 & 0.292 & 0.131 & 0.452 & 0.428 & 0.386 \\
% \acr{cot-bag} & 0.269 & 0.125 & 0.521 & 0.081 & 0.280 & 0.158 & 0.452 & 0.373 & 0.404 \\
% \acr{soft-prompt} & 0.056 & 0.018 & 0.832 & 0.162 & 0.098 & 0.068 & 0.838 & 0.544 & 0.462 \\
% GraphToken & 0.996 & 0.426 & 0.956 & 0.348 & 0.962 & 0.264 & 0.932 & 0.738 & 0.638 \\
 & \textbf{Method} & \textbf{Node count} & \textbf{Edge count} & \textbf{Edge existence} & \textbf{Node degree} &  \textbf{Connectivity} & \textbf{Cycle check} & \textbf{Shortest path} &  \textbf{Triangle counting}  \\ \midrule
\multirow{5}{*}{\begin{sideways}\scriptsize Prompting \end{sideways}} & \acr{zero-shot}~\cite{fhp23} & 21.7 & 12.4 & 44.5 & 14.0 & 84.9 & 76.0 & 11.5 & 1.5 \\
% & 0.147 \\
& \acr{zero-cot}~\cite{fhp23} & 14.6 & 9.4 & 33.5 & 10.4 & 73.5 & 32.3 & 33.6 & 12.7 \\
% & 0.088 \\
& \acr{few-shot}~\cite{fhp23} & 25.3 & 12.0 & 36.8 & 17.4 & 79.4 & 37.4 & 22.7 & 3.0 \\
% & 0.124 \\
& \acr{cot}~\cite{fhp23} & 27.6 & 12.8 & 42.8 & 29.2 & 45.2 & 58.0 & 38.6 & 8.1 \\
% & 0.131 \\
& \acr{cot-bag}~\cite{fhp23} & 26.9 & 12.5 & 37.3 & 28.0 & 45.2 & 52.1 & 40.4 & 8.1 \\
% & 0.158 \\
% & \acr{soft-prompt}~\cite{levine2022standing} & 5.6 & 1.8 & 54.4 & 9.8 & 83.8 & 83.2 & 46.2 & 16.2 \\
% & 0.068 \\
\midrule
\multirow{3}{*}{\begin{sideways}\scriptsize Graph-based\end{sideways}}  & GCN~\cite{kipf2016semi} & 6.4 & 1.2 & 47.0 & 9.8 & 83.8 & 83.2 & 50.2 & 4.0 \\
% & MPNN~\cite{gilmer2017neural} & 89.8 & 66.2 & 71.4 & \textbf{99.2} & 93.4 & \textbf{98.8} & 64.8 & 33.6\\
& MPNN~\cite{gilmer2017neural} & 19.4 & 16.2 & 69.2 & \textbf{99.4} & 94.0 & \textbf{99.0} & 66.8 & 30.6\\
& GIN~\cite{xhlj18} & 71.2 & 4.4 & 71.2 & 36.2 & 93.8 & \underline{98.8} & 54.0 &	30.4\\
& GraphToken~\cite{pfztkah24} & \underline{99.6} & 42.6 & 73.8 & \underline{96.2} & 93.2 & 95.6 & 63.8 & \underline{34.8} \\
% & 0.264 \\
\midrule
\multirow{5}{*}{\begin{sideways}\scriptsize Ours \end{sideways}} & 
% Transformer-1K 
% 60M transformer-1K
% & \underline{99.8} & \underline{97.2} &	67.6 &	32.5 &	91.4 &	97.1 &	55.5 &	33.4 \\
% & 
% 60M transformer-100K
% %Transformer-100K
% & \textbf{100.0} &	\textbf{100.0} &	\underline{96.1} &	75.6 &	\underline{97.2} &	98.0 &	\underline{68.0} &	\textbf{40.5} \\
60M transformer-1K
& \textbf{100.0} & \textbf{100.0} &	67.6 &	31.5 &	92.9 &	97.1 &	57.4 &	33.4 \\
& 
% Transformer-100K
60M transformer-100K
& \textbf{100.0} &	\textbf{100.0} &	96.1 &	91.7 &	\underline{98.0} &	98.0 &	\textbf{97.2} &	\textbf{40.5} \\
& 
% PaLM~2~1B 
XXS transformer (FT)-1K
& \textbf{100.0} & 70.6 & 73.0 & 31.0 & 93.6 & 98.0 & 60.4 & 29.0 \\
& 
% % PaLM~2~8B 
XS transformer (FT)-1K
& \textbf{100.0} & \underline{73.2} & \underline{98.6} & 50.6 & 96.6 & 96.8 & 60.0 & 28.6 \\
&
12B transformer (FT)-1K
%Transformer-1K~(pretrained) 
& \textbf{100.0} & 45.0 & \textbf{100.0} & 68.8 & \textbf{98.4} & 98.0 & \underline{92.8} & 26.0\\
\bottomrule
\end{tabular}
}
\end{table}

%% file: figures/chart_scaling_results.tex
\begin{tikzpicture}
\begin{groupplot}[group style={
                      group name=myplot,
                      group size= 4 by 2, horizontal sep=0.75cm,
                      vertical sep=1.3cm
                      },
                      height=6cm,
                      width=0.375 \linewidth,ymax=1,title style={at={(0.5,1.0)},anchor=south},ymajorgrids=true,xmin=100,xmax=100000,xmode=log,anchor=south east,every axis x label/.style={at={(axis description cs:0.5,-0.15)},anchor=north}]
\nextgroupplot[
 	title = \textbf{Node count},
 	ylabel=Accuracy,
 	ymin=0,
 	]
\addplot[thick,color=cycle4,mark=*] coordinates {
(100, 1.0)
(200, 1.0)
(500, 1.0)
(1000, 1.0)
(2000, 1.0)
(5000, 1.0)
(10000, 1.0)
(20000, 1.0)
(50000, 1.0)
(100000, 1.0)
};
\addplot[thick,color=cycle1,mark=square*] coordinates {
(100, 1.0)
(200, 1.0)
(500, 1.0)
(1000, 1.0)
(2000, 1.0)
(5000, 1.0)
(10000, 1.0)
(20000, 1.0)
(50000, 1.0)
(100000, 1.0)
};
\nextgroupplot[
 	title = \textbf{Edge count},
 	ymin=0,
 	]
\addplot[thick,color=cycle4,mark=*] coordinates {
(100, 1.0)
(200, 1.0)
(500, 1.0)
(1000, 1.0)
(2000, 1.0)
(5000, 1.0)
(10000, 1.0)
(50000, 1.0)
(100000, 1.0)
};
\addplot[thick,color=cycle1,mark=square*] coordinates {
(100, 1.0)
(200, 1.0)
(500, 1.0)
(1000, 1.0)
(2000, 1.0)
(5000, 1.0)
(10000, 1.0)
(50000, 1.0)
(100000, 1.0)
};
\nextgroupplot[
 	title = \textbf{Edge existence},
 	]
\addplot[thick,color=cycle4,mark=*] coordinates {
(100, 1.0)
(200, 1.0)
(500, 1.0)
(1000, 1.0)
(2000, 1.0)
(5000, 1.0)
(10000, 1.0)
(20000, 1.0)
(50000, 1.0)
(100000, 1.0)
};
\addplot[thick,color=cycle1,mark=square*] coordinates {
(100, 0.5977822542190552)
(200, 0.6683467626571655)
(500, 0.6391128897666931)
(1000, 0.6864919066429138)
(2000, 0.6602822542190552)
(5000, 0.6864919066429138)
(10000, 0.71875)
(20000, 0.9153225421905518)
(50000, 1.0)
(100000, 1.0)
};

\nextgroupplot[
 	title = \textbf{Node degree},
 	legend style={at={(1,1.55)}},
 	legend columns=2,
 	]
\addplot[thick,color=cycle4,mark=*] coordinates {
(100, 1.0)
(200, 1.0)
(500, 1.0)
(1000, 1.0)
(2000, 1.0)
(5000, 1.0)
(10000, 1.0)
(20000, 1.0)
(50000, 0.992943525314331)
(100000, 0.9294354319572449)
};
\addplot[thick,color=cycle1,mark=square*] coordinates {
(100, 0.23286288976669312)
(200, 0.2641128897666931)
(500, 0.3084677457809448)
(1000, 0.3155241906642914)
(2000, 0.3266128897666931)
(5000, 0.3649193346500397)
(10000, 0.3790322542190552)
(20000, 0.8296370506286621)
(50000, 0.9959676861763)
(100000, 0.9173386693000793)
};
\legend{Train, Test}
\nextgroupplot[
 	title = \textbf{Connectivity},
 	ylabel=Accuracy,
 	xlabel={Training Samples},
 	]
\addplot[thick,color=cycle4,mark=*] coordinates {
(100, 1.0)
(200, 1.0)
(500, 1.0)
(1000, 1.0)
(2000, 1.0)
(5000, 1.0)
(10000, 1.0)
(20000, 1.0)
(50000, 1.0)
(100000, 0.9868951439857483)
};
\addplot[thick,color=cycle1,mark=square*] coordinates {
(100, 0.05342741683125496)
(200, 0.05846773833036423)
(500, 0.09375)
(1000, 0.10080644488334656)
(2000, 0.13508063554763794)
(5000, 0.17439515888690948)
(10000, 0.20866934955120087)
(20000, 0.6340725421905518)
(50000, 0.9979838132858276)
(100000, 0.9818547964096069)
};

\nextgroupplot[
 	title = \textbf{Cycle check},
 	xlabel={Training Samples},
 	ytick={0.9,1},
 	ymin=0.9,
 	]
\addplot[thick,color=cycle4,mark=*] coordinates {
(100, 1.0)
(200, 1.0)
(500, 1.0)
(1000, 1.0)
(2000, 1.0)
(5000, 1.0)
(10000, 1.0)
(20000, 1.0)
(50000, 1.0)
(100000, 1.0)
};
\addplot[thick,color=cycle1,mark=square*] coordinates {
(100, 0.9485886693000793)
(200, 0.9354838132858276)
(500, 0.9495967626571655)
(1000, 0.9657257795333862)
(2000, 0.9818547964096069)
(5000, 0.9677419066429138)
(10000, 0.9808467626571655)
(20000, 0.9798386693000793)
(50000, 0.9798386693000793)
(100000, 0.980)
};

\nextgroupplot[
 	title = \textbf{Shortest path},
 	xlabel={Training Samples},
 	ymin=0.5,
 	]
\addplot[thick,color=cycle4,mark=*] coordinates {
(100, 1.0)
(200, 1.0)
(500, 1.0)
(1000, 1.0)
(2000, 1.0)
(5000, 1.0)
(10000, 1.0)
(20000, 1.0)
(100000, 0.9969757795333862)
};
\addplot[thick,color=cycle1,mark=square*] coordinates {
(100, 0.5)
(200, 0.5514112710952759)
(500, 0.5453628897666931)
(1000, 0.5745967626571655)
(2000, 0.5393145084381104)
(5000, 0.5987902879714966)
(10000, 0.6320564150810242)
(20000, 0.6421370506286621)
(100000, 0.971774160861969)
};
\nextgroupplot[
 	title = \textbf{Triangle counting},
 	xlabel={Training Samples},
 	]
\addplot[thick,color=cycle4,mark=*] coordinates {
(100, 1.0)
(200, 1.0)
(500, 1.0)
(1000, 1.0)
(2000, 1.0)
(5000, 1.0)
(10000, 1.0)
(20000, 0.9989919066429138)
(50000, 0.9969757795333862)
};
\addplot[thick,color=cycle1,mark=square*] coordinates {
(100, 0.2046370953321457)
(200, 0.21875)
(500, 0.23891128599643707)
(1000, 0.2792338728904724)
(2000, 0.3004032075405121)
(5000, 0.3326612710952759)
(10000, 0.3397177457809448)
(20000, 0.35080644488334656)
(50000, 0.38306450843811035)
(100000, 0.394)
};
\end{groupplot}
\end{tikzpicture}